\documentclass[11pt]{article}
\usepackage[utf8]{inputenc}
\usepackage{algorithm}
\usepackage{mystyle}
\newcommand{\revise}[1]{\textcolor{black}{#1}}

\begin{document}

\title{\huge Policy Optimization Using Semi-parametric Models for  Dynamic Pricing}

\author{Jianqing Fan\thanks{Department of Operations Research and Financial Engineering, Princeton University.  Research supported by the NSF grant  DMS-2052926, DMS-2053832, and the ONR grant N00014-19-1-2120.}\quad\quad\quad Yongyi Guo$^*$\quad\quad\quad Mengxin Yu$^*$}


\maketitle

\begin{abstract}
{

In this paper, we study the contextual dynamic pricing problem where the market value of a product is linear in some observed features plus some market noise. Products are sold one at a time, and only a binary response indicating success or failure of a sale is observed.  Our model setting is similar to \cite{JN19} except that we expand the demand curve to a semi-parametric model and need to learn dynamically both parametric and non-parametric components.
We propose a dynamic statistical learning and decision making policy that combines semi-parametric estimation from a generalized linear model  with online decision making to minimize regret (maximize revenue). Under mild conditions, we show that for a market noise c.d.f. $F(\cdot)$ with $m$-th order derivative ($m\geq 2$), our policy achieves a regret upper bound of $\tilde{\cO}_{d}(T^{\frac{2m+1}{4m-1}})$, where $T$ is time horizon and $\tilde{\cO}_{d}$ is the order that hides logarithmic terms and the dimensionality of feature $d$. The upper bound is further reduced to $\tilde{\cO}_{d}(\sqrt{T})$ if $F$ is super smooth whose Fourier transform decays exponentially. In terms of dependence on the horizon $T$, these upper bounds are close to $\Omega(\sqrt{T})$, the lower bound where $F$ belongs to a parametric class. We further generalize these results to the case with dynamically dependent product features under the strong mixing condition.}
\end{abstract}

\section{Introduction}

Dynamic pricing is the study of determining and adjusting the selling prices of products over time based on statistical learning and policy optimization. As an integral part of revenue management, it has wide applications to various industries. Research on dynamic pricing has spanned across the fields of statistics, machine learning, economics, and operations research
 \citep{dB15,WEI2018166,MP20}. In general, a good pricing strategy often involves good statistical learning of the demand function as well as revenue optimization over time.

Recent works particularly focus on feature-based (or contextual) pricing models, where the market value of a product as well as the pricing strategy depend on some observable features of the product \citep{JN19,Ban2020PersonalizedDP}. Given the product features (covariates) available through the massive real-time data in online platforms today, feature-based pricing models take product heterogeneity into account, which enable customized pricing for products.

In this work, we consider the following dynamic pricing problem: We assume that a seller sells one product at each time $t=1, \cdots, T$. Each product is attached with a known feature vector $\xb_t \in \RR^d$. In addition, the product's market value $v_t$ is linear in the features plus some i.i.d. market noise $z_t$ with an \emph{unknown} cumulative distribution $F(\cdot)$:
{
$$
v_t = \btheta_0^\top \tilde\xb_t+z_t, \qquad z_t \sim F.
$$
Here $\tilde\xb_t = (\xb_t^\top,1)^\top$ and $\btheta_0$ is some unknown parameter.} The customer makes an independent purchase decision for each product depending on whether the seller's posted price $p_t$ is higher than the market value $v_t$, after which the revenue is collected. In this case, the demand curve $P(v_t \geq p_t)$ actually depends on both the  parameter $\btheta_0$ as well as the distribution of $z_t$, which admits a semi-parametric form. They need to be learned or estimated dynamically from the observed binary data that indicates whether a sale is successful.   Under this setting, we propose a policy which utilizes semi-parametric estimation techniques to achieve a low regret. In particular, under mild regularity conditions,  if the c.d.f. $F(\cdot)$ of $z_t$ has $m^{th}$ derivative, the regret over a time horizon $T$ is upper bounded by $\cO((Td)^\frac{2m+1}{4m-1}\log T(1+\log T/d))$, where $d$ is the number of features. This result is further generalized to a setting where the product features $\xb_t$ are not independent, as long as $\{\xb_t\}_{t\ge 1}$ is a stationary series that satisfies certain $\beta$-mixing conditions. Moreover, when $F$ is infinitely differentiable, the total regret can be upper bounded by $\tilde \cO((Td)^{\frac{1}{2}}(\log T)^{\frac32+\frac{3}{2\alpha}}(\log (d+1) + \log T/d))$. { This rate is the same as the parametric lower bound up to some logarithmic factors, i.e. where the distribution of $z_t$ is generated from a parametric class.}

\subsection{Related Literatures}
Our work contributes to the recent line of dynamic pricing literature as well as the growing literature on decision making with covariate information and contributes to kernel regression. Our work is also closely related to the non-parametric statistics literature. We'll briefly review the related works in the below.

\textbf{Dynamic pricing}.   In the classical pricing models, one aims at maximizing the revenue over time by posting price sequentially while learning the underlying demand curve or market evaluation of a product. The demand curve is typically fixed over time, and falls into a known function class. Related literature includes \cite{KL03, PBP06, BZ09, BR12, KZ2014,dBZ14, WDY14, dBZ15, BDKA15, Bianchi2019, CJD2019}. { As an example, \cite{Bianchi2019} study the dynamic pricing problem where the buyer's valuation of a product is supported on a finite $K$ unknown points, and the success of a sale is determined by comparing the valuation to the proposed price. Using a generalization of UCB algorithm, the authors achieve the regret with order $\cO(K\log T)$.} For a comprehensive survey on this topic, see {\cite{dB15}}.
	
Recently, many papers have been focusing on contextual dynamic pricing, where product heterogeneity is taken into account when modeling the demand curve or market evaluation. A common and natural choice is to model the market value of the product at time $t$ as a linear function of its features $\xb_t$ plus some market noise $z_t$, i.e. $v_t = \btheta^\top \xb_t + z_t$ where $\btheta$ is some unknown parameter \citep{Qiang2016, J17, MCCLZ19, JN19, Ban2020PersonalizedDP, WWSC20, CSW20, THL20, GJM20}. Under this setting, for `truthful' buyers whose decision is based on comparing $v_t$ and offered price $p_t$, the demand curve can be expressed as a generalized linear model given feature covariates $\xb_t$, where the link function is closely related to the distribution of the market noise $z_t$ (see \eqref{eq:prob-model} for a detailed reasoning). \cite{Qiang2016} assume a linear model between the demand curve and the product features. They prove that the greedy iterative least squares (GILS) algorithm achieves a regret upper bound of $\cO_d(\log T)$, where $\tilde{\cO}_{d}$ is the order that hides logarithmic terms and the dimensionality of feature $d$, and provide a matching lower bound under their setting. \cite{MCCLZ19} and \cite{ Ban2020PersonalizedDP} consider a generalized linear model with known link, while \cite{JN19} and \cite{WWSC20} study the same problem with high dimensional sparse parameters. The algorithms are usually a combination of statistical estimation procedures and online learning techniques. Depending on the setting, the optimal regret ranges from $\cO_d(\log T)$ to $\tilde \cO_d(\sqrt {T})$. Other related works include \cite{CSW20, THL20} where the authors explore certain differentially private policies under similar model setting; \cite{GJM20} where the authors consider the second price auction problem with multiple customers, each of which has his/her own product evaluation; and \cite{J17} where the parameter $\btheta$ in the generalized linear model changes through time.

In practice, however, the distribution of the market noise $z_t$ is usually unknown to the seller. Thus, it might be desirable to only assume that the noise density falls into some general class. As will be discussed in \S\ref{sec:algo}, this leads to modeling the demand curve as a generalized linear model with unknown link, and will be our main focus in this paper. Compared to the previous setting, this setting is more challenging, and the related literature is sparse. \cite{JN19} propose a preliminary algorithm that achieves a regret upper bound of $\cO_{d}(T)$. \cite{Golrezaei2019} consider a second price auction with reserve where there are more than one customers, each of whom has his/her individual parameters in their demand curve model, and the customer bids are available as additional information. The authors propose the NPAC-T/NPAC-S policy that achieves a regret $\tilde{\cO}_{d}(\sqrt{T})$. \cite{GJM20} also explore the second price auction and derive a regret upper bound of $\tilde{\cO}_{d}(T^{2/3})$ compared to a `robust benchmark' where the price maximizes the revenue of the worst link function in the class. \cite{SJB2019} explore an alternative setting where the market value $v_t = \exp(\btheta^\top \xb_t + z_t)$ and $z_t$ has unknown distribution. By utilizing this specific structure, the authors propose the DEEP-C algorithm based on multi-arm bandit that has a regret upper bound of $\tilde{\cO}_{d}(\sqrt{T})$. The authors also propose some variants of the algorithm and study them via simulations. { Recently, \cite{LSL21,xu2022towards} study a similar problem to ours, assuming a linear market valuation with unknown noise distribution. In particular, \cite{LSL21} provide a DIP policy that achieves regret  $\cO_d(T^{2/3}+\|\hat\btheta-\btheta\|_1T)$, where $\|\hat\btheta-\btheta\|_1$ is the estimation accuracy of the parameter $\btheta$. In addition, \cite{xu2022towards} presents an algorithm `D2-EXP4' that achieves regret with order $\cO(T^{3/4})$.} 

{ There are some literature studying other dynamic pricing algorithms \citep{ARS14, CLPL16, MLS2018, LS18, NSW19, AA20, A2020, Ban2020PersonalizedDP, LZ20, JNS20, chen2020, LLS21}. For example, \cite{MLS2018} study a non-parametric dynamic pricing pricing where the market value is modeled as a general non-parametric function $f(\xb_t)$, where $\xb_t$ are the features. A binary feedback is similarly observed based on the comparison between $f(\xb_t)$ and the proposed price. The authors apply a variation of midpoint algorithm and achieve a regret upper bound of $\cO(T^{d/(d+1)})$ with $d$ being the dimension of $\xb_t$.} 

\textbf{Semi-parametric and non-parametric statistical estimation}.  Our work is also closely related to estimation of the single index model, or the generalized linear model with an unknown link. Such model has been studied in the statistics and econometrics literature for decades, and has wide applications in fields like econometrics and finance \citep{PSS89, ICHIMURA199371, HHI93, KS93, WW94, MG94, HH96, CFGW97, XL99, DHH03, FL04}. For a comprehensive summary of these works, please refer to \cite{McC00, Gyorfi2002,FY03, RWC03, T08, Hor12}. Various methods have been proposed to estimate the parametric part that achieves root-$n$ consistency under certain conditions \citep{PSS89, ICHIMURA199371, KS93}. \cite{CFGW97} study the generalized partial linear single index models, where the authors leverage local linear kernel regression with quasi-likelihood method to estimate both the parametric and non-parametric parts of the model. {\cite{XL99} investigate in the single index coefficient model with strong-mixing features. Estimators with uniform convergence rate to the ground truth based on kernel regression are proposed.}

Given a root-$n$ consistent estimation of the coefficients, standard univariate non-parametric regression techniques can be used to estimate the non-parametric part of the single index model that achieves $\ell_\infty$ consistency, which is necessary in deriving regret upper bounds. One common estimator is the Nadaraya-Watson estimator \citep{N64,W64}. \cite{Silverman78} and \cite{MS82} establish uniform convergence results for kernel density estimator and Nadaraya-Watson estimator for regression functions. In addition, \cite{Stone80,Stone82} derive uniform convergence results for the more general local polynomial regression estimators. \cite{Masry96} prove similar results when the covariates satisfy strong-mixing conditions.

In this paper, we'll provide non-asymptotic error bounds for both coefficient estimation as well as the plug-in Nadaraya-Watson estimator in a uniform sense. These non-asymptotic bounds are useful for constructing regret bounds within a finite horizon.

\subsection{Our Contributions}
Our contributions are the following: First, compared to related works, our policy achieves a low regret with few assumptions on the market noise distribution and little additional information. Given $F\in \mathbb{C}^{(m)}$ where $F$ is the c.d.f. of $z_t$, the regret over a time horizon $T$ is upper bounded by $\tilde \cO((Td)^\frac{2m+1}{4m-1})$; If $F$ is `super smooth', the bound is further reduced to $\tilde \cO(\sqrt{Td})$, which is nearly the same regret order by assuming a parametric distribution for $z_t$ as in \cite{JN19} where the $s$-sparsity on $\bbeta_0$ is imposed.  Table {\red 1}  illustrates the settings of our work as well as several related literatures. \cite{GJM20} choose a more `conservative' regret by comparing to a benchmark policy which minimizes revenue with the worst demand function over the whole ambiguity function class. In contrast, our notation of regret is more standard and 'accurate' in that our benchmark policy knows the exact demand function given any product features. \cite{SJB2019} consider a log-linear relation between the market value and the covariates instead of a linear relation and derive a regret upper bound of $\tilde{\cO}(\sqrt{T}d^{11/4})$. Their algorithm based on multi-arm bandit has sub-optimal dependence on the dimension $d$ in terms of both regret and complexity, and is quite difficult to implement under general conditions. 
Interestingly, the authors conjecture that under the linear settings, there is no policy that achieves an $\tilde{\cO}_d(\sqrt{T})$ regret. Our work partly answers their guess by providing a policy with a $\tilde{\cO}(\sqrt{Td})$ regret when the demand function is sufficiently smooth. { Compared with the DIP policy in \cite{LSL21} and its regret $\cO_d(T^{2/3}+\|\hat\btheta-\btheta\|_1T)$, we are more clear on how $\hat\btheta$ are estimated within the pricing algorithm, and we provide explicit rate on both the estimation error and the regret.} {Moreover, compared to several fully non-parametric dynamic pricing literatures, such as \cite{MLS2018} and \cite{chen2020}, our algorithm scales more nicely with dimension $d$, and can easily be generalized to a high-dimensional setting. Our algorithm is also easy to implement compared to some bandit-based algorithms that need dividing the feature space into bins.
}

 \begin{table}\label{form1}
	\begin{center}
		\def\arraystretch{1.3}
		\setlength\tabcolsep{2pt}
		\begin{tabular}{c||c|c|c}
			\hline
			& Feature-based & Non-parametric noise& Regret  \\
            \hline
			\cite{KL03} &  & \checkmark & $\tilde{\cO}(\sqrt{T})$ \\
			\cite{JN19} & \checkmark & & $\tilde{\cO}(s\sqrt{T})$ \\
		\multirow{2}{*}{	\cite{SJB2019}} & \checkmark
			 &\multirow{2}{*}{\checkmark}  &{\multirow{2}{*}{$\tilde{\cO}(\sqrt{T}d^{11/4})$} } \\
			  &(log-linear model) & &\\
			\!\!\!\!\multirow{2}{*}{  \cite{GJM20}}  \!\!\!\! &\multirow{2}{*}{\checkmark}   & \multirow{2}{*}{\checkmark}  &{$\tilde{\cO}(dT^{2/3})$}   \\
			 & & &(changed benchmark)\\
			\!\!\!\!  \cite{LSL21} \!\!\!\! &\checkmark   & \checkmark &\revise{$\cO_d(T^{2/3}+\|\hat\theta-\theta_0\|_1T) $}\\
					\multirow{2}{*}{Our work } & \checkmark
			&\multirow{2}{*}{\checkmark}  &{\multirow{2}{*}{$\tilde{\cO}((Td)^{\frac{2m+1}{4m-1}})$} } \\
			 &(linear model) & &\\
			\hline
		\end{tabular}
	\end{center}
\caption{Comparison with related works.}
\end{table}

Second, we generalize our results to the regime where the product features $\{\xb_t\}_{t\ge 1}$ are weakly dependent instead of independent, which is more likely in practice. For example, for many products (such as softwares, electric products, etc.), the features of the products evolve over time and definitely inherit some past information. In other situations, the products for sale might have some common time-dependent factors shared by all products in the same industry (such as weather condition, population composition, etc.). This setting with weakly-dependent features can also be found in literatures such as \cite{COPS2021}, where the authors study an offline pricing problem with parametric models and dependent covariates.


Last but not least, we establish non-asymptotic results on the $\ell_{\infty}$ error bound of the nonparametric kernel  density and regression estimation, which are potentially useful in other related study as well. As mentioned in the related literatures, most results on non-parametric kernel regression estimation are established under the asymptotic settings. Meanwhile, we believe that non-asymptotic results are necessary to achieve a finite-sample regret upper bound in the pricing problem. Please refer to Appendix \ref{pf:conckernel_main} for related lemmas.

\subsection{Notation}
 Throughout this work, we use $[n]$ to denote $\{1,2,\cdots,n\}.$ For any vector $\xb\in \RR^{n}$ and $q\ge 0$, we use $\|\xb\|_{q}$ to represent the vector $\ell_q$ norm, i.e. $\|\xb\|_{q}=(\sum_{i=1}^{n}|x_i|^q)^{1/q}.$ In addition, we let $\nabla_{\xb}L(\cdot),\nabla_{\xb}^2L(\cdot)$ be the gradient vector and Hessian matrix of loss function $L(\cdot)$ with respect to $\xb$. For any given matrix $\Xb\in \RR^{d_1\times d_2}$, we use $\|\cdot\|$ to denote the spectral norm of $\Xb$ and we write $\Xb\succcurlyeq 0$ or $\Xb\preccurlyeq 0$ if $\Xb$ or $-\Xb$ is semidefinite.  For any event $A$, we let $\II_{A}$ be a indicator random variable which is equal to $1$ if $A$ is true and $0$ otherwise. In addition, we use $\CC^{(m)}$ with $m\in \mathbb N$ to denote the function class which contains all functions with $m$-th order continuous derivatives.
For two positive sequences $\{a_n\}_{n\ge 1}$, $\{b_n\}_{n\ge 1}$, we write $a_n=\cO(b_n)$ or $a_n\lesssim b_n$ if there exists a positive constant $C$ such that $a_n\le C\cdot b_n$ and we write $a_n=o(b_n)$ if $a_n/b_n\rightarrow 0$. In addition, we write $a_n=\Omega(b_n)$ or $a_n\gtrsim b_n$ if $a_n/b_n\ge c$ with some constant $c>0$. We use $a_n=\Theta(b_n)$ if $a_n=\cO(b_n)$ and $a_n=\Omega(b_n)$. We use notations $\cO_{d}(\cdot),\Omega_{d}(\cdot)$ and $\Theta_{d}(\cdot)$ to denote similar meanings as above while treating the variable $d$ as fixed. Moreover, we let $\tilde{\cO}(\cdot),\tilde{\Omega}(\cdot),\tilde{\Theta}(\cdot)$ represent the same meaning with $\cO(\cdot),\Omega(\cdot )$ and $\Theta(\cdot)$ except for ignoring log factors.
\subsection{Roadmap}

	The rest of this paper is organized as follows. We describe the problem in \S \ref{sec:algo} and propose a solution in \S\ref{alg_reg} where  some heuristic arguments are offered for bounding the regret. In \S\ref{sec_result}, we provide our theoretical results on the upper bounds of the regret and  in \S\ref{lowerbound}, we discuss a lower bound result.  Our algorithm is illustrated  in \S\ref{sec_simu} by intensive simulation experiments.

\section{Problem Setting} \label{sec:algo}
We consider the pricing problem where a seller has a single product for sale at each time period $t=1,2,\cdots,T$. Here $T$ is the total number of periods (i.e. length of horizon) and may be unknown to the seller. The market value of the product at time $t$ is $v_t$ and is unknown. We assume that the range of $v_t$ is contained in a closed interval in $(0,B)$. In particular, we assume that $v_t\in[\delta_v, B-\delta_v]$ for some constant $\delta_v>0$. At each period $t$, the seller posts a price $p_t$. If $p_t\le v_t$, a sale occurs, and the seller collects a revenue of $p_t$; otherwise, no sale occurs and no revenue is obtained. Let $y_t$ be the response variable that indicates whether a sale has occurred at period $t$. Then
\begin{align}\label{mymodel}
y_t = \begin{cases}
+1&\text{ if }v_t \ge p_t\,,\\
0&\text{ if }v_t <p_t\,.
\end{cases}
\end{align}
The goal of the seller is to design a pricing policy that maximizes the collected revenue.

In this paper, we further model the market value $v_t$ as a linear function of the product's observable feature covariate $\xb_t\in \RR^{d}$.  In particular, define $\tilde{\xb}_t = (\xb_t^\top,1)^\top$, where we assume $\{\xb_t\}_{t\ge 1}$ are i.i.d. samples from an unknown distribution $\PP_X$ supported on a bounded subset $ \cX \subseteq \RR^d$. Assume that
\begin{align}\label{eq:model}
v_t =\btheta_0^\top \tilde\xb_t + z_t,
\end{align}
{where $\btheta_0 = (\bbeta_0^\top,\alpha_0) ^\top\in \RR^{d+1}$ is an unknown parameter, and $\{z_t\}_{t\ge 1}$ is an i.i.d. sequence of idiosyncratic noise drawn from an \textbf{unknown} distribution with zero mean and bounded support $(-\delta_z,\delta_z)$. The cumulative distribution function of $z_t$ is denoted by $F(\cdot)$.
 The above model implies that
\begin{align}\label{eq:prob-model}
y_t = \begin{cases}
+1&\text{ with probability }\, 1- F\left(p_t-\btheta_0^\top \tilde{\xb}_t\right)\,,\\
0&\text{ with probability }\, F\left(p_t- \btheta_0^\top \tilde{\xb}_t\right).
\end{cases}
\end{align}}

\begin{remark}
In fact, each $\xb_t$ here can contain both product information and the buyer information, as long as this information is revealed to the seller.
\end{remark}

\begin{remark}
The reason that we assume $z_t$ has bounded support $[-\delta_z,\delta_z]$ is to ensure the market valuation $v_t\ge0$, which is more reasonable in practice (Otherwise $v_t$ has positive probability to be negative, since $z_t$ is independent with the covariates $\xb_t$). The truncated Gaussian distribution falls in such category. If the market allows $v_t$, $p_t$ to be negative, then we can replace the boundness of $z_t$ by any sub-Gaussian distributions.
\end{remark}

In a non-dynamic setting, the model \eqref{eq:prob-model} is closely related to the single index model, or generalized linear  (logistic regression) model with unknown link function \citep{ICHIMURA199371,FHW95,CFGW97}. In their works, it's usually assumed that $p_t = 0$ and $\{(\tilde \xb_t)\}_{t\ge 1}$ are independent observations, and the goal is to estimate $\btheta_0$ and $F$. Meanwhile, we work on the dynamic setting where we need to optimize some revenue function by iteratively deciding $p_t$ given previous observations based on dynamically learned $\btheta_0$ and $F$.  These two problems are closely related but also decisively different.



We now state our objective in more details.  Given observed features $\xb_t$, the expected revenue at time $t$ with a posted price $p$ is
\begin{align}\label{revenuet}
\text{rev}_t(p,\btheta_0,F):=\EE p\cdot \ind(v_t \ge p)= p(1-F(p-\btheta_0^\top \tilde{\xb}_t )).
\end{align}
The optimal posted price $p_t^{*}$  for a product with attribute $\xb_t$ is given by
\begin{align}\label{eq:pt*}
	p_t^{*}=\argmax_{p\ge 0} p(1-F(p-\btheta_0^\top\tilde{\xb}_t)),
\end{align}
which depends on unknown parameters and needs to be learned dynamically from the data.
As in common practice, we evaluate the performance of any policy $\pi$ that governs the rule of posted prices $\{p_t\}_{t\ge 1}$ by investigating the regret compared to the `oracle pricing policy' that uses the knowledge of both $\btheta_0$ and $F(\cdot)$  and offers $p_t^{*}$ according to \eqref{eq:pt*} for any given $t$. In other words,
we consider the problem of maximizing revenue as minimizing the following maximum regret  
\begin{align}\label{eq:Regret_def}
\text{Regret}_\pi(T) \equiv \max_{\substack{\btheta_0 \in \Omega\\ \PP_X\in \mathcal Q(\cX)}} \EE \left[\sum_{t=1}^T \bigg(p^*_t \ind(v_t \ge p^*_t) - p_t(\pi) \ind(v_t \ge p_t(\pi)) \bigg)\right]\,,
\end{align}
where the expectation is taken with respect to the the idiosyncratic noise $z_t$ and $\xb_t$, and $p_t(\pi)$ denotes the price offered at time $t$ by following policy $\pi$. Here $\mathcal Q(\cX)$ represents the set of probability distributions
supported on a bounded set $\cX$.
Our goal is to choose a good strategy $\pi$ such that the above total regret is small.

Apparently, learning $\btheta_0$ and $F(\cdot)$ over time gives the seller much more information to estimate the market value of a new product given it's feature covariates. On the other hand, the seller also wants to always give optimized price so as to maximize the expected revenue by \eqref{eq:pt*}. Therefore, it's necessary to have a good policy that strikes a balance between exploration (collecting data information for learning parameters) and exploitation (offering optimal pricing based on learned parameters).

Before proposing our algorithm, we first impose some regularity condition on $F$ so that the optimization problem (\ref{eq:pt*}) is 'well-behaved'.

\begin{assumption}\label{assp_inverse_0}
		There exists a positive constant $c_{\phi}$ such that $\phi'(u)\ge c_{\phi}$ for all $u\in (-\delta_z,\delta_z)$, where $\phi(u) := u-\frac{1-F(u)}{F'(u)}$.		
\end{assumption}
Assumption \ref{assp_inverse_0} ensures that $\phi(\cdot)$ is strictly increasing, which implies a unique solution to (\ref{eq:pt*}). In fact, the first order condition of \eqref{eq:pt*} yields
$$
p_t^* = g(\btheta_0^\top\tilde{\xb}_t),
$$
where $g(u)\triangleq u+\phi^{-1}(-u)$.

\begin{remark}
We only put some necessary assumptions on $F$ in order to guarantee the existence of the unique optimal price $p_t^{*}$ in \eqref{eq:pt*}, given observed $\tilde{\xb}_t$ and unknown but fixed $\btheta_0.$ Comparing to the Assumption 2.1 in \cite{JN19}, our Assumption \ref{assp_inverse_0} is weaker, since assumption that $1-F(u)$ is log-concave is a special case of our assumption with $c_{\phi}\ge1$.
\end{remark}



\section{Algorithm and Basic Regret Analysis}\label{alg_reg}
 We first propose  Algorithm \ref{pricing_1} in \S\ref{alg} which describes our policy for minimizing the regret given in \eqref{eq:Regret_def}, and then provide the main idea for the regret analysis achieved by our Algorithm \ref{pricing_1} in \S\ref{main_analysis}.
\subsection{A Proposed Algorithm}\label{alg}
In the following algorithm, we divide the time horizon into `episodes' with increasing lengths. The first part of each episode is a short exploration phase where the offered prices are i.i.d. to collect the data and model parameters (i.e. $\hat\btheta$, $\hat F$) are then updated based on the collect data. The second part is an exploitation phase, where the optimal $p_t$ is offered according to the current estimate of parameters and the new $\tilde \xb_t$. The details are stated in Algorithm \ref{pricing_1}.

\begin{algorithm}[H]
	\caption{Feature-based dynamic pricing with unknown noise distribution}	
	\label{pricing_1}
	\begin{algorithmic}[1]		
		\STATE \textbf{Input:} { Upper bound of market value ($\{v_t\}_{t\ge 1}$)}: $B>0$, minimum episode length: $\ell_0$, degree of smoothness: $m$.
		\STATE \textbf{{Initialization:}}
		$p_1=0,\,\hat\btheta_1=0.$
		\FOR{each episode $k=1,2,\dots,$}
		\STATE Set length of the $k$-th episode $\ell_k = 2^{k-1}\ell_0$; Length of the exploration phase $a_k = \lceil(\ell_kd)^{\frac{2m+1}{4m-1}}\rceil$.
		\STATE \textbf{\underline{Exploration Phase} ($t\in I_k:= \{\ell_k,\cdots,\ell_{k} + a_k-1 \}$):}
		\STATE \quad Offer price $p_t\sim \text{Unif}(0,B).$ 
		\STATE \textbf{\underline{Updating Estimates} (at the end of the exploration phase with data $\{(\tilde{\xb}_t, y_t)\}_{t\in I_{k}}$):}
		\STATE \quad Update estimate of $\btheta_0$ by $\hat\btheta_k = \hat\btheta_k(\{(\tilde{\xb}_t, y_t)\}_{t\in I_{k}})$;
			{\begin{align}\label{thetaupdate}
			\hat{\btheta}_k=\argmin_{\btheta} L_k(\btheta):= \frac{1}{|I_{k}|}\sum_{t\in I_{k}} (By_t-\btheta^\top\tilde{\xb}_t)^2
			\end{align}}
		\STATE \quad Update estimates of $F$, $F'$ by $ F_k(u,\hat\btheta_k) =  F_k(u; \hat\btheta_k, \{(\tilde{\xb}_t, y_t,p_t)\}_{t\in I_{k}})$, $ F_k^{(1)}(u,\hat\btheta_k)= F_k^{(1)}(u,\hat\btheta_k, \{(\tilde \xb_t,y_t,p_t)\}_{t\in I_{k}})$. The detailed formulas are given by \eqref{exp:f_k} and \eqref{exp:f_11}.
		\STATE \quad Update estimate of $\phi$ by $\hat\phi_k(u) =  u -\frac{1-\hat F_k (u)}{\hat F^{(1)}(u)} $	     and estimate of $g$ by $\hat g_k(u)=u+\hat\phi^{-1}_k(-u)$.
		\STATE \textbf{\underline{Exploitation Phase} ($t\in I_k':= \{\ell_{k} + a_k, \cdots, \ell_{k+1} -1 \}$): }
		\STATE \quad Offer $p_t$ as
		\begin{align}\label{pt}
	p_t=\min\{ \max\{\hat g_k(\tilde\xb_t^\top\hat\btheta_k),0\},B\}.		\end{align}
		\ENDFOR
	\end{algorithmic}
\end{algorithm}

Despite semiparametric model \eqref{eq:prob-model} with unknown link, by offering $p_t\sim \text{Unif}(0, B)$,  $By_t$ follows the linear model with regression $\tilde \xb_t^\top \btheta_0$ and this leads to the least-squares estimate  \eqref{thetaupdate}.  To see this, it follows that
\begin{equation*}
\EE [By_t \given \tilde\xb_t ] =  B \EE_{z_t} \EE[y_t \given \tilde\xb_t, z_t] = B \EE_{z_t} \EE[\ind(p_t\le \btheta_0^\top\tilde \xb_t+z_t)\given\tilde \xb_t, z_t]
=  B \EE \frac{\btheta_0^\top\tilde \xb_t + z_t}{B} = \tilde \xb_t^\top \btheta_0 .\end{equation*}
On the other hand, a uniform distribution for $p_t$ is actually critical for the above property. Suppose that $p_t$ is drawn from a c.d.f. $F_p(\cdot)$  and there is a transform $f_1$ of $y_t$ that satisfies
\begin{align*}
 \EE f_1(y_t)=\EE \tilde\xb_t^\top\btheta_0 = \EE v_t
\end{align*}
for all $\PP_X$, then according to \eqref{eq:prob-model}, we have
\begin{align*}
\EE v_t&= \EE \EE[f_1(y_t)\given \tilde \xb_t, z_t] = \EE \EE[f_1(\ind(p_t\leq \tilde \xb^\top \btheta_0 + z_t))\given \tilde \xb_t, z_t]\\
& = \EE F_p(\tilde \xb^\top \btheta_0 + z_t)f_1(1) + \EE (1 - F_p(\tilde \xb^\top \btheta_0 + z_t)) f_1(0) \\
&= f_1(0) + (f_1(1) - f_1(0))\EE F_p(v_t).
\end{align*}
Since the above equation holds for all $\PP_X\in\mathcal Q(X)$, it can only be the case that $F_p$ is linear within the region $[0, B]$, which implies that $p_t$ should follow a uniform distribution.

\begin{remark}
\revise{ In Algorithm \ref{pricing_1}, the interval $[0, B]$ can be replaced with any interval that covers the range of the market value $v_t$. In practice, we can shrink the sampling interval at each exploration phase according to the feedback information observed in the past.}
\end{remark}
\begin{remark}
{ 
If $z_t$ follows distributions with unbounded support and sub-Gaussian tails, in Algorithm \ref{pricing_1}, 
we only need to replace $B$ by $B_k=C\sqrt{\log |I_k|}$ such that $v_t$ falls in $(-B_k,B_k)$ with high probability. We then offer $p_t\sim \textrm{Unif}(-B_k,B_k)$. Conditional on $v_t\in(-B_k,B_k)$, $B_k(2y_t-1)$ serves as an unbiased estimator for $\tilde{\xb}_t^\top\btheta_0.$ Thus, all the following theoretical results work. 
}
\end{remark}

\subsection{Main Idea for Regret Analysis}\label{main_analysis}

The main idea behind our regret analysis is a balance between exploration and exploitation. This idea is shown in the following heuristic arguments. For simplicity, we assume for now that there is only one episode, and that the total length of time (horizon) $\ell$ is known and $d$ is bounded.

First, denote $\ell_1$ as the length of the exploration phase. During this phase, the regret $R_1$ at each time is bounded by a constant due to bounded distribution $F(\cdot)$ that entails bounded $p_t^*$ in \eqref{eq:pt*}. Therefore,  the total regret in this phase is
\begin{equation}\label{eq:R1}
R_1 = \cO(\ell_1).
\end{equation}
For the second phase, the expected regret can be controlled by the estimation error of both $\btheta$ and $g$ (which is a functional of $F$ as mentioned in (\ref{pt})). In fact, let the regret at each time point $t$ be
\begin{align*}
R_t:=p_t^{*}\II_{(v_t\ge p_t^{*})}-p_t\II_{(v_t\ge p_t)}.
\end{align*}
Then the conditional expectation of regret at time $t$ given previous information and $\tilde \xb_t$ is
\begin{align}
\EE[R_t\given \bar{\cH}_{t-1}]
& = \EE[p_t^{*}\II_{(v_t\ge p_t^{*})}-p_t\II_{(v_t\ge p_t)}\given \bar \cH_{t-1}]\nonumber\\
&=p_t^{*}(1-F(p_t^{*}-\tilde{\xb}_t^\top\btheta_0))-p_t(1-F(p_t-\tilde{\xb}_t^\top\btheta_0))\nonumber\\
&=\text{rev}_t(p_t^{*},\btheta_0,F)-\text{rev}_t(p_t,\btheta_0,F) \label{diffrevenue}
\end{align}
Here $\bar{\cH}_t=\sigma(\xb_1, \xb_2, \cdots, \xb_{t+1}; z_1, \cdots, z_t)$. On the other hand, under mild conditions, the above difference in revenue can further be upper bounded by an order of $(p_t - p_t^*)^2$ using Taylor expansion. Therefore, we have
\begin{align}
		\EE[R_t|\bar{\cH}_{t-1}]\lesssim (p_t-p_t^{*})^2&= (\hat g(\hat\btheta^\top\tilde{\xb}_t)-g(\btheta_0^\top\tilde{\xb}_t))^2\nonumber\\\label{diffg}&\le 2(\hat g(\hat\btheta^\top\tilde{\xb}_t)-g(\hat\btheta^\top\tilde{\xb}_t))^2+2( g(\hat\btheta^\top\tilde{\xb}_t)-g(\btheta_0^\top\tilde{\xb}_t))^2\\&:=\mathbf{J_1}+\mathbf{J_2}.\nonumber
	\end{align}
In fact, $\mathbf{J_2}$ is upper bounded by $\|\hat \btheta - \btheta_0\|_2^2$ (given the Lipschitz property of $g$ according to Assumption \ref{assp_inverse_0} and suitable conditions over $\PP_X$). By solving (\ref{thetaupdate}), we prove that the squared $\ell_2$ error is of order $\cO(\ell_1^{-1})$, which is the order of $\mathbf{J_2}$. The term $\mathbf{J_1}$ is upper bounded by $\|\hat g - g\|_\infty^2$, and is further bounded by $\max\{\|\hat F - F\|_\infty^2, \|\hat F' - F'\|_\infty^2\}$. Note that by \eqref{mymodel}, $F(\cdot)$ is the non-parametric function of $1-Y_t$ given $w_t = p_t-\tilde{\xb} _t^\top \btheta_0$, in which $p_t$ is the observed price given in the exploration phase.  Since $\btheta_0$ is estimated at a faster rate, we can assume that $w_t$ is observable given a proper estimator of $\btheta_0$.  Therefore, the error rate is dominated by estimating $F'(\cdot)$.  Assuming $F$ has an $m$-th continuous derivative, we construct $\hat g$ using the kernel estimator with a $m$-th order kernel, and prove that $\max\{\|\hat F - F\|_\infty, \|\hat F' - F'\|_\infty\}\lesssim \cO(\ell_1^{-{(m-1)}/{(2m+1)}})$ in which a logarithmic order is ignored for simplicity of presentation.  Therefore, the total regret during the exploitation phase can be upper bounded by 
\begin{equation}\label{eq:R2}
R_2 \lesssim \ell\cdot \ell_1^{-{2(m-1)}/{(2m+1)}}.
\end{equation}

Combining (\ref{eq:R1}) and (\ref{eq:R2}), we know that by choosing $\ell_1$ of the order of $\ell^{(2m+1) / (4m-1)}$, we balance the regret of both exploration and exploitation phase, and the total regret during the episode is given by
$$
R_1 + R_2 = \cO(\ell^{(2m+1) / (4m-1)}).
$$

For a second order kernel, the above regret is of order $\cO(\ell^{5/7})$. For a relatively large $m$, the regret is close to $\cO(\ell^{1/2})$, which is actually proven to be the lower bound for a wider class of problems.

\section{Regret Results on Proposed Policy}\label{sec_result}

{ In this section, we divide our results into three parts. In \S\ref{secindp}, we consider the setting with independent covariates and finite differentiable noise distributions. In \S\ref{sec:mix}, we further extend our results in \S\ref{secindp} to the setting with correlated features. Finally we extend the aforementioned results to the regime with infinitely differentiable noise distributions i.e. $m=\infty$ in \S\ref{sec_inf}.}

\subsection{Result under Independence Settings}\label{secindp}

The main result of this section is Theorem~\ref{mainthm}.  To obtain this results, we first state some technical conditions and technical lemmas, which demonstrate the accuracy of statistical learning in each episode.  These lemmas provide insights how statistical accuracy influences on the regret of our policy and have interests of their own rights.  

Assume that $\|\btheta_0\|\leq R_{\Theta}$ for some constant $R_{\Theta}>0$. We also define $R_\cX := \sup_{\xb\in\cX}\|\xb\|_2$. Before stating our main results, we first make the following assumptions on $\xb_t$.

\begin{assumption}\label{ass:bound}
There exist positive constants $c_{\min}$ and $c_{\max}$, such that the covariance matrix $\bSigma$ given by $\bSigma=\EE[\tilde\xb_t \tilde\xb_t^\top]$ satisfies $c_{\min}\II\preccurlyeq \bSigma\preccurlyeq c_{\max}\II$, where $\tilde{\xb}_t = (\xb_t^\top,1)^\top$
\end{assumption}
As we observe from $\Jb_1,\Jb_2$ given in \eqref{diffg}, bounding the regret in the exploitation phase needs to estimate both parameter $\btheta_0$ and function $g(\cdot)$. { In the following, we first present an upper bound of estimating $\btheta_0$ at the end of the exploration phase within each episode in the following Lemma \ref{thmsignal}. Recall $|I_k|$ is the length of the $k$-th exploration phase. }

\begin{lemma}\label{thmsignal}
	Under Assumption \ref{ass:bound}, there exist positive constants $c_0$ and $c_1$ depending only on {absolute constants given in assumptions} such that for any episode $k$, as long as $|I_k| \ge c_0(d+1)$, with probability  at least $1-2e^{-c_1c_{\min}^2|I_{k}|/16}-2/|I_k|$,
		\begin{align}\label{eq:thmsignal}
		\|\hat\btheta_k-\btheta_0\|_2\le \frac{8\max\{R_{\cX},1\}(R_{\cX}R_{\Theta}+B)}{c_{\min}}\sqrt{\frac{(d+1)\log |I_k|}{|I_{k}|}}.
	\end{align}
\end{lemma}
 Let $\Theta_k := B(\btheta_0, R_k)$, where $R_k$ is the right hand side of \eqref{eq:thmsignal}. We conclude from Lemma \ref{thmsignal} that with high probability, $R_k$ is of order at most $\sqrt{{d\log |I_k|}/{|I_k|}},$ and we can achieve similar upper bounds for $\Jb_2$ for any episode $k$.

Next, we proceed to construct the estimator $\hat g_k$ in each episode and  bound its distance to $g$. Notice that $g(u) = u+\phi^{-1}(-u)$, and $\phi(u) = u-\frac{1-F(u)}{F'(u)}$. Thus, a natural way to construct $\hat g_k$ is from an estimate of $F$ and $F'$, as mentioned in our algorithm. Moreover, the uniform error bounds of our estimators $\hat F_k$ and $\hat F_k^{(1)}$ guarantee a uniform error bound of $\hat g_k$.

We use the kernel regression method and $\hat\btheta_k$ obtained above to construct $\hat F_k$ and $\hat F^{(1)}_k$.   Recall that by \eqref{eq:prob-model}, we have
$E(y_t | w_t(\btheta_0)) =  1- F\left(w_t(\btheta_0)\right)$
where $w_t(\btheta):=p_t-\tilde{\xb} _t^\top \btheta$. Recall $p_t$ is the observed price offered in the $k$-th exploration phase.  Thus, given $\hat \btheta_k$, $F(\cdot)$ can be estimated by using the Nadaraya-Watson kernel regression estimator and $F'(\cdot)$ can be estimated by the derivative of the estimator.  Specifically, we define
\begin{align}
\hat F_k(u,\btheta) = 1-  \hat r_k(u, \btheta) & = 1-\frac{h_k(u,\btheta)}{f_k(u,\btheta)},\label{exp:f_k}
\end{align}
and $\hat F_k(u)  = \hat F_k(u, \hat\btheta_k)$, where
\begin{align}
h_k(u,\btheta) = \frac{1}{|I_{k}|b_k}\sum_{t\in I_k} K(\frac{w_t(\btheta)-u}{b_k})Y_t,&\qquad f_k(u,\btheta) = \frac{1}{|I_{k}|b_k}\sum_{t\in I_k} K(\frac{w_t(\btheta)-u}{b_k}),\label{eq-def-f-h}
\end{align}
for a chosen $m$-th order kernel $K$ and a suitable bandwidth $b_k$.  Now, we estimate the derivative $F'(\cdot)$ by taking the derivative of the estimator.  That is, $\hat F_k^{(1)}(u) =\hat F_k^{(1)}(u,\hat\btheta_k)$ where
\begin{align}
\hat F_k^{(1)}(u,\btheta) & = - \hat r_k^{(1)}(u, \btheta)  = -\frac{h_k^{(1)}(u,\btheta)f_k(u,\btheta)-h_k(u,\btheta)f_k^{(1)}(u,\btheta)}{f_k^2(u,\btheta)},\label{exp:f_11}\\
h_k^{(1)}(u,\btheta) &= \frac{-1}{|I_{k}|b_k^2}\sum_{t\in I_k} K'(\frac{w_t(\btheta)-u}{b_k})Y_t,\qquad f_k^{(1)}(u,\btheta) = \frac{-1}{|I_{k}|b_k^2}\sum_{t\in I_k} K'(\frac{w_t(\btheta)-u}{b_k})\label{eq-def-f-h-1} .
\end{align}

{ Recall we mention in \S\ref{sec:algo} that  $(-\delta_z,\delta_z)$ is the support of noise $z_t$. In addition, we also mentions that $T$ denotes the length of time horizon which is unknown. In the following, we will state other necessary assumptions to derive the regret upper bound:}
\begin{assumption}\label{ass-pdfx}
	The density of $w_t(\btheta)$ (denoted as $f_{\btheta}$) satisfies the following:
	\begin{itemize}
	\item (Smoothness) There exists an integer $m\geq 2$ and a constant $l_f$ such that for all $\btheta\in \Theta_{0}:=\big\{\btheta\given \|\btheta-\btheta_0\|_2\le C_{\btheta}T^{-\frac{2m+1}{4(4m-1)}}d^{\frac{m-1}{4m-1}}\sqrt{\log T+2\log d}\big\}$, $f_{\btheta}(u)\in \mathbb C ^{(m)}$, and $f_{\btheta}^{(m)}$ is $l_f$-Lipschitz on { $I:=[-\delta_z,\delta_z]$}.
		\item (Boundedness) There exists a constant $\bar f>0$ such that $\forall u\in \RR$ and $\btheta\in \Theta_{0}$, $\max\{|f_{\btheta}(u)|, |f_{\btheta}'(u)|\}\leq \bar f$. In addition, there exists a universal constant $c>0$ such that $f_{\btheta}(u)\ge c$ for all $u\in I$, $\btheta\in \Theta_{0}$.
	\end{itemize}
\end{assumption}

\begin{remark}
{We provide some examples for Assumption \ref{ass-pdfx}. For any covariate $\xb\in \RR^{d},$ as long as there exists an entry of it that follows a continuous distribution in $\CC^{(m)},$ $m\ge 1$, such as Beta-distribution or truncated Gaussian distribution, we can ensure the density of $w(\btheta)=p_t-\tilde{\xb}_t^\top\btheta$ satisfies both the smoothness and boundedness conditions in Assumption  \ref{ass-pdfx}. 
}
	\end{remark}

\begin{assumption}\label{ass-F}
$r_{\btheta}(u):=\EE [y_t\given w_t(\btheta)=u]$ satisfies the following:
\begin{itemize}
{\item (Smoothness) $h_{\btheta}(u) = f_{\btheta}(u)r_{\btheta}(u)\in \mathbb C ^{(m)}$; $h_{\btheta}^{(m)}$ is $l_f$-Lipschitz on { $I$}
for all  $\btheta\in \Theta_{0}$. Here $m$ and $l_f$ are defined in Assumption \ref{ass-pdfx}.}
\item (Lipschitz) There exists a constant $l_r$ such that $r_{\btheta_0} = 1-F$ is $l_r$-Lipschitz, and for any $\epsilon >0$, $\sup_{\|\btheta-\btheta_0\|_2\le \epsilon,u\in I}|r_{\btheta}'(u)-r_{\btheta_0}'(u)|\le l_r\epsilon$.

\end{itemize}
\end{assumption}
\begin{assumption}\label{ass-kernel}
	The kernel $K$ satisfies the following:
	\begin{itemize}
		\item (Order-$m$ kernel)$\int K(s) \ud s = 1$, $\int s^jK(s) \ud s = 0$ for $j\in\{1, \cdots, m-1\}$, and that $\int |s^mK(s)|\ud s <+\infty$. Here $m$ is the same as in Assumption \ref{ass-pdfx}.
		\item (Lipschitz) Both $K(s)$ and $K'(s)$ are $l_K$-Lipschitz continuous with bounded support.
	\end{itemize}
\end{assumption}
The Assumptions \ref{ass-pdfx}-\ref{ass-kernel} are quite standard assumptions in non-parametric statistics; see \cite{fan1996local,T08} for more details.
Given these assumptions, we will prove that with high probability, the estimators $\hat F_k(u,\btheta)$ and $\hat F_k^{(1)}(u, \btheta)$ are sufficiently close to { $F(u)$ and $F'(u)$} respectively given any $\btheta\in \Theta_0$ for every sufficiently large $k$. Specifically, we obtain the desired error bound for $\hat F_k(u) = \hat F_k(u,\hat\btheta_k)$ and $\hat F_k^{(1)}(u) = \hat F_k^{(1)}(u,\hat\btheta_k)$.
\begin{remark}\label{remark-m}
\revise{ Assumptions \ref{ass-pdfx} and \ref{ass-F}    can be relaxed in terms of the smoothness requirements: For all $m\ge 3$, we only need $f_{\btheta}(u),h_{\btheta}(u)\in \CC^{(m-1)}$, and that $f_{\btheta}^{(m-1)}(u),h_{\btheta}^{(m-1)}(u)$ are $\ell$-Lipschitz for some constant $\ell$. For $m=2$, we only need  $f_{\btheta}(u),h_{\btheta}(u)\in \CC^{(1)}$, and that the second order derivatives of $f_{\btheta}(u),h_{\btheta}(u)$ exist and are bounded. One is able to see assuming functions in $\CC^{(m)}$ is a sufficient condition for the aforementioned conditions to hold, for the simplicity of our notations here, we keep the original assumptions.}
\end{remark}
{\begin{remark}
If we only assume $F(\cdot)$ is $\ell$-Lipschitz continuous (i.e. it may not be differentiable), we also provide an alternative algorithm in \S{\red F} which achieves a regret upper bound $\tilde \cO(T^{3/4})$. 
	\end{remark}}
\begin{remark}
One is also able to estimate $F(u),F'(u)$ with the local polynomial estimator (see e.g. \cite{fan1996local}). In this case, the assumptions can be weaken further.  Specifically, the local polynomial estimators for $F$ and $F'$ enjoy all the theoretical guarantees given only the second part of Assumptions \ref{ass-pdfx} and \ref{ass-kernel} instead of both Assumptions \ref{ass-pdfx} and \ref{ass-kernel}. For example, Lipschitz continuous density functions  on $[-\delta_z,\delta_z]$ satisfy Assumption {\red 4.2}. 
The proof is very similar. For simplicity, we only focus on studying kernel regression in this paper.
\end{remark}

\begin{lemma}\label{conckernel_main}
Under Assumptions \ref{ass-pdfx}, \ref{ass-F} and \ref{ass-kernel}, there exist constants $B_{x, K}$, $B'_{x, K}$ and $C_{x, K}$ (depending only the absolute constants within the assumptions) such that as long as
$$
T\geq B_{x, K} (\log T +2\log d)^{\frac{4m-1}{m}} d^{\frac{2m-1}{m}},
$$
we have for any $k\geq \lfloor(\log (\sqrt{T} + \ell_0) - \log \ell_0) / \log 2\rfloor + 2$ and
{$\delta \in[4\exp(-B'_{x, K}|I_k|^{\frac{2m}{2m+1}}/\log |I_k|),\frac{1}{2}],$ }
with probability at least $1-2\delta$,
	\begin{align}
\sup_{u\in I,\btheta\in \Theta_{k}}& |\hat F_k(u, \btheta) - F(u)| \leq
C_{x, K}|I_K|^{-\frac{m}{2m+1}}\sqrt{\log |I_K|}(\sqrt{d} + \sqrt{\log \frac{1}{\delta}}) \label{eq:conckernel_main}.
	\end{align}
	Here $I=[-\delta_z,\delta_z]$ and we choose the bandwidth $b_k = |I_k|^{-\frac{1}{2m+1}}$.
\end{lemma}

\begin{lemma}\label{conckernel2_main}
Under the same conditions as Lemma~\ref{conckernel_main}, 
with probability at least $1-4\delta$, we have
\begin{align}
\sup_{u\in I,\btheta\in \Theta_{k}}& |\hat F_k^{(1)}(u, \btheta) - F'(u)| \leq
\tilde C_{x, K}|I_K|^{-\frac{m-1}{2m+1}}\sqrt{\log |I_K|}(\sqrt{d} + \sqrt{\log \frac{1}{\delta}}) \label{eq:conckernel2_main}.
\end{align}
\end{lemma}

We next develop a uniform upper bound for term $\Jb_1$ given in \eqref{diffg} for the $k$-th episode in Lemma \ref{inverse_conv_main} below.

\begin{lemma}\label{inverse_conv_main}
Reinstating the notations and conditions in Lemma~\ref{conckernel_main}, 
with probability at least $1-6\delta$, we have
$$
\sup_{u\in [\delta_z,B-\delta_z] } |\hat g_k(u) - g(u)|\leq \bar C_{x, K}|I_K|^{-\frac{m-1}{2m+1}}\sqrt{\log |I_K|}(\sqrt{d} + \sqrt{\log \frac{1}{\delta}}).
$$
\end{lemma}
\begin{remark}
In Algorithm \ref{pricing_1} we define $\hat g_k(u) = u + \hat \phi_k^{-1}(-u)$ with $u\in [\delta_z,B-\delta_z]$. Thus, computing $\hat g_k(u)$ involves obtaining the inverse of $\hat \phi_k$, which is not necessarily monotonic. Nevertheless, it's not difficult to define or compute $\hat \phi_k^{-1}$. In fact, we'll show in the proof of Lemma \ref{inverse_conv_main} that $\hat \phi_k$ is very `close' to $\phi$ in some main interval of interest, which contains $[\phi^{-1}(\delta_z-B),\phi^{-1}(-\delta_z)]$ and depends only on $F$. (Recall in Assumption \ref{assp_inverse_0} that $\phi'$ is bounded below from 0, so $\phi$ is strictly increasing). Thus, for any $u\in[\delta_z,B-\delta_z]$, the above fact will guarantee the existence of $\hat \phi_k^{-1}(-u)$ as some $x$ within the interval such that $\hat \phi_k(x) = -u$.
\end{remark}

Combining the above lemmas, which give us upper bounds for terms $\Jb_1,\Jb_2$ in every episode, we have the following Theorem \ref{mainthm}, which provides an upper bound for the regret. 
\begin{theorem}\label{mainthm}
Under Assumptions \ref{assp_inverse_0}, \ref{ass-pdfx}, \ref{ass-F} and \ref{ass-kernel}, there exist constants $\bar B_{x, K}$, $\bar B'_{x, K}$ and $C^*_{x, K}$ (depending only on  the absolute constants within the assumptions) such that for all $T$ satisfying
\textbf{ $$
T\geq \max\{\bar B_{x, K} (\log T +2\log d)^{\frac{4m-1}{m-1}} d^{\frac{2m+1}{m-1}}, 4d^{\frac{2m+1}{m-1}}\},
$$}
the regret of Algorithm \ref{pricing_1} over time $T$ is no more than $C^*_{x, K}(Td)^{\frac{2m+1}{4m-1}}\log T(1+\log T/d)$.
\end{theorem}

\begin{remark}
{We note that \cite{GJM20} shares a similar framework with ours, although with a different regret measure. Specifically, we use a more traditional notion of regret by setting the benchmark $p_t^{*}$ from \eqref{eq:pt*} with true $\btheta_0$ and $F(\cdot)$. In \cite{GJM20}, the authors instead set the benchmark $p_t^{*}$ so as to maximize the worst function in their function class $\cF$, i.e.
$$p_t^{*}=\argmax_{p\ge 0}\min_{F\in\cF} p(1-F(p-\btheta_0^\top\tilde{\xb}_t)).$$
Their optimal regret is of order $\tilde{\cO}_{d}(T^{2/3})$, while ours is $\tilde{\cO}_{d}(T^{\frac{2m+1}{4m-1}})$, which is closer to ${\cO}_{d}(T^{1/2})$ when $m$ is sufficiently large. Intuitively, a benchmark being the price maximizing the worst function is too conservative when their ambiguity function class is very large and the market noises are only sampled from a fixed distribution function in that function class, which is true in our semi-parametric setting.

On the other hand, \cite{Golrezaei2019} also work on similar but simpler settings, where they assume having unknown demanding curves but observable valuations instead of censored responses. By contrast, we work on a more common setting where the actual market values of products are unknown.}
\end{remark}

\revise{
\begin{remark}\label{rem-m}
Both Algorithm \ref{pricing_1} and Theorem \ref{mainthm} depend on the smoothness class of the function $F(\cdot)$. A popular choice in nonparametric curve estimation literature is  $m=2$, as other choices do not improve much for practical sample sizes.  Nevertheless, we provide two ways to choose $m$ that addresses a referee's query.
\begin{itemize}
    \item \textbf{Estimate $m$ using cross-validation.} Specifically, we pick some relatively small $m$ during the first episode. At each episode $k\geq 2$, before entering the exploration phase, we update the estimate of $m$ using cross-validation \citep{HALL2015} with the data gathered from the previous exploration phase. Then, we proceed with the main algorithm with this updated estimate until the next episode.
    For more details of the cross-validation procedure and the combined algorithm, see Section \ref{secH}. 
    \item \textbf{Pick a constant pessimistic estimation of $m$.} In fact, we can directly fix a relatively small $m$ (e.g. $m=2$ or $m=4$). In many cases, the performance of the algorithm ($ \tilde{\cO}((Td)^{5/7})$ and $\tilde \cO((Td)^{3/5})$) will not be significantly different from where $m$ is known (at least $\Omega((Td)^{1/2})$).
\end{itemize}
The above two ways can be applied to all settings in this paper as long as $F$ is only required to be smooth to a finite degree.
\end{remark}
}

\subsection{Results under the setting with strong-mixing features}\label{sec:mix}

As mentioned in the introduction, we believe that in many situations, the dependence of features over time is inevitable. Thus, in this section, we generalize our results to the case where $\xb_t$ can be dependent. For this purpose, we first impose the strong-mixing condition which measure the dependence between covariates over time.

\begin{definition}\label{defbeta}[$\beta$-mixing] For a sequence of random vectors $\xb_t\in\RR^{d\times 1}$ on a probability space $(\Omega,\cX,\PP)$, define $\beta$-mixing coefficient
	\begin{align*}
		\beta_k=\sup_{l\ge 0}\beta(\sigma(\xb_t,t\le l),\sigma(\xb_t,t\ge l+k))
	\end{align*}
	in which
	\begin{align*}
		\beta(\cA,\cB)=\frac{1}{2}\sup\Big\{\sum_{i\in I}\sum_{j\in J}|\PP(A_i\cap B_j)-\PP(A_i)\PP(B_j) |\Big\},
	\end{align*}
	the maximum being taken over all finite partitions $(A_i)_{i\in I}$ and $(B_i)_{i\in J}$ of $\Omega$ with elements in $\cA$ and $\cB$.
\end{definition}

The following assumption ensures that $\{\xb_t\}_{t\ge 1}$ are not too strongly dependent. { Combining with other assumptions, we ensure that the empirical covariance matrix $\frac{1}{n}\sum_{i=1}^{n}\tilde{\xb}_i\tilde{\xb}_i^\top$ concentrate around the population version, which is necessary in  deriving the regret in every episode.}

\begin{assumption}  \label{assmix1} The sequence $\xb_t,t\ge 0$ are strictly stationary time series and follow $\beta$-mixing condition, in a sense we assume that $\beta_k\le e^{-ck}$ holds with some constant $c$.
\end{assumption}

{ In order to derive the final regret upper bound under the stong-mixing setting, we also need an additional technical assumption stated below:}

\begin{assumption}\label{lipschitz_mix}
Let $r_{\btheta}(u_i,u_j):=\EE[y_iy_j\given w_j(\btheta)=u_j,w_i(\btheta)=u_i],\,j>i\ge 0$, $r_{\btheta}(u_j):=\EE[y_j\given w_j(\btheta)=u_j],j\ge 0$ be the joint regression function and marginal regression function. In addition, we also set $f_{\btheta}(u_i,u_j),\,j>i\ge 0$, $f_{\btheta}(u_i),i\ge 0$ as the joint density of $w_i(\btheta)$ and $w_j(\btheta)$ and marginal density of $w_i(\btheta)$ respectively. Then we define $g_{1,\btheta}(u_i,u_j):=r_{\btheta}(u_i,u_j)f_{\btheta}(u_i,u_j)-r_{\btheta}(u_i)f_{\btheta}(u_i)r_{\btheta}(u_j)f_{\btheta}(u_j)$ and $g_{2,\btheta}(u_i,u_j)=f_{\btheta}(u_i,u_j)-f_{\btheta}(u_i)f_{\btheta}(u_j)$. We assume $g_{1,\btheta}(u_i,u_j)$ and $g_{2,\btheta}(u_i,u_j)$ follow $l$-Lipschitz continuous condition, in a sense that
	\begin{align*}
		|g_{q,\btheta}(u_i,u_j)-g_{q,\btheta}(u_i',u_j')|\le l\sqrt{(u_i-u_i')^2+(u_j-u_j')^2},\,q\in\{1,2\}
	\end{align*}
	holds for all $(u_i,u_j)$, with $i,j\in[n]$ and $\btheta\in\Theta_0$.
\end{assumption}
{ When the covariates $\xb_i,\xb_j$  are independent, we have $g_{q,\btheta}(u_i,u_j)=0,q\in\{1,2\}$, for all $(u_i,u_j)$. Under such a mild assumption, we obtain a uniform upper bound of $|g_{q,\btheta}(u_i,u_j)|$, which is dominated by the $\beta$-mixing constant $\beta_{j-i}^{1/3}$, for all $\btheta\in\Theta_0$ and $(u_i,u_j)$ (see Appendix \ref{pf-b3}). Thus, this assumption essentially guarantees that the joint regression and density functions of the features still stay close to the products of their marginal ones even if they are correlated.} 


{ Following similar analysis with \S\ref{secindp}, we reach the following theorem which gives a regret upper bound at similar rate with Theorem \ref{mainthm} under the strong-mixing feature setting.}
\begin{theorem}\label{mainthm2} { Let Assumptions \ref{assp_inverse_0}, \ref{ass-pdfx}, \ref{ass-F}, \ref{ass-kernel}, \ref{assmix1} and \ref{lipschitz_mix} hold. Then there exist constants $ B^*_{mx, K}$ and $C^* _{mx, K}$ (depending only on   the absolute constants within the assumptions) such that for all $T$ satisfying
$$
T\geq \max\{B^{*}_{mx, K} (\log T +2\log d)^{\frac{12m-3}{m-1}} [(d+1)\log (d+1)]^{\frac{4m-1}{m-1}}/d^2,d^{\frac{2m+1}{m-1}}\}
$$
	the regret of Algorithm \ref{pricing_1} over time $T$ is no more than $C^*_{mx, K}(Td)^{\frac{2m+1}{4m-1}}\log^4 T$.}
\end{theorem}

\subsection{Result on infinitely differentiable market noise distribution}\label{sec_inf}

In \S\ref{secindp} and \S\ref{sec:mix}, we analyze the regret upper bounds when the noise distribution $F$ has an $m$-th order continuous derivative, with any finite $m\ge 2$. The regret of our algorithm is of order $\tilde \cO((Td)^{\frac{2m+1}{4m-1}})$, which gets closer to $\tilde{\cO}(\sqrt{Td})$ as the degree of smoothness $m$ goes to infinity. In fact, this is mainly due to inaccurate estimation of $F$ and $F'$ resulting from the bias of the kernel estimator. In this section, we deal with super smooth noise distributions \citep{fan1991optimal}, where $F$ is infinitely differentiable. Under mild conditions, we're able to control the bias within $\cO(1/{|I_k|^\frac{1}{2}})$ for each episode $k$ by using extremely smooth kernels. {  As a reminder, here $|I_k|$ is the length of the $k$-th exploration phase. }This leads to a $\tilde \cO_{d}(\sqrt{T})$ regret bound in our algorithm. In particular, we assume the following:

\begin{assumption}\label{ass_inf}
	 Define $\phi_{\btheta}$, $\xi_{\btheta}$, $\phi_{\btheta}^{(1)}$ and $\xi_{\btheta}^{(1)}$ as the  Fourier transform  of the function $f_{\btheta}$, $h_{\btheta}$, $f'_{\btheta}$ and $h'_{\btheta}$ respectively:
	\begin{align*}
		&\phi_{\btheta}(s)=\int_{-\infty}^{\infty}f_{\btheta}(x)e^{isx} \ud x,\,\,\xi_{\btheta}(s)=\int_{-\infty}^{\infty}h_{\btheta}(x)e^{isx} \ud x,\\
		&\phi_{\btheta}^{(1)}(s)=\int_{-\infty}^{\infty}f'_{\btheta}(x)e^{isx} \ud x,\,\,\xi_{\btheta}^{(1)}(s)=\int_{-\infty}^{\infty}h'_{\btheta}(x)e^{isx} \ud x,
	\end{align*}
	and $h_{\btheta}(x)=f_{\btheta}(x)r_{\btheta}(x)$.
		There exist positive constant $D_{\phi}$ and $d_{\phi}$ and $\alpha>0$ such that
		$$\max\{|\phi_{\btheta}(s)|, |\xi_{\btheta}(s)|, |\phi_{\btheta}^{(1)}(s)|, |\xi_{\btheta}^{(1)}(s)|\}\le D_{\phi}e^{-d_{\phi}|s|^{\alpha}}$$
		for  all $s\in \RR$.
\end{assumption}
\begin{remark}
-This assumption is quite standard, and ensures that $f_{\btheta}(u),\,F_{\btheta}(u)\in \CC^{\infty}$. The class of functions are still infinite dimensional nonparametric functions.
  The class of supersmooth functions has been used in non-parametric density literature.  In particular, it has been used in \cite{fan1991optimal} for characterizing the difficulty of non-parametric deconvolution.
\end{remark}

Under the Assumption of \ref{ass_inf}, for each episode $k$, we can successfully control the bias within $\cO(1/\sqrt{|I_k|})$ via an infinite order kernel \citep{MP04,Berg2009CDFAS}. In order to construct an infinite order kernel $K$, we simply let $K$ be the Fourier inverse transform of some `well-behaved' function. In particular, let
\begin{align}
	K(x)=\frac{1}{2\pi}\int_{-\infty}^{\infty} \kappa(s)e^{-isx}\ud s, \label{eq-K-inf}
\end{align}
be the Fourier inversion of $\kappa$ satisfying
$$ \kappa(s)=\left\{
\begin{array}{lll}
1 ,    &     & |s|\le c_{\kappa}\\
g_{\infty}(|s|),   &        & \textrm{otherwise.}
\end{array} \right. $$
Here $g_{\infty}$ is any continuous, square-integrable function that is bounded in absolute value by $1$ and satisfies $g_{\infty}(|c_{\kappa}|)=1$.
This  defines an infinity order kernel function \citep{fan1996local}.

By plugging the infinite order kernel $K$ into our algorithm, we're able to obtain the following lemma:
\begin{lemma}\label{inf_bias}
	Under Assumption \ref{ass_inf}, there exists a positive constant $C_{\inf}$ depending only on $\alpha$, $D_{\phi}$ and $d_{\phi}$ such that for all kernel $K$ satisfying \eqref{eq-K-inf}, for each episode $k$, by choosing the bandwidth $b_k = c_{\kappa}(d_\phi/\log |I_k|)^{1/\alpha}$ in \eqref{eq-def-f-h} and \eqref{eq-def-f-h-1}, we have
	\begin{align*}
		\sup_{u\in I, \btheta\in \Theta_k}	|\EE[f_k(u,\btheta)]-f_{\btheta}(u)|&\le \frac{C_{\inf}}{\sqrt{|I_k|}}, \quad
		\sup_{u\in I, \btheta\in \Theta_k}	|\EE[h_k(u,\btheta)]-h_{\btheta}(u)|\le \frac{C_{\inf}}{\sqrt{|I_k|}}, \\
		\sup_{u\in I, \btheta\in \Theta_k}	|\EE[f_k^{(1)}(u,\btheta)]-f'_{\btheta}(u)|&\le \frac{C_{\inf}}{\sqrt{|I_k|}}, \quad
		\sup_{u\in I, \btheta\in \Theta_k}	|\EE[h_k^{(1)}(u,\btheta)]-h'_{\btheta}(u)|\le \frac{C_{\inf}}{\sqrt{|I_k|}}.
	\end{align*}
\end{lemma}

Following similar proof procedures of Theorems \ref{mainthm} and \ref{mainthm2}, Lemma \ref{inf_bias} leads to the following theorem, which gives a regret upper bound of $\tilde {\cO}_{d}(\sqrt{T})$, achieving the same convergence rate with the parametric case up to logarithmic terms \citep{JN19}.

\begin{theorem}\label{mainthm3} Let Assumptions \ref{assp_inverse_0}, \ref{ass-pdfx}, \ref{ass-F}, \ref{ass-kernel}, \ref{assmix1}, \ref{lipschitz_mix} and \ref{ass_inf} hold. Then there exist constants $ B^*_{\inf}$ and $C^* _{\inf}$ (depending only on  the absolute constants within the assumptions) such that by choosing $|I_k| = \lceil \sqrt{l_k d}\rceil$ instead in Algorithm \ref{pricing_1}, for all $T$ satisfying
$$
T\geq B^{*}_{\inf} d^2(\log T +2\log d)^{12+12/\alpha} \log^4 (d+1),
$$
	the regret of the algorithm over time $T$ is no more than $C^*_{\inf}(Td)^{\frac{1}{2}}(\log T)^{\frac32+\frac{3}{2\alpha}}[\log (d+1) + \log T/d]$.
\end{theorem}

\begin{remark}
{ Theorem \ref{mainthm3} partly overturns the conjecture  in \cite{SJB2019} that there is no policy can achieve an $\tilde{\cO}_d(\sqrt{T})$ regret under the setting where the market value is linear in the features as in \eqref{eq:model}. We provide a regime with super smooth market noise in which $\tilde{\cO}_{d}(\sqrt{T})$ regret upper bound is attainable by our policy. }
\end{remark}


\subsection{Extension: High-dimensional Feature-based Dynamic Pricing}
\revise{
Algorithm \ref{pricing_1} can be naturally extended to the high-dimensional setting, where $\btheta_0\in \mathbb R^d$, $d$ can be large compared to $T$, while $\|\btheta_0\|_0\leq s$ for a relatively small sparsity $s$. This happens in applications when a large amount of covariate information is available, and the actual market value only depends on some essential factors. One way of extension is the following: at each episode, we can replace estimation of $\hat\btheta_k$ in \eqref{thetaupdate} with the two steps below.}

\revise{\noindent\textbf{Step 1.} Let 
\begin{align}\label{thetaupdate_highd1}
\tilde{\btheta}_k=\argmin_{\btheta} L_k(\btheta) + \lambda p(\btheta),
\end{align}
where 
$$
L_k(\btheta):= \frac{1}{|I_{k}|}\sum_{t\in I_{k}} (By_t-\btheta^\top\tilde{\xb}_t)^2, \quad 
p(\btheta) = \sum_{j=1}^p p(|\btheta^{(j)}|)
$$
for some penalty function $p(\cdot)$. As in \cite{zhao2006model, fan2001variable, mcp2010}, by choosing different $p(\cdot)$ such as in the $\ell_1$, SCAD or MCP penalty, under suitable conditions such as irrepresentable condition, variable selection consistency is achieved with high probability. }

\revise{\noindent\textbf{Step 2.} Let $\hat S_k = \text{supp}(\tilde \btheta_k)$, we then refit the least squares \eqref{thetaupdate} on $\hat S_k$:
\begin{align}\label{thetaupdate_highd2}
\hat{\btheta}_k=\argmin_{\supp(\btheta)\subseteq \hat S_k} L_k(\btheta).
\end{align}
Then the conclusions of Lemma \ref{conckernel_main} hold with high probability. }

\revise{After Step 2, we continue the remaining steps of Algorithm \ref{pricing_1} in the episode. In fact, if we can learn the support of $\btheta_0$, we essentially translate the problem into a low-dimensional one, and we can prove that Algorithm \ref{pricing_1} achieves a regret upper bound of $\Tilde \cO((Ts)^{\frac{4m+1}{2m-1}})$ if $F\in \mathcal C^{(m)}$ (or $\cO((Ts)^{1/2})$ if $F$ is super smooth).}

\section{Discussion }\label{lowerbound}

\begin{itemize}
    \item[1.] [Minimax Lower Bound] Our work shares a similar setting with \cite{BR12}, in which they study a general choice model with parametric structure and binary response, but without any covariates. A lower bound of order $\Omega(\sqrt{T})$ is established by constructing an `uninformative price' in their work. To be more precise, an uninformative price is a price that all demand curves (probability of successful sales) as offered price indexed by unknown parameters intersect. Namely, the demands at this uninformative price are the same for all unknown parameters. In addition, such price is also the optimal price with some parameters. In this case, the price is uninformative because it doesn't reveal any information on the true parameter. Intuitively, if one tries to learn model parameters, the only way is to offer prices that are sufficiently far from the uninformative price (optimal price) which leads to a larger regret.

Borrowing the idea from \cite{BR12} and \cite{JN19}, we deduce that there exists an `uninformative price' in the following class of models: Consider a class of distributions $\cF$ which satisfies Assumption \ref{assp_inverse_0}:
\begin{align*}
 \cF:=\{F_{\sigma}: \sigma>0,F_{\sigma}=F(x/\sigma) \}.
\end{align*}
{Here, $F$ is the c.d.f. of a known distribution with mean zero}. Moreover, we assume the support of $F_{\sigma}'$ is contained in $[-a,a]$ (For instance, the class of distributions with density $f_{\sigma}(x)=4/(3\sigma^3)(\sigma-x)^k(\sigma+x)^k \cdot\II_{\{|x|\le \sigma\}},k\ge 1$  or $f_{\sigma}(x)=C_{\sigma}\exp\Big(-\frac{\sigma^2}{\sigma^2-x^2}\Big)\cdot\II_{\{|x|\le \sigma\}}$ with $\sigma\le a$ etc.)

Let $\beta=1/\sigma$ and multiply $\beta$ on both sides of \eqref{eq:model}, which leads to
\begin{align*}
	\tilde{v}(\xb_t)=\tilde{\bbeta}_0^\top\xb_t+\tilde{\alpha}_0+\tilde{z}_t.
\end{align*}
Here, $\tilde{v}_t=\beta v_t,\tilde{\bbeta}_0=\beta\bbeta_0,\tilde{\alpha}_0=\beta\alpha_0$ and $\tilde{z}_t=\beta z_t$. The distribution of $\tilde{z}_t$ is $F_{1}$, which is denoted as $F$ here for convenience. Next, in our sub-parameter class, we first let ${\bbeta}_0=0$ and fix a number $\xi$ with $F'(\xi)\neq 0$. Then we choose a collection of $\{(\sigma,\alpha_0)\}$ which satisfies $\beta=1/\sigma=(\xi+\tilde{\alpha}_0)$. Following the same arguments as in \cite{JN19},  one can prove that $p=1$ is indeed an uninformative price. Since in the sub-parametric class given above, all demand curves intersect at a point $1-F(\xi)$ when $p=1$, and for a special $(\sigma,\alpha_0)=(1/(\xi-\phi(\xi)),-\phi(\xi)/(\xi-\phi(\xi))$, $p=1$ is the optimal price. Thus the $\Omega(\sqrt{T})$ lower bound applies.

\begin{remark}
\revise{When we only consider explore-then-commit algorithms and offer price as $p_t=\hat\phi_k^{-1}(-\xb_t^\top\hat\theta)+\xb_t^\top\hat\theta,$ with $\hat\phi_k(u)=u-\frac{1-\hat F_k(u)}{\hat F^{(1)}(u)}$, the optimality of $p_t$ reduces to the optimality of estimating $F(\cdot),f(\cdot)$ and $\btheta$. According to \cite{Stone80,Stone82,T08}, the statistical rates of our estimators on $\hat F,\hat F^{(1)}$ and $\hat\btheta$ are minimax optimal in every episode. Thus, our posted price is optimal constrained on this type of policies. However, if we consider a general policy class, there is currently no lower bound for  feature-based pricing given unknown noise distribution with finite smoothness degree besides the general $\sqrt{T}$ lower bound mentioned above. It remains an open problem whether our upper bound is tight for finite $m$.}
\end{remark}
\item [2. ]\revise{[The adversarial setting] We note that in some real applications with potentially adversarial contexts, the covariance of the feature vectors might be singular or ill-conditioned (e.g. due to repeated buyers recorded in $\xb_t$). However, our algorithm can be adjusted to cope with such situations. 
    The key observation here is that this assumption is \emph{only} required in our exploration phase: For any $k$, we allow arbitrary $\xb_t$ in the $k$-th exploitation phase, since we have already obtained accurate estimators $\hat\btheta_k$ and $\hat g_k(\cdot)$ for $\btheta_0$ and $g(\cdot)$. Therefore, whenever there is a sign of a repeated buyer, we can modify our algorithm slightly by using the $\hat g_{k-1}(\cdot)$ in the last episode to offer a price, and then move this buyer to the corresponding exploitation phase. If the number of similar buyers in the $k$-th episode is $\ell_k^{r}$ with any $r<1$ and we assume the remaining buyers are sampled i.i.d. from a distribution, we are still able to proceed by only arranging some contexts with similar buyers into the exploitation phase directly. This matches with some real situation in online shopping where personal preference features will be recorded by the seller in order to make recommendation in the future. }
\item [3.] \revise{[Online inference of the demand]
Recently, \cite{WCCG2020} use a de-biased approach to quantify the uncertainty of the demand function in a parametric class which offers new insight to the field of statistical decision making. \\
In our work, we combine the non-parametric statistical estimation and online decision making to derive a policy that maximize the seller's revenue. We next also briefly discuss our intuition on depicting the uncertainty of the demand curve in a non-parametric class. Recall the demand curve given in \eqref{revenuet}. 
For given $p,\xb,$ and estimators $\hat F_k,\hat\btheta_k$, in the $k$-th exploitation phase, deriving asymptotic behavior of the demand curve reduces to deriving the asymptotic behavior of our estimator on $\hat F_k(\cdot).$ This is due to the statistical rate of $\hat F_k(\cdot)$ dominates that of $\hat\btheta_k$. According to asymptotic behavior of the kernel regression \citep{fan1996local,CFGW97,Fan1998}, we have the following pointwise confidence interval for $\hat F:$
    \begin{align*}
        \sqrt{|I_k|h_k}(\hat F_k(u)-F(u)-h_k^m\kappa_mB(u))\rightarrow N\Big (0, \int K^2(x)\ud x\sigma^2(u)/f(u)\Big ),
    \end{align*}
where $f(\cdot)$ is the density of $p_t-\xb_t^\top\theta_0$ with $p_t\sim$ Unif$(0,B)$ and we recall that $|I_k|$ is the length of our $k$-th exploration phase. In addition, $\kappa_m=\int K(x)x^m\ud x$, $B(u)=F^{(m)}(u)f(u)/m!+F^{(m-1)}(u)f^{(1)}(u)/(m-1)!+\cdots +F^{(1)}(u)f^{(m-1)}(u)/(m-1)!,$ and $\sigma^2(u)=\Var(y_t\given p_t-\xb_t^\top \btheta_0=u).$
Thus, for any given $p,\xb,$ and an $\hat\btheta_k,$ we are able to derive the pointwise asymptotic behavior of our demand curve as follows: 
    \begin{align*}
        &\sqrt{|I_k|h_k}(p\hat F_k(p-\xb^\top\hat\btheta_k)-pF(p-\xb^\top\btheta_0)-ph_k^m\kappa_mB(p-\xb^\top\btheta_0))\\&\qquad\rightarrow N\bigg (0, p^2\int K^2(s)\ud s\sigma^2(p-\xb^\top\btheta_0)/f(p-\xb^\top\btheta_0)\bigg ).
    \end{align*}
The data-driven confidence interval for our demand curve given in   \eqref{eq:pt*} 
can be established via bootstrap and the undersmoothing technique (to remove the bias), see e.g. \cite{Hall1992,Horowitz2001} for more details.
Similarly, uniform statistical inference results can also be established by using similar non-parametric tools, see e.g. \cite{Eubank1993,Neumann1998,Hall2013} for more details.
We will leave the detailed proof for future work.}
    \item[4.] {In some situations, it might be difficult for retailers to adopt a uniform pricing strategy even during a short period of time. An alternative strategy might be the following: As in Algorithm \ref{pricing_1}, we divide the time horizon into episodes according to the doubling strategy. However, now we no longer divide an episode into explore then exploitation phases. Instead, at the beginning of each episode $k>1$, we leverage all the data $\{p_t, \xb_t, y_t\}$ collected from the previous episode to estimate $\btheta_0$ and $F$. Then, we compute $\hat g_k$ from the estimates $\hat F_k$ and $\hat F_k^{(1)}$, and perform exploitation directly throughout this episode. This procedure can help us to get rid of uniform exploration in practice. We leave the theoretical guarantees for this refined algorithm as our future work.}
\end{itemize}

\section{Simulations}\label{sec_simu}
\subsection{Justification of theoretical results}\label{justification_theory}
In this section, we illustrate the performance of our policy through large-scale simulations under various settings. Recall our model \eqref{eq:model}, where $\xb_t\in \RR^{d}$ and $z_t$ follows distributions with bounded support and smooth c.d.f. Throughout this section, we let the dimension $d=3$ and the coefficients $\alpha_0=3$, $\bbeta_0=\sqrt{2/3}\cdot \mathbf{1}_{3\times 1}$. For each value of smoothness degree $m\in \{2,4,6\}$, we fix a density function from $\mathbb C^{(m-1)}$ for all $z_t$ (thus the c.d.f. $F$ belongs to $\CC^{(m)}$). Specifically, we set the p.d.f. of $z_t$ as $f_m(x)\varpropto({1}/{4}-x^2)^{m/2}\cdot\II_{\{|x|\le 1/2\}}$ for $m\in \{2,4,6\}$. Moreover, for each $m$, the covariates $\xb_t\in\RR^{3}$ are generated from a p.d.f. in $\CC^{(m)}$ in the following ways:

\begin{itemize}
\item \textbf{i.i.d. $\xb_t$ with independent entries:} Each coordinate of $\xb_t$ is generated from density $f_m(x)\varpropto(2/3-x^2)^{m+1}\cdot\II_{\{|x|\le \sqrt{2/3}\}}$.
\item \textbf{i.i.d. $\xb_t$ with dependent entries:} $\xb_t$ is generated from the density function $f_m(\xb)\varpropto(1-\xb^\top\bSigma^{-1}\xb)^{m+1}.$ Here $\bSigma$ is a positive definite matrix with $(i,j)$-th entry being equal to $0.2^{|i-j|},1\le i,j\le 3.$
\item \textbf{Strong mixing $\xb_t$ with dependent entries:} We generate $\xb_t$ from the VAR (vector autoregression) model, where $\xb_t=\Ab \xb_{t-1}+\Bb\xb_{t-2}+\bxi_t$. Here $\Ab,\Bb\in \RR^{3\times 3}$ with $\Ab_{i,j}=0.4^{|i-j|+1},\, \Bb_{i,j}=0.1^{|i-j|+1},\,i,j\in\{1,2,3\}$. In addition, $\{\bxi_t\}_{t\ge 1}$ are i.i.d. with density $f_m(\bxi)\varpropto(1-\bxi^\top\bSigma^{-1}\bxi)^{m+1}$ where the $\bSigma$ is the same as the one given in \textbf{(ii)}.
\end{itemize}

When implementing our algorithm, we divide the time horizon into consecutive episodes by setting the length of the $k$-th episode as $\ell_k=2^{k-1}\ell_0$ with $k\in\NN^{+}$ and $\ell_0=200$. We further separate every episode into an exploration phase with length $|I_k|=\min\{(d\ell_k)^{(2 m+1)/(4 m-1)},\ell_k\}$ depending on the values of $m$ and $d$. The exploitation phase contains the rest of the time in that episode. In the exploration phase, we sample $p_t$ from $\textrm{Unif}(0,B=6)$, since $B=6$ is a valid upper bound of $v_t$. In the exploitation phase, we set the kernels as follows: For any given $m\in \{2,4,6\}$ prefixed at the beginning of the algorithm, we choose the kernel function with $m$-th order. Here we choose the second, fourth, sixth-order kernel functions as $K_2(u)=35/12(1-u^2)^3\cdot \II_{\{|u|\le 1\}}$, $K_4(u)=27/16(1-11/3u^2)\cdot K_2(u)$ and $K_6(u)=297/128(1-26/3 u^2+13u^4)\cdot K_2(u)$ respectively.   In episode $k$, we set the bandwidth $b_k$ as $3\cdot |I_k|^{-\frac{1}{2m+1}}$ in \eqref{exp:f_k} and \eqref{exp:f_11} according to the settings in the theoretical analysis. In reality, one can also tune the bandwidth by using cross validation at the end of every exploration phase. Moreover, when calculating $p_t=\hat g(\tilde{\xb}_t^\top\hat\btheta_k)=\tilde{\xb}_t^\top\hat\btheta_k+\hat\phi_k^{-1}(-\tilde{\xb}_t^\top\hat\btheta_k)$, we find $\hat\phi_k^{-1}(-\tilde{\xb}_t^\top\hat\btheta_k)$ as follows: First, we look for $x\in[-1,1]$ such that $\hat\phi_k(x)=-\tilde{\xb}_t^\top\hat \btheta_k$ (The interval $[-1,1]$ contains the true support of $\phi(x)$ [-0.5, 0.5], since in reality, we might only know a range of the true support). Then, we do a transformation of variable $x$ to $x(y)=-2\cdot \exp(y)/(1+\exp(y))+1$ and solve $y$ as the root of $\hat\phi_k(x(y))+\tilde{\xb}_t^\top\hat \btheta_k=0$ by using Newton's method starting at $y=0$. Finally, we set $x=-2\cdot \exp(y)/(1+\exp(y))+1$ as $\hat\phi_k^{-1}(-\tilde{\xb}_t^\top\hat\btheta_k)$ and offer $p_t$ according to the algorithm.

For any given $m\in\{2,4,6\}$, under the three covariate settings discussed above, we input $m$ into the algorithm, select the corresponding kernel and repeat Algorithm \ref{pricing_1} for $30$ times until $T=6300$. For each $T\in[1500,2000,3100,4000,5000,6300]$, we record the cumulative regret reg$(T)$. For the first two covariate settings, recall from Theorem \ref{mainthm} that the regret reg$(T)\lesssim T^{\frac{2m+1}{4 m-1}}\log^2 T$. Thus, we plot $\tilde{\textrm{reg}}(T)$ against $\log(T)-\log(1500)$ in Figure \ref{indp}-\ref{mix_dp}, where $\tilde{\textrm{reg}}(T):=\log(\textrm{reg}(T))-2\log\log T-(\log(\textrm{reg}(1500))-2\log\log 1500)$; 


\begin{figure}[H]
	\centering
	\begin{tabular}{ccc}
		\hskip-30pt\includegraphics[width=0.35\textwidth]{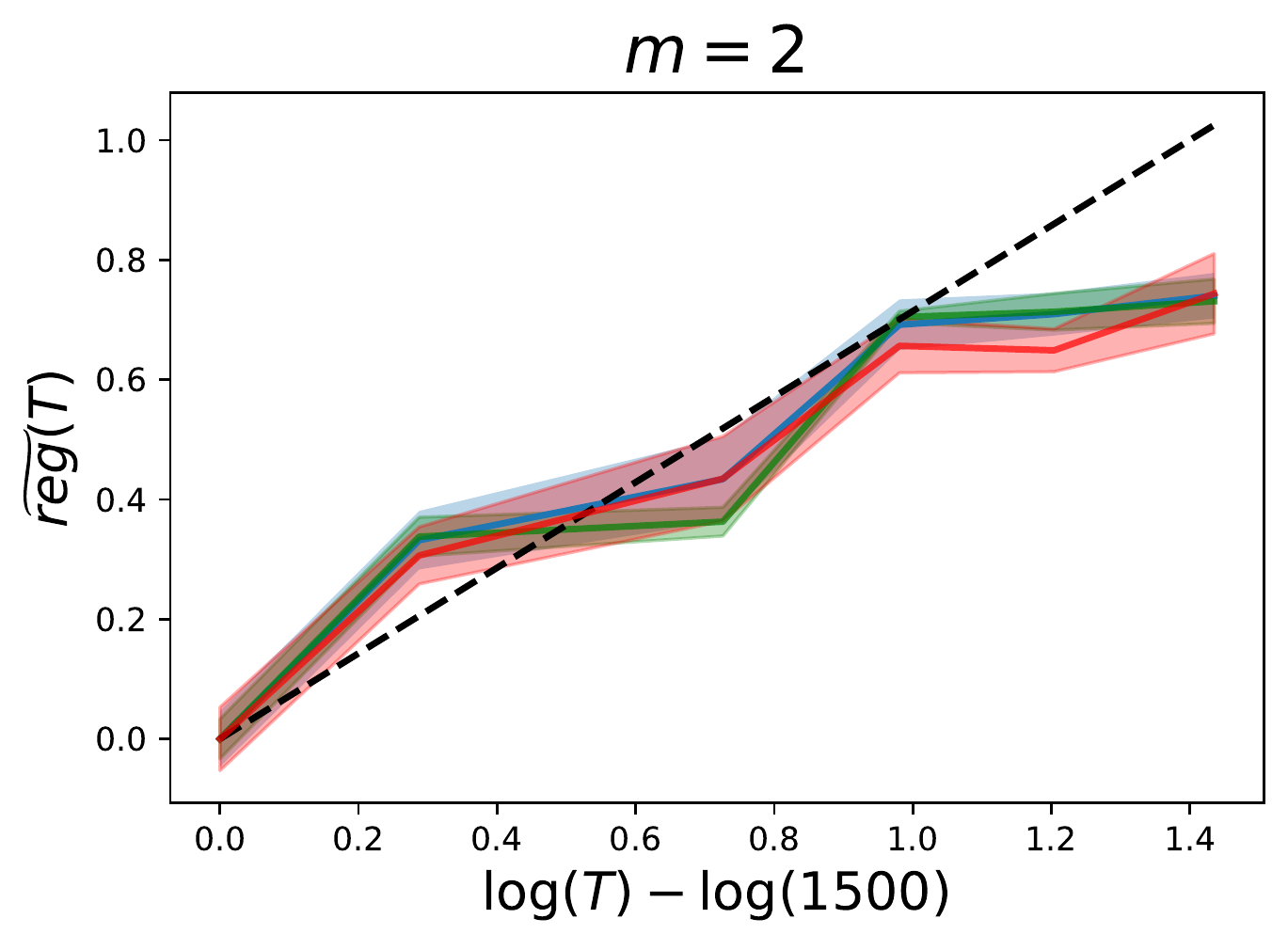}
		&
		\hskip-6pt\includegraphics[width=0.35\textwidth]{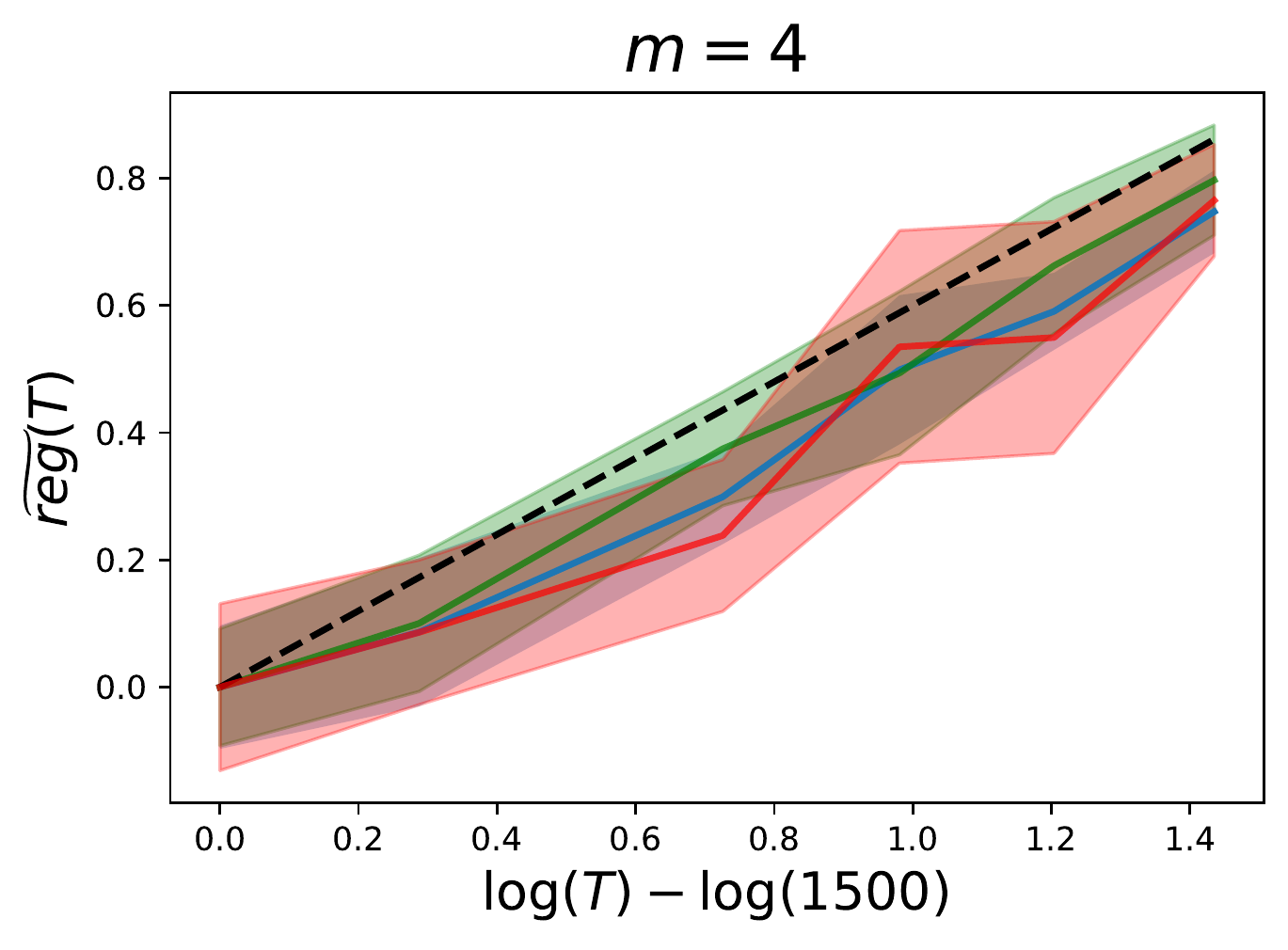}
		&
		\hskip-5pt\includegraphics[width=0.35\textwidth]{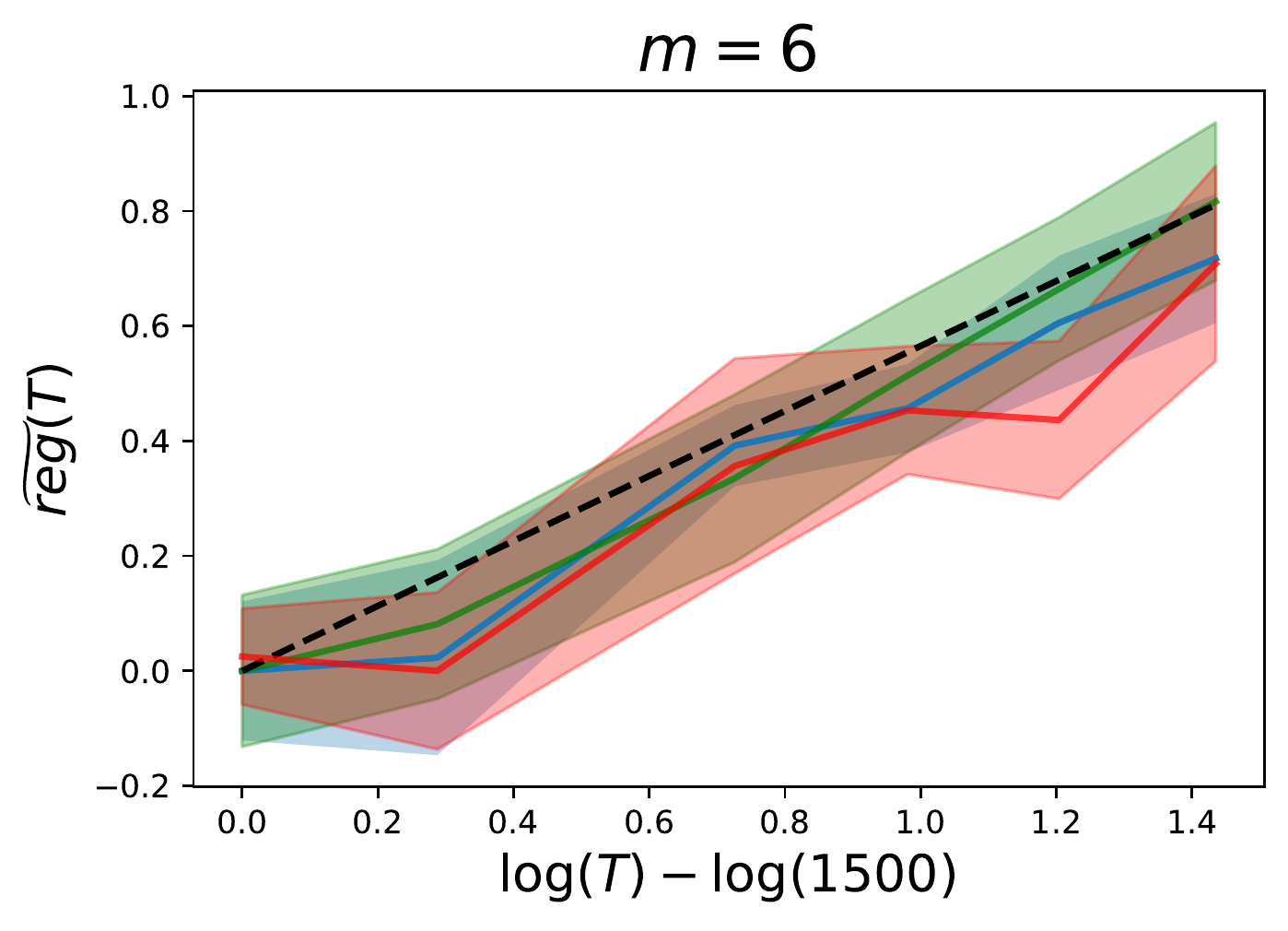} \\
		(a)  & (b) &(c)
	\end{tabular}\\
	\caption{Regret log-log plot in the setting with i.i.d. covariates with independent entries. The three subplots show the case $m\in [2, 4, 6]$ respectively. The x-axis is $\log(T)-\log(1500)$ for $T\in[1500,2000,3100,4000,5000,6300]$, while the y-axis is $\tilde{\textrm{reg}}(T):=\log(\textrm{reg}(T))-2\log\log T-(\log(\textrm{reg}(1500))-2\log\log 1500)$. The solid blue, green and red lines represent the mean $\tilde{\textrm{reg}}(T)$ of the Algorithm \ref{pricing_1} with unknown $g(\cdot)$ and $\btheta_0$, unknown $g(\cdot)$ but known $\btheta_0$, and known $g(\cdot)$ but unknown $\btheta_0$ respectively over $30$ independent runs. The light color areas around those solid lines depict the standard error of our estimation of $\log(\textrm{reg}(T))-2\log\log T$. The dashed black lines in $(a)-(c)$ represents the benchmark whose slopes are equal to $\frac{2 m+1}{4 m-1}$ with $m\in \{2,4,6\}$.}
	\label{indp}
\end{figure}

\begin{figure}[H]
	\centering
	\begin{tabular}{ccc}
		\hskip-30pt\includegraphics[width=0.35\textwidth]{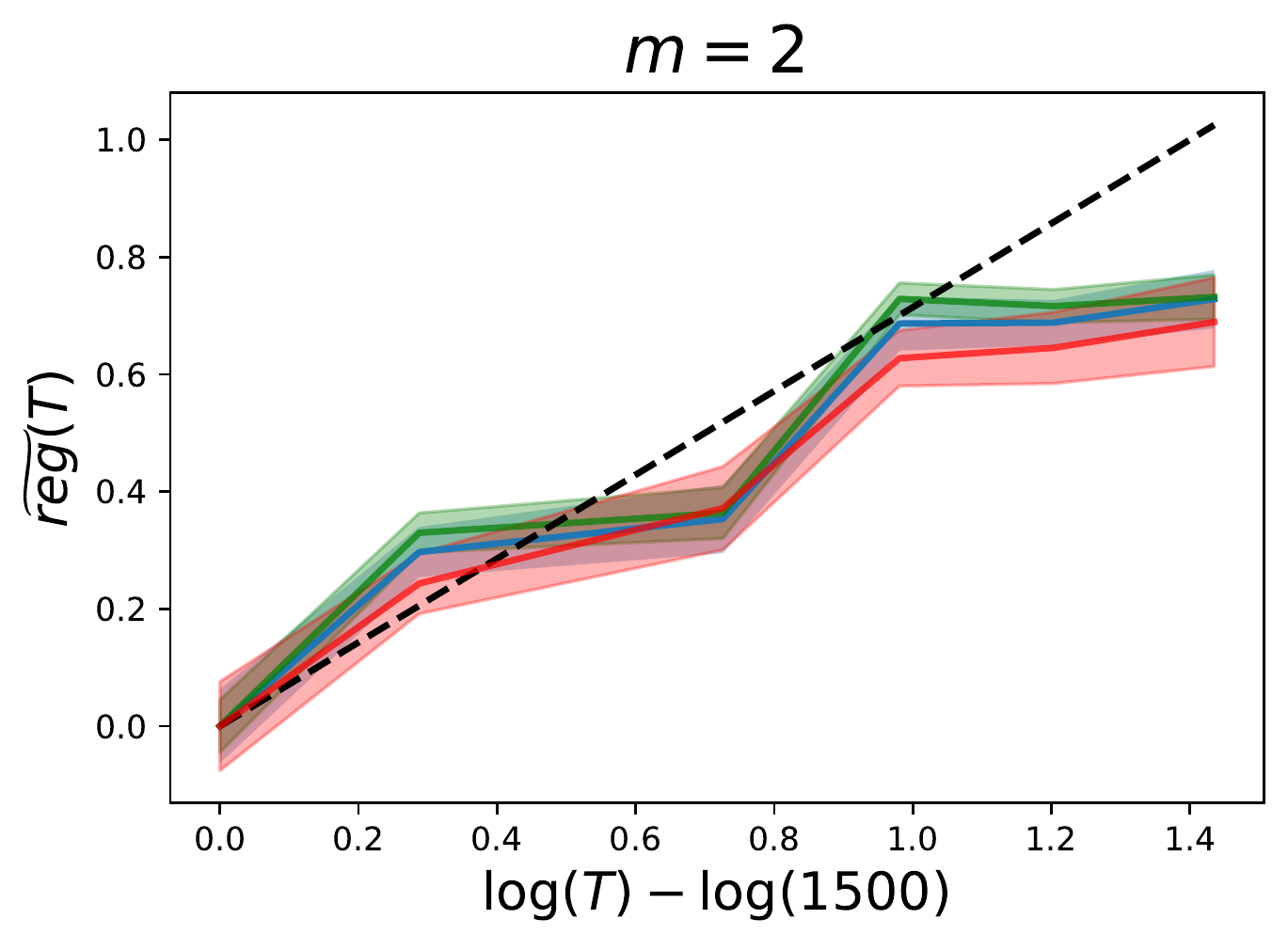}
		&
		\hskip-6pt\includegraphics[width=0.35\textwidth]{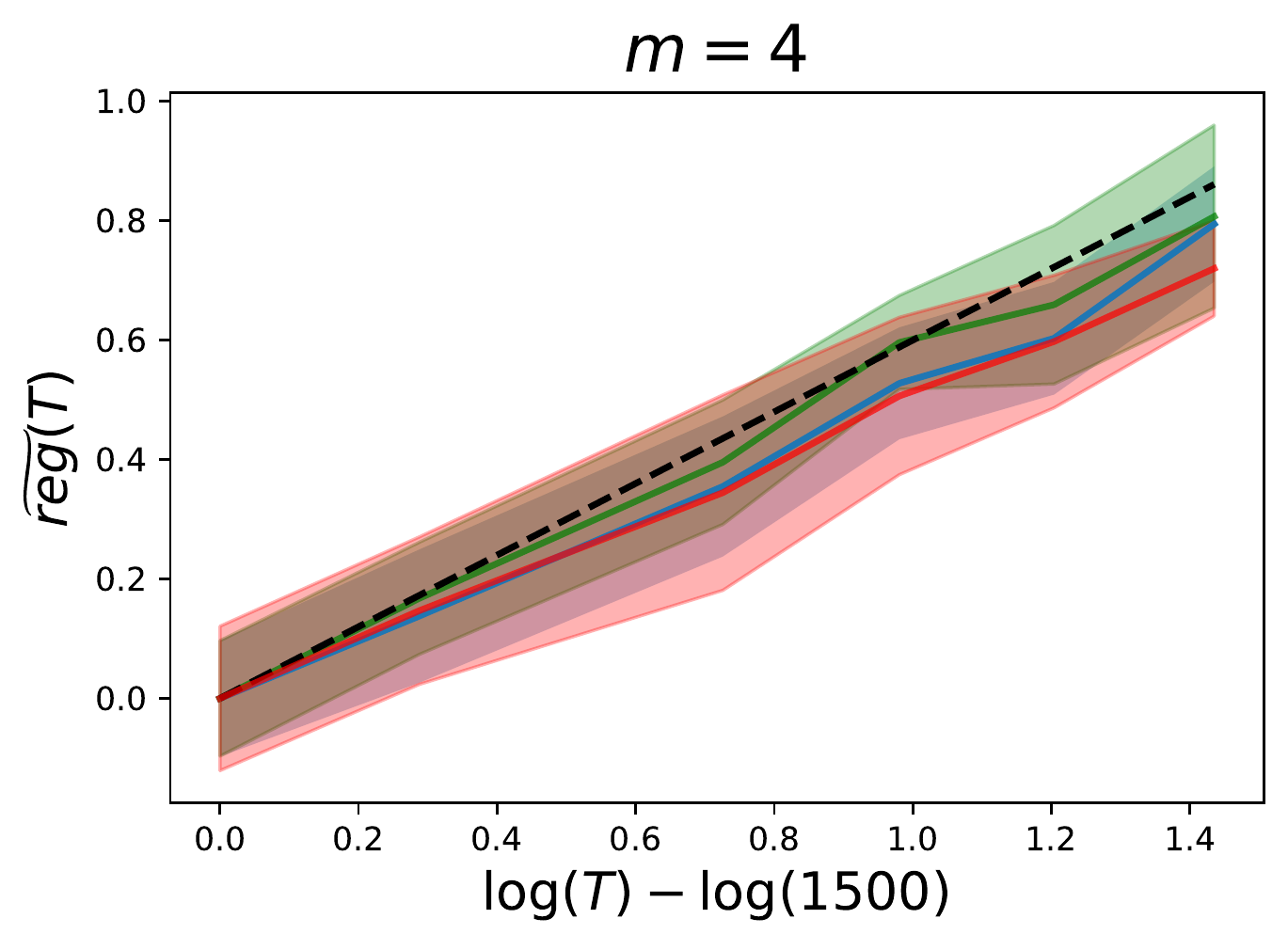}
		&
		\hskip-5pt\includegraphics[width=0.35\textwidth]{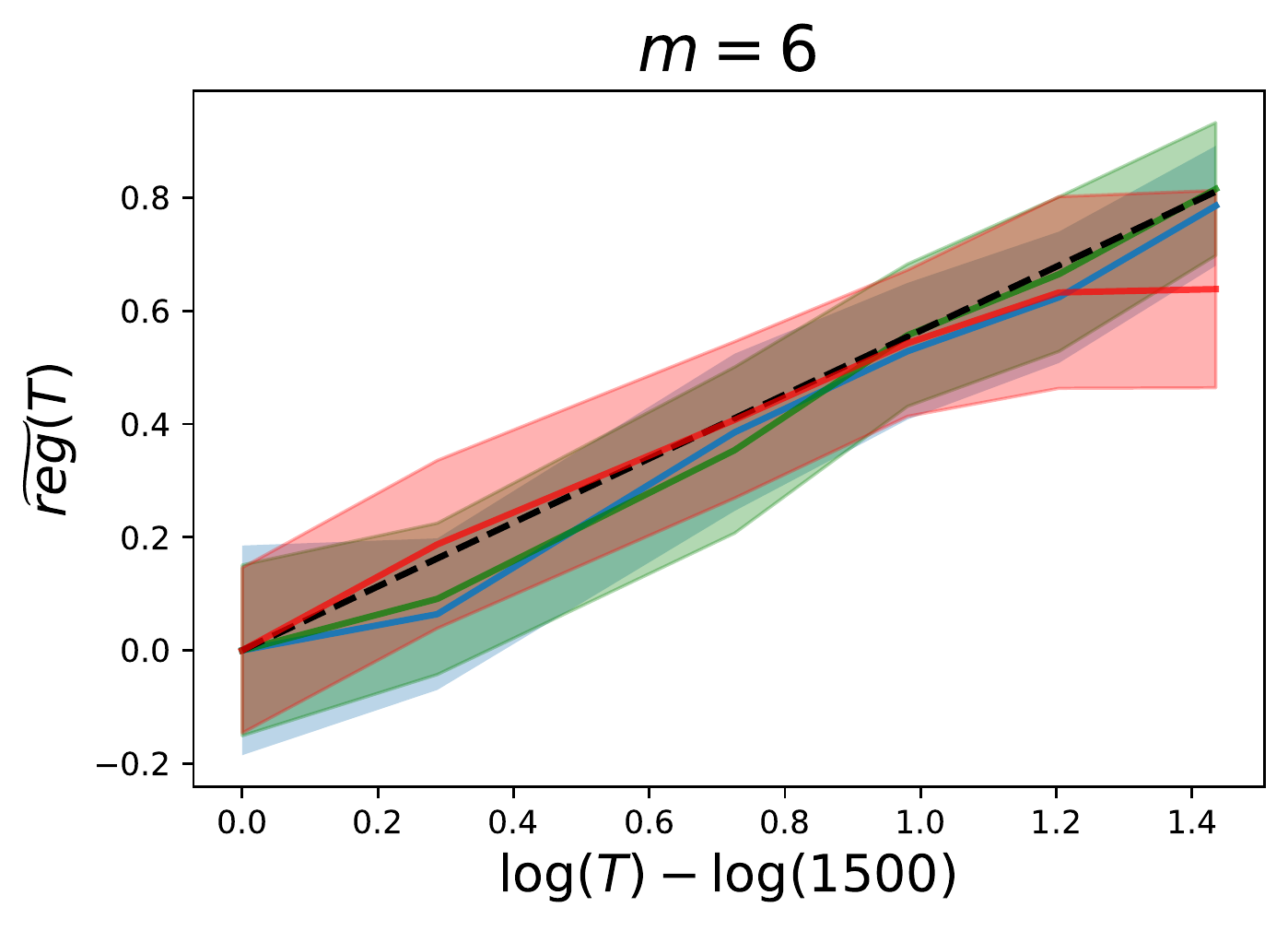} \\
		(a)  & (b) &(c)
	\end{tabular}\\
	\caption{Regret log-log plot in the setting with i.i.d. covariates with dependent entries.
		The remaining caption is the same as Figure~\ref{indp}.}
	\label{dp}
\end{figure}

\begin{figure}[H]
	\centering
	\begin{tabular}{ccc}
		\hskip-30pt\includegraphics[width=0.35\textwidth]{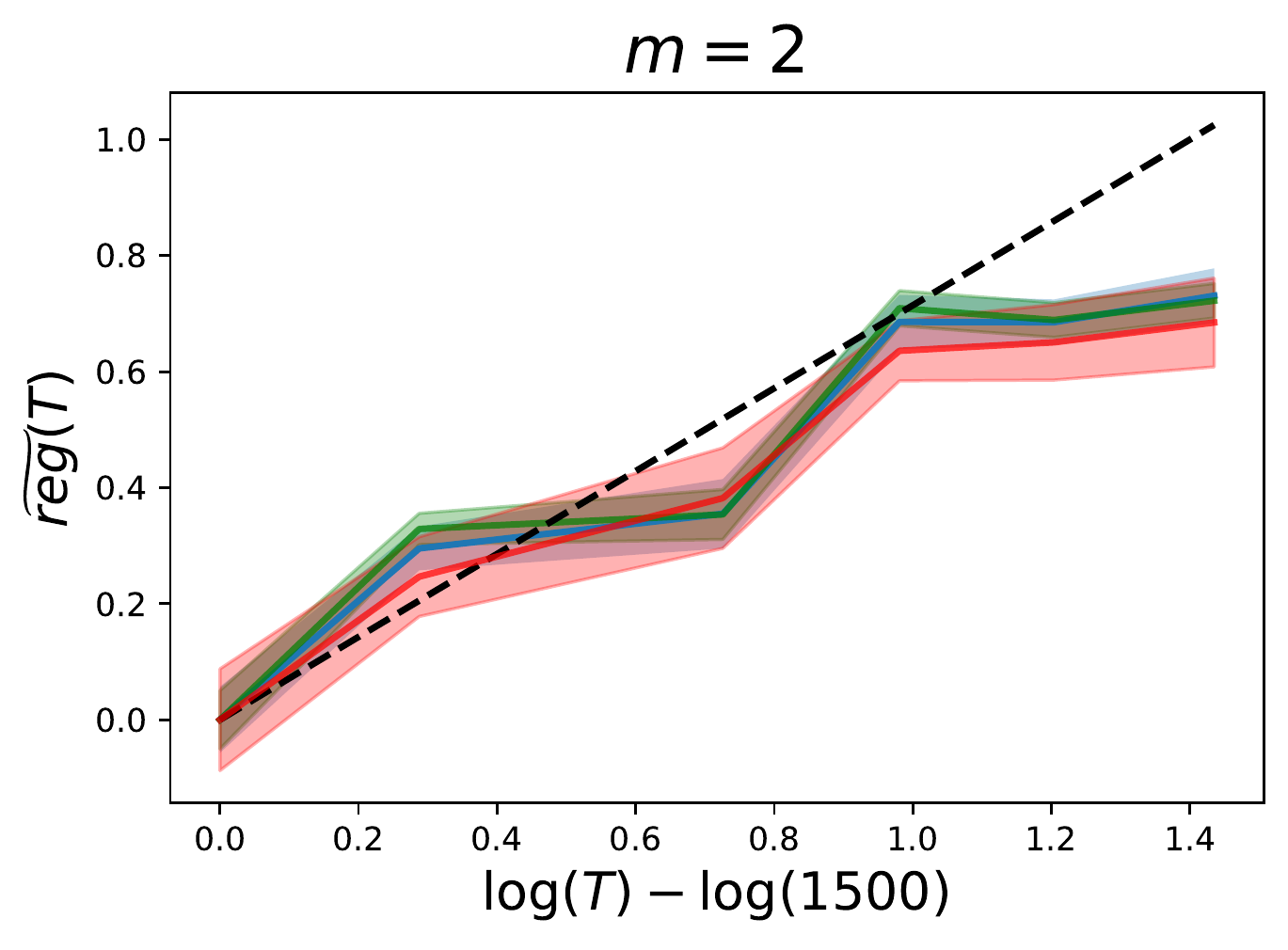}
		&
		\hskip-6pt\includegraphics[width=0.35\textwidth]{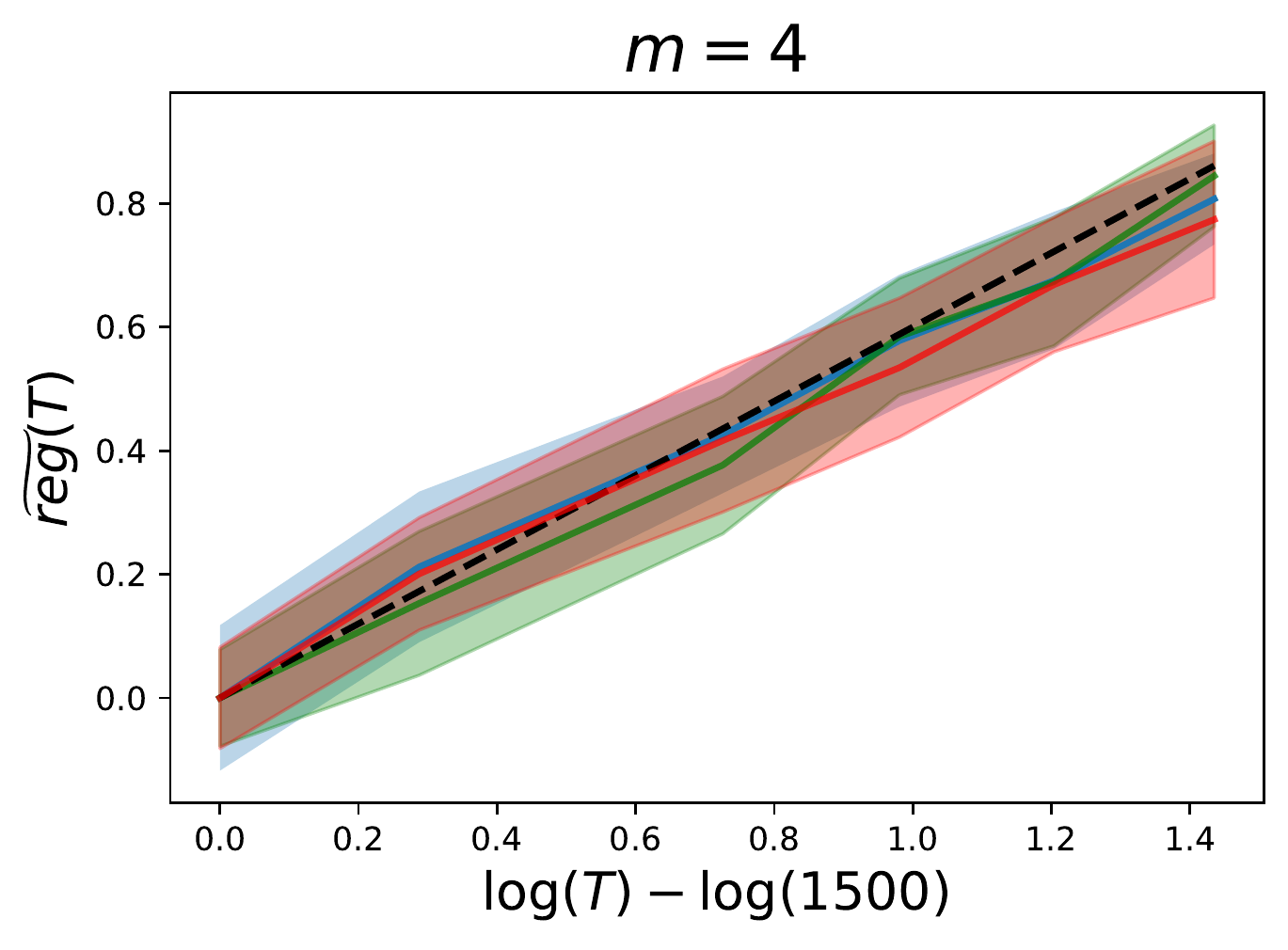}
		&
		\hskip-5pt\includegraphics[width=0.35\textwidth]{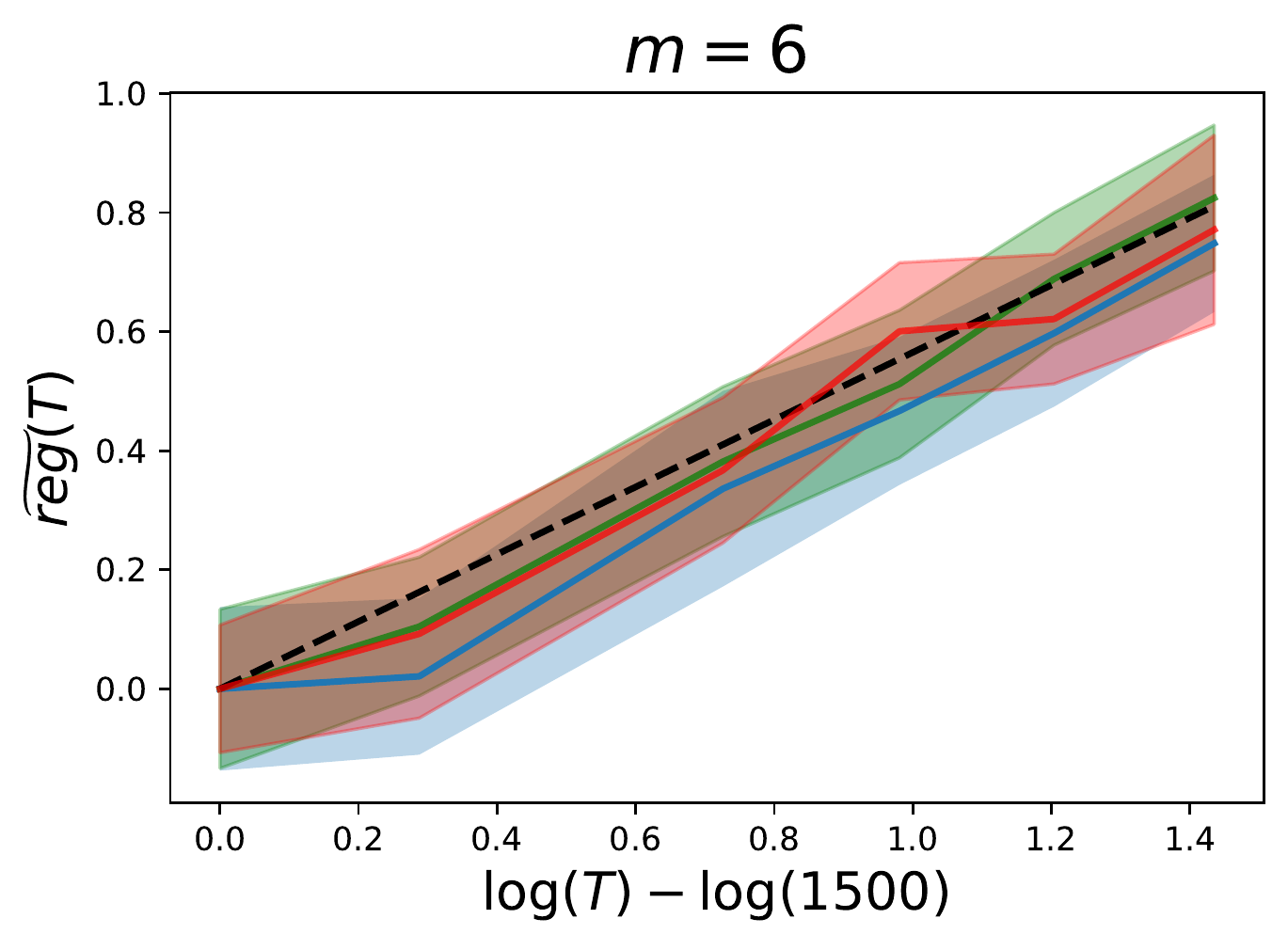} \\
		(a)  & (b) &(c)
	\end{tabular}\\
	\caption{Regret log-log plot in the setting with strong mixing covariates.
		The remaining caption is the same as Figure~\ref{indp}.}
	\label{mix_dp}
\end{figure}

{From Figures \ref{indp}-\ref{mix_dp}, we conclude that under all settings, the rates of the empirical regrets' increments produced by Algorithm \ref{pricing_1} (as shown by the solid blue lines) do not exceed their theoretical counterparts given in Theorems \ref{mainthm} and \ref{mainthm2} (as shown by the dashed black lines). In many cases, the growth rates of the empirical regrets are very close to those of the theoretical lines. This demonstrates the tightness of our theoretical results. Moreover, as all the solid lines have similar growth rates, we show that Algorithm \ref{pricing_1} is robust to the estimation of $\btheta_0$ and $g(\cdot)$. This is further proved in Appendix \ref{add_plots}, where we directly plot reg$(T)$ for all the settings discussed here. See Appendix \ref{add_plots} for more plots and discussions.}



\subsection{Comparison with other methods}\label{sec:simulation-compare}
\revise{In this subsection, we provide numerical studies which illustrate differences between our methods and two highly related prior arts (`RMLP-2' and `Bandit') using both synthetic and real data. Here, `RMLP-2' is the policy proposed in \cite{JN19} that solves the same problem as ours except that the noise distribution falls in a \emph{parametric} function class. In addition, we denote the policy proposed in \cite{KL03} as 'Bandit', which leverages a variant of UCB algorithm under non-parametric noise distribution that achieves $\cO(\sqrt{T})$ regret \emph{without} modeling covariate information. 
}

\revise{We first use synthetic data to illustrate the efficiency of our method over `RMLP-2' and `Bandit'. For each smoothness degree $m=\{2,4,6\},$ we generate our data following the same way given in \S\ref{justification_theory}, except that we only generate the distribution of $\xb_t$ according to the first option discussed in \S\ref{justification_theory}. We illustrate the performance of our method against those two prior arts in the following figures. Here we follow Algorithm \ref{pricing_with_m} which uses a data-driven way to determine $m$ before every episode. For RMLP-2, since there is no way the algorithm knows the true noise distribution, we instead assume the noise falls into a Gaussian distribution when executing the algorithm.}
\begin{figure}[H]
	\centering
	\begin{tabular}{ccc}
		\hskip-30pt\includegraphics[width=0.37\textwidth]{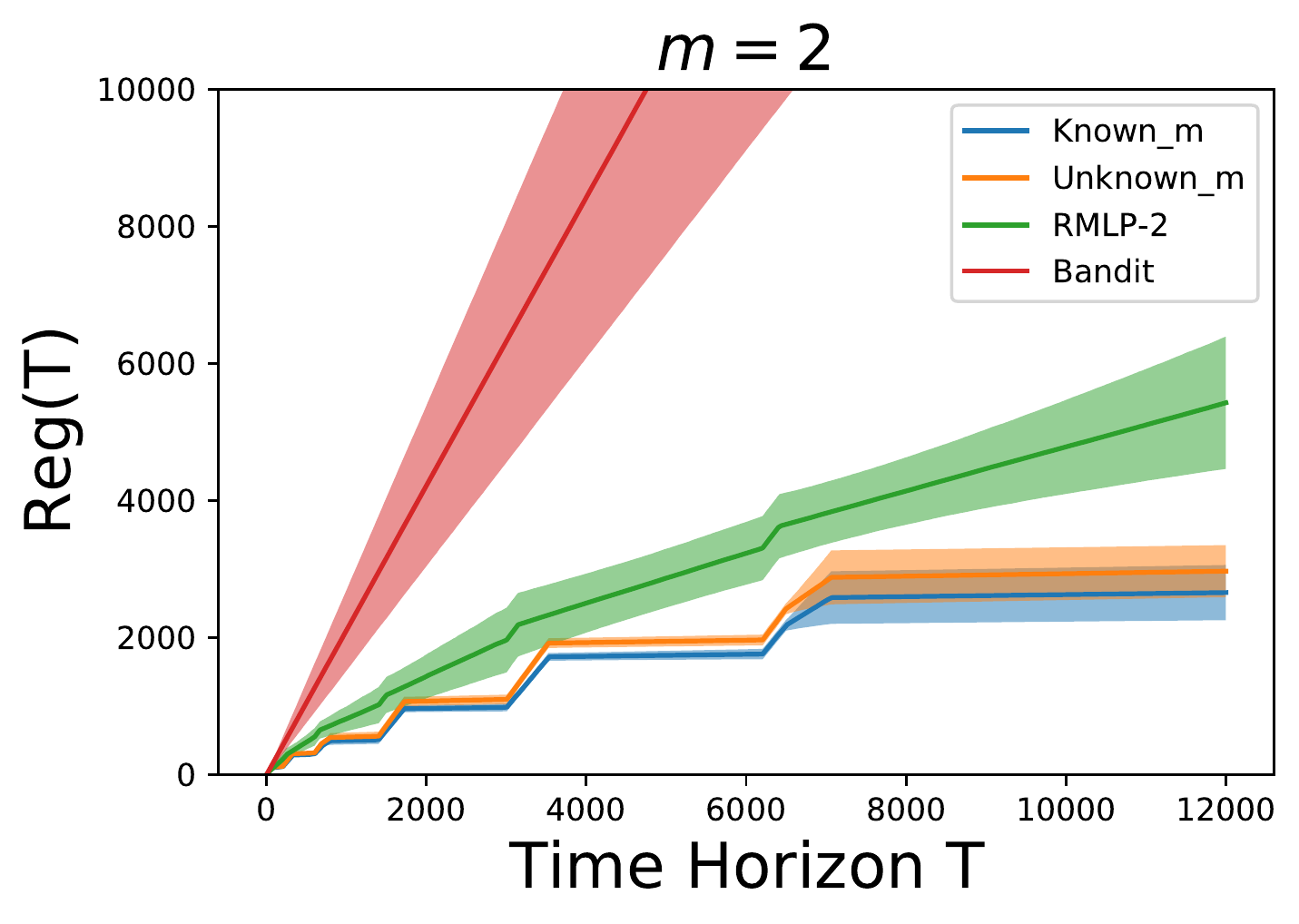}
		&
		\hskip-6pt\includegraphics[width=0.37\textwidth]{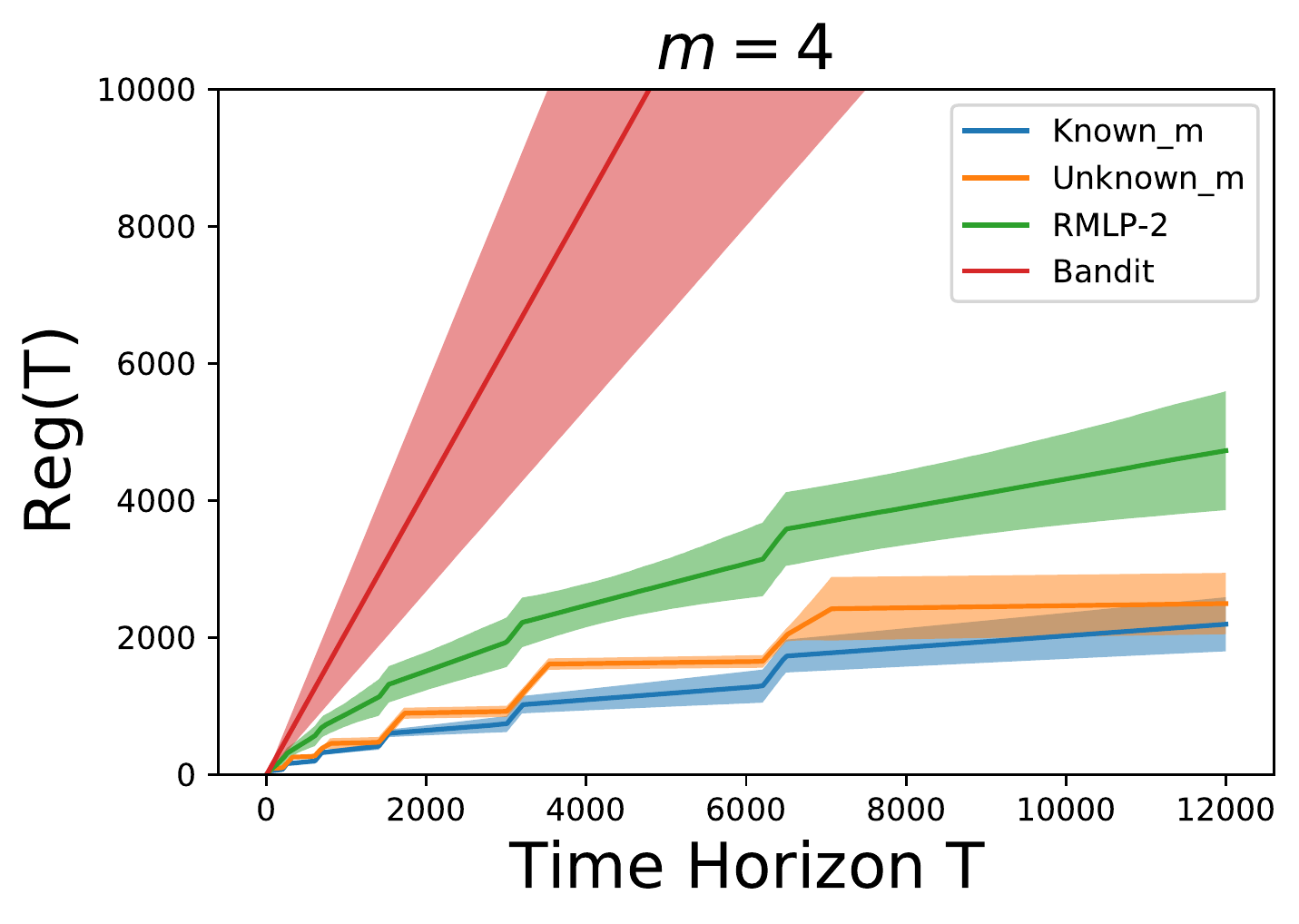}
		&
		\hskip-5pt\includegraphics[width=0.37\textwidth]{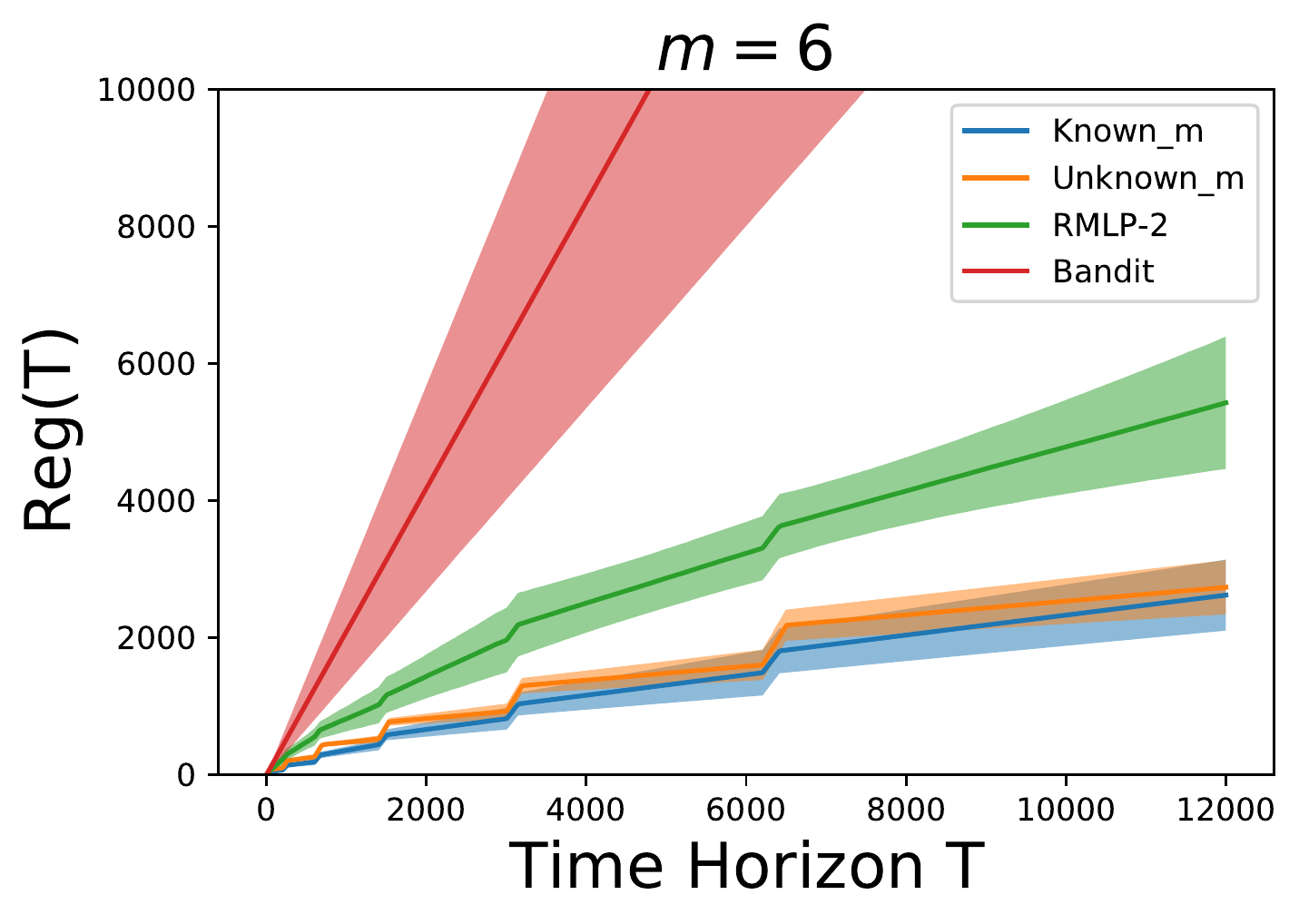} \\
		(a)  & (b) &(c)
	\end{tabular}\\
	\caption{Regret Comparison between our methods and two benchmarks (RMLP-2 and Bandit). From the left to the right, the true underlying degree of smoothness is $m=\{2,4,6\}$ respectively. The x-axis denotes the time stamp $T$  ranges from $1\sim 12000,$ and the y-axis denotes the regret at the time $T$ defined in \eqref{eq:Regret_def}. We repeat the experiment 30 times and record the averaged regrets (solid lines) and standard errors (light areas) of every policy. The blue line denotes the regret of our policy (in Algorithm \ref{pricing_1}) with knowing degree of smoothness $m$ and the orange line represents the regret of our policy (given in Algorithm \ref{pricing_with_m}) without knowing degree of smoothness $m$. The green and red lines are the regrets of implementing 'RMLP-2' and 'Bandit' policy respectvely. }
		\end{figure}

\revise{We see from the simulation results that the regret we achieved is much smaller than those two benchmarks. As for the comparison with RMLP-2, our method is robust to the mis-specification of the parametric function class since our algorithm can adapt to all functions in the non-parametric class. For the comparison with 'Bandit', we see that only using the non-parametric bandit algorithm without considering the contextual information (heterogeneity of product) will lose much efficiency in gaining revenue.
} 
\subsubsection{Real Application}
 \revise{Next, we leverage a simulation based on the real data to further illustrate the merits of our Algorithm over 'RMLP-2' and 'Bandit'.}
 
\revise{ We use the real-life auto loan dataset  provided by the Center for Pricing and Revenue Management at Columbia University. This dataset is used by several related works \citep{Phillips2015,Ban2020PersonalizedDP,LSL21,WWSC20} and many others. The dataset contains $208,085$ auto loan applications received from July 2002 to November 2004. Some features such as the amount of loan, the borrower's information is
contained in that dataset. We adopt the feature selection in the same way with \cite{Ban2020PersonalizedDP,LSL21,WWSC20} and consider the following four features: the loan amount approved, FICO score, prime rate and competitor's rate. As for the price variable, we also computed it in the same way with the aforementioned literature, where $p_t=\textrm{Monthly Payment}\cdot \sum_{t=1}^{\textrm{Term}}(1+\textrm{Rate})^{-t}-\textrm{Loan Amount}$. The rate is set as $0.12\% $, which is an approximate average of the monthly London interbank rate for the studied time period.
Moreover, this dataset also records purchasing decision of the borrowers given the price set by the lender. For more details on this dataset, please refer to \cite{Phillips2015,Ban2020PersonalizedDP}. }

\revise{Note that one is not able to obtain online responses to any algorithms, thus, we follow the calibration idea proposed in \cite{Ban2020PersonalizedDP,LSL21,WCCG2020} to first estimate the binary choice model and leverage it as the ground truth to conduct online numerical experiments. To be more specific, we first scale all variables into the scale of $[0,1] $ (since the prediction results of single index model won't be affected by scale of the covariates). 
We randomly sample 5000 data points,  estimate $\btheta_0$ and $F$ using semi-parametric estimation tools from these data.
We next treat them as the underlying true parameters for our binary choice model stated in \eqref{eq:prob-model}. Given these key components, the remaining experiments remain almost the same as discussed in \S\ref{justification_theory} and \S\ref{sec:simulation-compare}, except that here we set $\btheta_0$, distribution $F(\cdot)$ as the estimated one given above and  sample $\xb_t$ from those four features above. We set $B_0=4,\ell_0=200$ and conduct Algorithm \ref{pricing_with_m} (in this algorithm, we use cross-validation to select $m$ at the beginning of every episode, details are given in Algorithm \ref{setting_m}).  }

\revise{We next compare Algorithm \ref{pricing_with_m} with `RMLP-2' and `Bandit' policies. The details are given in Figure \ref{fig-5}. }
\begin{figure}[H]
	\centering
	\begin{tabular}{c}
		\hskip-30pt\includegraphics[width=0.45\textwidth]{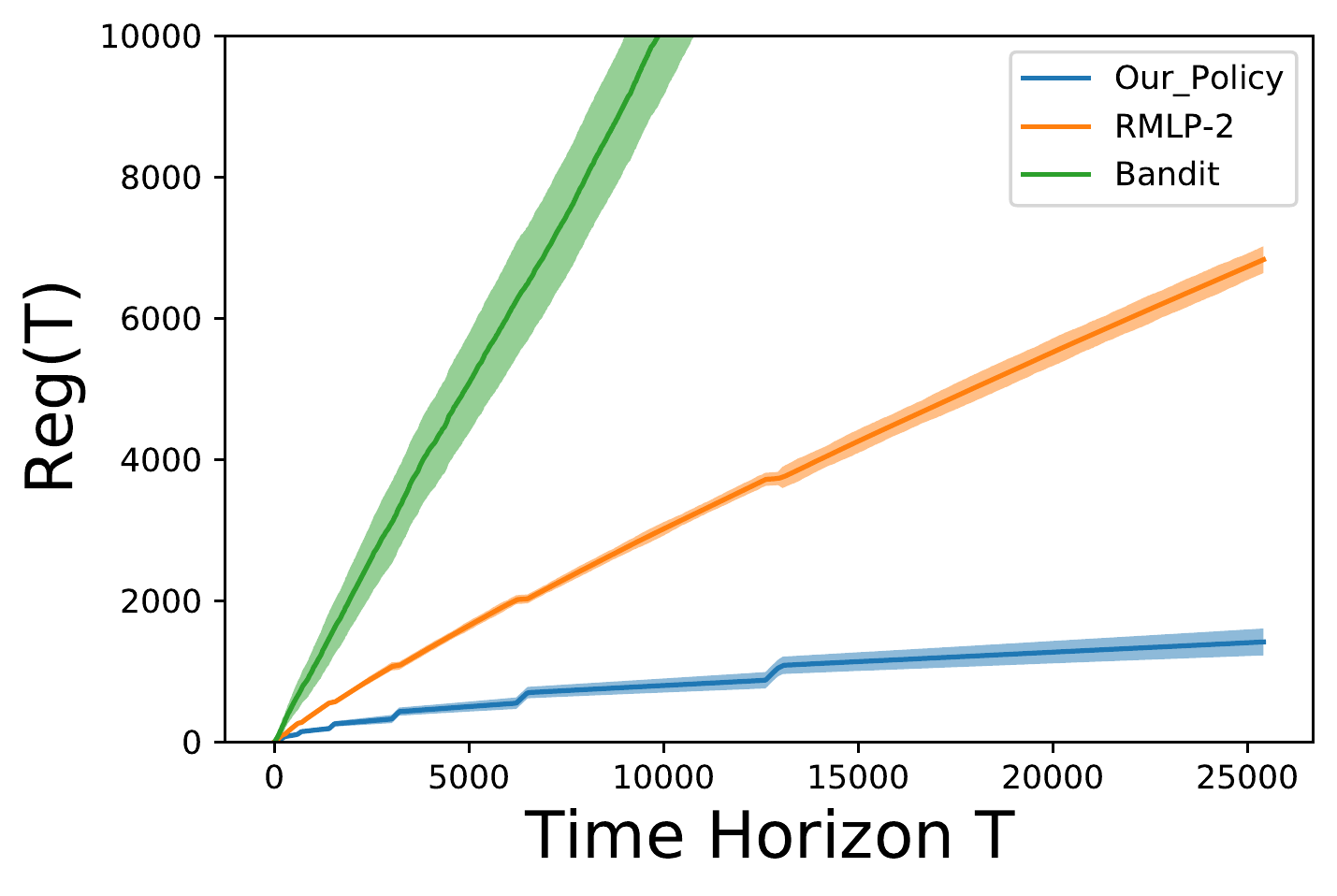}
	\end{tabular}\\
	\caption{Comparison between our policy and `RMLP-2' and `Bandit' based on real data application.}\label{fig-5}
		\end{figure}
\revise{To summarize, our policy outperforms the RMLP-2 \citep{JN19} and non-parametric bandit policy \citep{kleinberg2003} in terms of both the regret performance and the ability to adapt to different noise distributions.}

\section{Conclusion}

In this paper, we study the contextual dynamic pricing problem where the market value is linear in features, and the market noise has unknown distribution. We propose a policy that combines semi-parametric statistical estimation and online decision making. Our policy achieves near optimal regret, and is close to the regret lower bound where the market noise distribution belongs to a parametric class. We further generalize these results to the case when the product features satisfy the strong mixing condition. The practical performance of the algorithm is proved by extensive simulations.

There are several directions worth exploring in the future. First, we conjecture that the estimation accuracy of the market noise distribution $F$ is crucial in the regret. Thus, within the function class $F\in \CC^{(m)}$, we conjecture that a tighter regret lower bound $\Omega_d(T^\frac{2m+1}{4m-1})$ can be achieved instead of $\Omega_d(\sqrt{T})$,  namely, our procedure is optimal.  Second, in this work, we consider a linear model for the market value. In case a more complex model is appropriate, it's possible to extend our methodology to where the market value is nonlinear in product features, e.g. $v_t = \phi(\btheta_0^\top \xb_t) + z_t$ or other structured statistical machine learning model such as the additive model $v_t = f_1(x_{t1}) + \cdots + f_d(x_{td}) + z_t$. Finally, it's worth studying similar pricing problems with adversarial or strategic buyers, which is potentially more suitable in some specific applications.





\newpage
\appendix

\section{Proof under the time-independent feature setting}
\subsection{Proof of Lemma \ref{thmsignal}}\label{proofparathm}
First, recall that $R_\cX := \sup_{\xb\in\cX}\|\xb\|_2$, we deduce that $\xb_t$ is also subgaussian with norm upper bounded by $\psi_x = R_\cX$. This fact is useful in later proofs as well. Now according to \eqref{thetaupdate}, for the $k$-th episode, our loss function $L_k(\btheta)$ is defined as 
\begin{align}\label{lossepisode}
L_k(\btheta)=\frac{1}{|I_{k}|}  \sum_{t\in I_{k}}(By_t-\btheta^\top\tilde{\xb}_t)^2.
\end{align}
For notational convenience, denote $n = |I_{k}|$. Then the gradient and Hessian of $L_k(\btheta)$ is given by
\begin{align}\label{gradient}
\nabla_{\btheta} L_k(\btheta)&=\frac{1}{n}\sum_{t\in I_{k}} 2(\btheta^\top\tilde{\xb}_t-By_t)\tilde{\xb}_t,\\
\nabla^2_{\btheta}L_k(\btheta)&=\frac{1}{n}\sum_{t\in I_{k}} 2\tilde{\xb}_t\tilde{\xb}_t^\top.\label{hessien}
\end{align}
Let $\hat\btheta_k$ be the global minimizer of $L_k(\btheta)$. We do a Taylor expansion of $L_k(\hat\btheta_k)$ at $\btheta_0$:
 \begin{align}\label{taylorexp}
L_k(\hat\btheta_k)-L_k(\btheta_0)=\langle \nabla L_k(\btheta_0),\hat\btheta_k-\btheta_0 \rangle +\frac{1}{2}\langle\hat\btheta_k-\btheta_0,\nabla^2_{\btheta}L_k(\tilde{\btheta})(\hat\btheta_k-\btheta_0) \rangle.
 \end{align}
Here $\tilde{\btheta}$ is a point lying between $\hat\btheta_k$ and $\btheta_0$. As $\hat\btheta_k$ is the global minimizer of loss \eqref{lossepisode}, we have
 \begin{align*}
\langle \nabla L_k(\btheta_0),\hat\btheta_k-\btheta_0 \rangle +\frac{1}{2}\langle\hat\btheta_k-\btheta_0,\nabla^2_{\btheta}L_k(\tilde{\btheta})(\hat\btheta_k-\btheta_0) \rangle\le 0
\end{align*}
which implies
 \begin{align}\label{ineq}
\langle\hat\btheta_k-\btheta_0,\frac{1}{n}\sum_{t\in I_{k}} \tilde{\xb}_t\tilde{\xb}_t^\top(\hat\btheta_k-\btheta_0) \rangle\le \langle \nabla L_k(\btheta_0),\btheta_0-\hat\btheta_k \rangle\le\sqrt{d}\|\nabla L_k(\btheta_0) \|_{\infty}\cdot\|\btheta_0-\hat\btheta_k\|_2.
\end{align}
In order to achieve $\ell_2$-convergence rate of $\hat\btheta_k$, we separate our following analysis into two steps.

\textbf{Step I:} In this step, we lower bound the minimum eigenvalue of 
\begin{align}\label{sigmak}
\bSigma_k:=\frac{1}{n}\sum_{t\in I_{k}} \tilde{\xb}_t\tilde{\xb}_t^\top.
\end{align}
using concentration inequalities.

Since $\bSigma_k$ is an average of $n$ i.i.d. random matrices with mean $\Sigma=\EE[\tilde{\xb}_t\tilde{\xb}_t^\top]$ and that $\{\tilde{\xb}_t\}$ are sub-Gaussian random vectors, according to Remark 5.40 in \cite{vershynin_2012}, there exist $c_1$ and $C>c_{\min}$ such that with probability at least $1-2e^{-c_1t^2}$,
\begin{align}\label{concent}
\|\bSigma_k-\bSigma\|\le \max\{\delta,\delta^2\},\textrm{\,\, where\,\,} \delta:=C\sqrt{\frac{d+1}{n}}+\frac{t}{\sqrt{n}}.
\end{align}
Here $c_1, C$ are both constants that are only related to sub-Gaussian norm of $\tilde{\xb}_t$. Now we plug in $t=c_{\min}\sqrt{n}/4$ and $c_0=16C^2/c_{\min}^2$, then as long as $n\geq c_0(d+1)$, with probability at least $1-2e^{-c_1c_{\min}^2n/16}$, 
\begin{equation}\label{eventak}
(c_{\min}/2)\cdot\II\preccurlyeq \bSigma_k.
\end{equation} 

\textbf{Step II:} In this step, we provide an upper bound of $\|\nabla_{\btheta} L_k(\btheta_0) \|_{\infty}$.

First, we prove $\EE[\nabla_{\theta} L_k(\btheta_0)]=0$.
By definition we have 
\begin{align*}
\nabla_{\btheta} L_k(\btheta_0)&=\frac{1}{n}\sum_{t\in I_{k}} 2(\btheta_0^\top\tilde{\xb}_t-By_t)\tilde{\xb}_t
\end{align*}
We take the conditional expectation of $\nabla_{\btheta} L_k(\btheta_0)$ and obtain
\begin{align*}
\EE[\nabla_{\btheta} L_k(\btheta_0)\given \tilde{\xb}_t]=\frac{1}{n}\sum_{t\in I_{k}} 2\EE[(\btheta_0^\top\tilde{\xb}_t-By_t)\given \tilde{\xb_t}]\tilde{\xb}_t.
\end{align*}
By our definition on $y_t$,
\begin{align*}
\EE[\btheta_0^\top\tilde{\xb}_t-By_t\given \tilde{\xb}_t]&=\btheta_0^\top\tilde{\xb}_t-\EE[B\II_{\{p_t\le v_t\}}\given  \tilde{\xb}_t ]\\&=\btheta_0^\top\tilde{\xb}_t-\EE[\EE[B\II_{\{p_t\le v_t\}}\given v_t]\given  \tilde{\xb}_t ]\\&=\btheta_0^\top\tilde{\xb}_t-B\cdot\EE[v_t/B\given  \tilde{\xb}_t ]=0,
\end{align*}
where the third equality follows from $p_t\sim \textrm{Uniform}(0,B)$. After finally taking expectation with respective to $\tilde{\xb}_t$ we deduce that $\EE[\nabla_{\btheta} L_k(\btheta_0)]=0$.

Next, we get an upper bound of $\|\nabla_{\btheta} L_k(\btheta)\|_{\infty}$.
By \eqref{gradient}, we have every entry of $\nabla_{\btheta} L_k(\btheta_0)$ is mean zero. In addition, according to our Assumption \ref{ass:bound}, we have $\xb_t$ are i.i.d. sub-Gaussian random vectors with sub-Gaussian norm $\psi_x$. Thus, we have $\max_{i\in[d]}\|\xb_{t,i}\|_{\psi_2}\le \psi_x$. On the other hand, $\tilde{\xb}_t^\top\btheta_0-By_t$ is bounded by the constant $R_{\cX}R_{\Theta}+B$. Therefore,

\begin{align*}
\PP\big(|2(\btheta_0^\top\tilde{\xb}_t-By_t)\tilde{\xb}_{t,i}|\ge u\big)\le\PP\big(2(R_{\cX}R_{\Theta}+B)|\tilde{\xb}_{t,i}|\ge u \big)\le 2\exp\Big(\frac{-u^2}{8\psi_x^2(R_{\cX}R_{\Theta}+B)^2}\Big)
\end{align*}
for $i\in[2:(d+1)]$, which implies that $2(\btheta_0^\top\tilde{\xb}_t-By_t)\tilde{\xb}_{t,i},i\in[2:(d+1)]$ are sub-Gaussian random variables with variance proxy $2\psi_x(R_{\cX}R_{\Theta}+B)$. Moreover, We can also obtain $\|2(\btheta_0^\top\tilde{\xb}_t-By_t)\tilde{\xb}_{t,1}\|_{\psi_2}\le 2(R_{\cX}R_{\Theta}+B)$ by Hoeffding's inequality. 

We now take the union bound of all entries of $\nabla_{\btheta} L_k(\btheta_0)$:\begin{align}\label{tailineq1}
\PP\big(\|\nabla_{\btheta} L_k(\btheta_0)\|_{\infty}\ge t\big)&\le 2(d+1)\exp\Big(\frac{-t^2}{8\max\{\psi_x^2,1\}(R_{\cX}R_{\Theta}+B)^2}\Big)\\&=2\exp\Big(\frac{-nt^2}{8\max\{\psi_x^2,1\}(R_{\cX}R_{\Theta}+B)^2}+\log (d+1)\Big)\label{tailineq2}.
\end{align}
As we assume $n\ge d+1$, by taking $t=4\max\{\psi_x,1\}(R_{\cX}R_{\Theta}+B)\sqrt{\log n/n}$ in \eqref{tailineq2}, then with probability $1-2/n$, we have
\begin{align}\label{gradup}
\|\nabla_{\btheta} L_k(\btheta_0)\|_{\infty}\le 4\max\{\psi_x,1\}(R_{\cX}R_{\Theta}+B)\sqrt{\frac{\log n}{n}}.
\end{align}


Finally, combining \eqref{ineq}, \eqref{eventak} and \eqref{gradup}, we obtain that with probability at least $1-2e^{-c_1c_{\min}^2|I_{k}|/16}-2/|I_{k}|$,
\begin{align*}
\|\hat\btheta_k-\btheta_0\|_2\le  \frac{8\max\{\psi_x,1\}(R_{\cX}R_{\Theta} +B)}{c_{\min}}\sqrt{\frac{(d+1)\log|I_{k}|}{|I_{k}|}}.
\end{align*}

\subsection{Proof of Lemma \ref{conckernel_main}}\label{pf:conckernel_main}

For the following analysis, we fix any episode index $k$ satisfying the conditions of Lemma \ref{conckernel_main}. It's easy to verify that for any $k\geq (\log(\sqrt{T} - \log \ell_0))/\log 2$, $\Theta_k\subset\Theta_0$. Therefore, all the assumptions hold for $\btheta\in \Theta_k$. Our goal is to prove \eqref{eq:conckernel_main} holds with high probability on the $k$-th episode. 

Now we have the i.i.d. samples $\{w_t(\btheta):= p_t - \tilde \xb _t^\top \btheta , y_t\}_{t\in I_k}$ from some distribution $P_{w(\btheta), y}$. According to the previous notations, the marginal distribution $P_{w(\btheta)}$ has density $f_{\btheta}(u)$. Moreover, $r_{\btheta}(u):=\EE [y_t\given w_t(\btheta)=u]$. We're interested in bounding the quantity $\sup_{u\in I,\btheta\in \Theta_k} |\hat r_k(u,\btheta) - r_{\btheta_0}(u)|$, which leads to the desired conclusion of the lemma. 

For notational simplicity, let $n = |I_k|$ be the length of the exploration phase. Recall that $\hat r_k(u, \btheta) = h_k(u, \btheta) / f_k(u, \btheta)$, where
$$
h_k(u,\btheta) = \frac{1}{nb_k}\sum_{t\in I_k} K(\frac{w_t(\btheta)-u}{b_k})Y_t,\quad f_k(u,\btheta) = \frac{1}{nb_k}\sum_{t\in I_k} K(\frac{w_t(\btheta)-u}{b_k}).
$$
Here, $b_k>0$ is the bandwidth (to be chosen), and $K(\cdot)$ is some kernel function. 

Note that $r_{\btheta}(u) = \frac{h_{\btheta}(u)}{f_{\btheta}(u)}$, we can write the difference between $\hat r_k$ and $r$ as
	\begin{equation}\label{eq-diffr}
	\hat r_k(u,\btheta) - r_{\btheta}(u) =\frac{h_k(u,\btheta)}{f_k(u,\btheta)} - \frac{h_{\btheta}(u)}{f_{\btheta}(u)} = \frac{h_k(u,\btheta) - h_{\btheta}(u)}{f_k(u,\btheta)} + h_{\btheta}(u)\cdot [\frac{1}{f_k(u,\btheta)} - \frac{1}{f_{\btheta}(u)}].
	\end{equation}
The following lemmas are used as tools to control the right hand side of the above equation. The proof of the lemmas can be found in \S\ref{sec-pf-a1} and \ref{sec-pf-a2}.

	\begin{lemma}\label{lem-bias}
		Under Assumptions \ref{ass-pdfx} -- \ref{ass-kernel}, 
		 for any $b_k\leq 1$,
		\begin{align}
		\sup_{u\in I,\btheta\in\Theta_k} |\EE h_k(u,\btheta) - h_{\btheta}(u)|\leq C_{x, K}^{(1)}b_k^m,\label{eq-biash}\\
		\sup_{u\in I,\btheta\in\Theta_k} |\EE f_k(u,\btheta) - f_{\btheta}(u)|\leq C_{x, K}^{(1)}b_k^m.\label{eq-biasf}
		\end{align}
Here, $C_{x, K}^{(1)} = l_f\frac{\int |s^mK(s)|\db s}{(m-1)!}$.
	\end{lemma}
	\begin{lemma}\label{lem-dev}
		Under Assumptions \ref{ass-pdfx} -- \ref{ass-kernel}, $\forall b_k\leq 1$, $\delta\in[4e^{-nb_k/3}, \frac{1}{2})$, as long as 
		$nb_k\geq \max\{132d(\log \frac{1}{b_k} + 1), 3\log n\}$,
		either of the following inequalities holds with probability at least $1-\delta$:
		\begin{align}
		\sup_{u\in I,\btheta\in\Theta_k} | h_k(u,\btheta) - \EE h_k(u,\btheta)|\leq C_{x, K}^{(2)}\sqrt{\frac{{\log n}}{{nb_k}}}\left(\sqrt{d} + \sqrt{\log {1}/{\delta}}\right),\label{eq-devh}\\
		\sup_{u\in I,\btheta\in\Theta_k} | f_k(u,\btheta) - \EE f_k(u,\btheta)|\leq C_{x, K}^{(2)}\sqrt{\frac{{\log n}}{{nb_k}}}\left(\sqrt{d} + \sqrt{\log {1}/{\delta}}\right).\label{eq-devf}
		\end{align}
		Here 
		$C_{x, K}^{(2)} = l_K\bigg(8\sqrt{22}\max\{2\bar f \int K^2\db s, 2\bar f \int K'^2\db s, \frac{2}{3}\bar K, 1\} +$
		
		$ \frac{60(6\sqrt{\log 2}+\sqrt{c_0})}{c_0}\sqrt{1+R_{\cX}^2}\max\{\delta_z, \frac{\max\{1,\psi_x\}(B+{R_{\cX}R_{\Theta}})}{c_{\min}}\}\bigg)$ (Numerical constants are not optimized).
	\end{lemma}

	Now according to (\ref{eq-diffr}), we have 
	\begin{align}
	\sup_{u\in I,\btheta\in \Theta_k}|\hat r_k(u,\btheta)-r(u)|&\leq \sup_{u\in I,\btheta\in \Theta_k}\frac{|h_k(u,\btheta)-h_{\btheta}(u)|}{|f_{\btheta}(u) - |f_k(u,\btheta) - f_{\btheta}(u)||} + \sup_{u\in I,\btheta\in \Theta_k} \frac{h_{\btheta}(u)}{f_{\btheta}(u)}\cdot \frac{|f_k(u,\btheta) - f_{\btheta}(u)|}{|f_{\btheta}(u) - |f_k(u,\btheta) - f_{\btheta}(u)||}\nonumber\\
	&\leq \frac{\sup_{u\in I,\btheta\in \Theta_k}|h_k(u,\btheta) - h_{\btheta}(u)|}{c - \sup_{u\in I,\btheta\in \Theta_k}|f_k(u,\btheta)-f_{\btheta}(u)|} + \sup_{u\in I,\btheta\in \Theta_k}r_{\btheta}(u)\cdot \frac{\sup_{u\in I,\btheta\in \Theta_k}|f_k(u,\btheta) - f_{\btheta}(u)|}{c - \sup_{u\in I,\btheta\in \Theta_k}|f_k(u,\btheta) - f_{\btheta}(u)|}\nonumber\\
	&\leq \frac{\sup_{u\in I,\btheta\in \Theta_k}|h_k(u,\btheta) - h_{\btheta}(u)|}{c - \sup_{u\in I,\btheta\in \Theta_k}|f_k(u,\btheta)-f_{\btheta}(u)|} + \frac{\sup_{u\in I,\btheta\in \Theta_k}|f_k(u,\btheta) - f_{\btheta}(u)|}{c - \sup_{u\in I,\btheta\in \Theta_k}|f_k(u,\btheta) - f_{\btheta}(u)|}\label{eq-diffr2}
	\end{align}
	as long as we ensure that $\sup_{u\in I,\btheta\in \Theta_k}|f_k(u,\btheta) - f_{\btheta}(u)|\leq \frac{c}{2}$.
	
	Let $b_k = n^{-\frac{1}{2m+1}}$. {By letting $B_{x, K} = \max\{4{C_{x, K}^{(3)}}^8 / c^8, (2c_0)^{4}, (2C_b)^4\}$, we can verify that for any qualifying episode $k$, $nb_k\geq \max\{C_bd(\log \frac{1}{b_k} + 1), 3\log n\}$}. Combining (\ref{eq-biash}) and (\ref{eq-devh}), we have that $\forall \delta\in[4\exp(-n^{\frac{2m}{2m+1}}/3), \frac{1}{2})$, with probability at least $1-\delta$, 
	\begin{align*}
	\sup_{u\in I,\btheta\in\Theta_k}|h_k(u,\btheta) - h_{\btheta}(u)| &\leq \sup_{u\in I,\btheta\in\Theta_k}|h_k(u,\btheta) - \EE h_k(u,\btheta)|+\sup_{u\in I,\btheta\in\Theta_k}|\EE h_k(u,\btheta) - h_{\btheta}(u)|\\
	&\leq C_{x, K}^{(1)} n^{-\frac{m}{2m+1}} + C_{x, K}^{(2)}\sqrt{\frac{{\log n}}{{nb_k}}}\left(\sqrt{d} + \sqrt{\log {1}/{\delta}}\right)\\
	&\leq C_{x, K}^{(3)}n^{-\frac{m}{2m+1}}\sqrt{\log n}\left(\sqrt{d} + \sqrt{\log {1}/{\delta}}\right).
	\end{align*}
	Here, $C_{x, K}^{(3)} = C_{x, K}^{(1)} + C_{x, K}^{(2)}$. Similarly, with probability at least $1-\delta$, 
	$$
	\sup_{u\in I,\btheta\in\Theta_k}|f_k(u,\btheta) - f_{\btheta}(u)|\leq C_{x, K}^{(3)}n^{-\frac{m}{2m+1}}\sqrt{\log n}\left(\sqrt{d} + \sqrt{\log {1}/{\delta}}\right).
	$$
It's easily seen that as long as $n^{\frac{m}{2m+1}}/\sqrt{\log n}\geq\frac{{2C_{x, K}^{(3)}}}{c}(\sqrt{d}+\sqrt{\log 1/\delta}),$ 
		The right hand side of the above inequality is upper bounded by $c/2$, which guarantees that 
	$$
	\sup_{u\in I,\btheta\in\Theta_k}|f_k(u,\btheta) - f_{\btheta}(u)|\leq \frac{c}{2}.
	$$
(\emph{Remark}: From the conditions in the lemma, by letting $B_{x, K} = \max\{4{C_{x, K}^{(3)}}^8 / c^8, (2c_0)^{4}, (2C_b)^4\}$ and $B_{x, K}' = \min\{\big(\frac{c}{4C_{x, K}^{(3)}}\big)^2, 1/3\}$, we have
$$
n^{\frac{m}{2m+1}}/\sqrt{\log n}\geq\frac{{4C_{x, K}^{(3)}}}{c}\sqrt{d},\quad 
n^{\frac{m}{2m+1}}/\sqrt{\log n}\geq\frac{{4C_{x, K}^{(3)}}}{c}\sqrt{\log \frac{1}{\delta}}, $$
which lead to $n^{\frac{m}{2m+1}}/\sqrt{\log n}\geq\frac{{2C_{x, K}^{(3)}}}{c}(\sqrt{d}+\sqrt{\log 1/\delta})$.)

Plugging the above results into inequality (\ref{eq-diffr2}) gives 
	\begin{align}\label{conc1}
	\sup_{u\in I,\btheta\in\Theta_k}|\hat r_k(u,\btheta)-r_{\btheta}(u)|\leq \frac{4C_{x, K}^{(3)}}{c}n^{-\frac{m}{2m+1}}\sqrt{\log n}\left(\sqrt{d} + \sqrt{\log {1}/{\delta}}\right).
	\end{align}

	Next, we proceed to upper bound the quantity $\sup_{t\in I,\btheta\in \Theta_k}|r_{\btheta}(u)-r_{\btheta_0}(u)|$.
	We know that for any $\btheta\in \Theta_k$, 
	$$r_{\btheta}(u)=\EE[Y_t\given p_t - \tilde \xb_t^\top\btheta=u]=\EE[ \EE[Y_t\given \tilde \xb_t,p_t]\given p_t-\tilde\xb_t^\top\btheta=u]=\EE[r_{\btheta_0}(p_t - \tilde \xb_t^\top\btheta_0)\given p_t- \tilde \xb_t^\top\btheta=u].$$
	Moreover from the Lipchitz property of $r_{\btheta_0}$,
	$$\sup_{\xb\in \cX,\btheta\in \Theta_k}|r_{\btheta_0}(p_t - \tilde \xb^\top\btheta_0)-r_{\btheta_0}(p_t - \tilde \xb^\top\btheta)|\le{l_r}R_{\cX} R_k = {l_r}R_{\cX} \cdot \frac{10\max\{\psi_x,1\}(B+R_{\cX}R_{\Theta})}{c_{\min}}\sqrt{\frac{(d+1)\log n}{n}}.$$
	Therefore,
	\begin{align}\label{conc2}
	 \sup_{u\in I, \btheta\in \Theta_k}|r_{\btheta}(u)-r_{\btheta_0}(u)|\le\EE\Big[\sup_{\xb\in \cX,\btheta\in \Theta_k}|r_{\btheta_0}(p_t - \tilde\xb^\top\btheta_0)-r_{\btheta_0}(p_t - \tilde \xb^\top\btheta)|\given p_t-\tilde\xb_t^\top\btheta=u\Big]\le C_{x, K}^{(4)}\sqrt{\frac{d\log n}{n}},
	\end{align}
	where $C_{x, K}^{(4)} = {l_r}R_{\cX} \cdot \frac{10\max\{\psi_x,1\}(B+{R_{\cX}R_{\Theta}})}{c_{\min}}$.
	
	Finally, after combing our results in \eqref{conc1}-\eqref{conc2}, we claim our conclusion for Lemma \ref{conckernel_main}.

\subsection{Proof of Lemma \ref{conckernel2_main}
}\label{inf_deri}
Following the same settings as in the proof of Lemma \ref{conckernel_main}, we now aim at bounding the quantity $\sup_{u\in I, \btheta\in \Theta_k}|\hat r_k^{(1)}(u, \btheta) - r_{\btheta_0}'(u)|$, where 
$$
r_k^{(1)}(u,\btheta) = \frac{h_k^{(1)}(u,\btheta)f_k(u,\btheta)-h_k(u,\btheta)f_k^{(1)}(u,\btheta)}{f_k^2(u,\btheta)},
$$
$$
h_k(u,\btheta) = \frac{1}{nb_k}\sum_{u\in I_k} K(\frac{w_t(\btheta)-u}{b_k})Y_t,\quad f_k(u,\btheta) = \frac{1}{nb_k}\sum_{t\in I_k} K(\frac{w_t(\btheta)-u}{b_k}),
$$
$$
h_k^{(1)}(u,\btheta) = \frac{-1}{nb_k^2}\sum_{t\in I_k} K'(\frac{w_t(\btheta)-u}{b_k})Y_t,\quad f_k^{(1)}(u,\btheta) = \frac{-1}{nb_k^2}\sum_{t\in I_k} K'(\frac{w_t(\btheta)-u}{b_k}).
$$

Similar to the proof of Lemma \ref{conckernel_main}, we will bound $\sup_{u\in I, \btheta\in \Theta_k}|\hat r_k^{(1)}(u, \btheta) - r_{\btheta}'(u)|$ and $\sup_{u\in I, \btheta\in \Theta_k}|r_{\btheta}'(u) - r_{\btheta_0}'(u)|$ separately. First, notice that 
$$
r'_{\btheta}(u) = \frac{h_{\btheta}'(u)f_{\btheta}(u)-f'_{\btheta}(u)h_{\btheta}(u)}{f_{\btheta}^2(u)},
$$
we can bound $\sup_{u\in I, \btheta\in \Theta_k}|\hat r_k^{(1)}(u, \btheta) - r_{\btheta}'(u)|$ from the following four terms: $\sup_{u\in I, \btheta\in \Theta_k}|f_k(u,\btheta) - f_{\btheta}(u)|$, $\sup_{u\in I, \btheta\in \Theta_k}|h_k(u,\btheta) -h_{\btheta}(u)|$, $\sup_{u\in I, \btheta\in \Theta_k}|f_k^{(1)}(u,\btheta) - f'_{\btheta}(u)|$ and $\sup_{u\in I, \btheta\in \Theta_k}|h_k^{(1)}(u,\btheta) - h'_{\btheta}(u)|$. In fact, we can upper bound the first two terms from Lemma \ref{lem-bias} and \ref{lem-dev}. The lemmas below help us bound the last two terms. The proof can be found in \S\ref{sec-pf-a3} and \ref{sec-pf-a4}.

\begin{lemma}\label{lem-bias_2}
Given Assumptions \ref{ass-pdfx}-\ref{ass-kernel}, for any $b_k\leq 1$,
\begin{align}
	\sup_{u\in I,\btheta\in\Theta_k} |\EE h_k^{(1)}(u,\btheta) - h_{\btheta}'(u)|\leq C_{x, K}^{(5)}b_k^{m-1},\label{eq-biash_2}\\
	\sup_{u\in I,\btheta\in\Theta_k} |\EE f_k^{(1)}(u,\btheta) - f_{\btheta}'(u)|\leq C_{x, K}^{(5)}b_k^{m-1}.\label{eq-biasf_2}
	\end{align}
Here, $C_{x, K}^{(5)}=\frac{l_f}{(m-2)!}\int |K(s)s^{m-1}|\ud s$.
\end{lemma}

\begin{lemma}\label{lem-dev_2}
		Given assumptions \ref{ass-pdfx}, \ref{ass-F} and \ref{ass-kernel}, $\forall b_k\in [\frac{1}{n}, 1]$, $\delta\in[4e^{-nb_k/3}, \frac{1}{2})$, either of the following inequalities holds with probability at least $1-\delta$:
		\begin{align}
		\sup_{u\in I,\btheta\in\Theta_k} | h_k^{(1)}(u,\btheta) - \EE h_k^{(1)}(u,\btheta)|\leq C_{x, K}^{(2)}\sqrt{\frac{{\log n}}{{nb_k^3}}}\left(\sqrt{d} + \sqrt{\log {1}/{\delta}}\right),\label{eq-devh_2}\\
		\sup_{u\in I,\btheta\in\Theta_k} | f_k^{(1)}(u,\btheta) - \EE f_k^{(1)}(u,\btheta)|\leq C_{x, K}^{(2)}\sqrt{\frac{{\log n}}{{nb_k^3}}}\left(\sqrt{d} + \sqrt{\log {1}/{\delta}}\right).\label{eq-devf_2}
		\end{align}
		Here $C_{x, K}^{(2)} = l_K\bigg(8\sqrt{22}\max\{2\bar f \int K^2\db s, 2\bar f \int K'^2\db s, \frac{2}{3}\bar K, 1\} +$
		
		$ \frac{60(6\sqrt{\log 2}+\sqrt{c_0})}{c_0}\sqrt{1+R_{\cX}^2}\max\{\delta_z, \frac{\max\{1,\psi_x\}(B+{R_{\cX}R_{\Theta}})}{c_{\min}}\}\bigg)$ (Numerical constants are not optimized).

	\end{lemma}
			
	Now let $b_k = n^{-\frac{1}{2m+1}}$. Combining (\ref{eq-biash_2}) and (\ref{eq-devh_2}), we obtain that $\forall \delta\in[4\exp(-n^{\frac{2m}{2m+1}}/3), \frac{1}{2})$, with probability at least $1-\delta$, 
	\begin{align*}
	\sup_{u\in I,\btheta\in\Theta_k}|h_k^{(1)}(u,\btheta) - h_{\btheta}'(u)| &\leq \sup_{u\in I,\btheta\in\Theta_k}|h_k^{(1)}(u,\btheta) - \EE h_k^{(1)}(u,\btheta)|+\sup_{u\in I,\btheta\in\Theta_k}|\EE h_k^{(1)}(u,\btheta) - h_{\btheta}'(u)|\\
	& \leq C_{x, K}^{(5)} n^{-\frac{m-1}{2m+1}} + C_{x, K}^{(2)}\sqrt{\frac{{\log n}}{{nb_k^3}}}\left(\sqrt{d} + \sqrt{\log {1}/{\delta}}\right)\\
	&\leq C_{x, K}^{(6)}n^{-\frac{m-1}{2m+1}}\sqrt{\log n}\left(\sqrt{d} + \sqrt{\log {1}/{\delta}}\right)
	\end{align*}
	Here, $C_{x, K}^{(6)} = C_{x, K}^{(5)} + C_{x, K}^{(2)}$. Similarly, with probability at least $1-\delta$, 
	$$
	\sup_{u\in I,\btheta\in\Theta_k}|f_k^{(1)}(u,\btheta) - f_{\btheta}'(u)|\leq C_{x, K}^{(6)}n^{-\frac{m-1}{2m+1}}\sqrt{\log n}\left(\sqrt{d} + \sqrt{\log {1}/{\delta}}\right).
	$$

Recall that when $n^{\frac{m}{2m+1}}/\sqrt{\log n}\geq\frac{{2C_{x, K}^{(3)}}}{c}(\sqrt{d}+\sqrt{\log 1/\delta})$, we have
	$$
	\sup_{u\in I,\btheta\in\Theta_k}|f_k(u,\btheta) - f_{\btheta}(u)|\leq \frac{c}{2},\quad \sup_{u\in I,\btheta\in\Theta_k}|h_k(u,\btheta) - h_{\btheta}(u)|\leq \frac{c}{2}.
	$$
Moreover, we have 
$$
\sup_{u\in I, \btheta\in \Theta_k} \max\{|h_{\btheta}(u)|, |f_{\btheta}(u)|, |f'_{\btheta}(u)|\}\leq \bar f,
\quad
\sup_{u\in I, \btheta\in \Theta_k} |h'_{\btheta}(u)| = \sup_{u\in I, \btheta\in \Theta_k}|f'_{\btheta}(u)r_{\btheta}(u)+f_{\btheta}(u)r'_{\btheta}(u)|\leq l_f+l_r\bar f.
$$

Therefore, from the definition of $r_k^{(1)}(u,\btheta)$ and $r_{\btheta}'(u)$, we have
	\begin{align}
&\sup_{u\in I, \btheta\in \Theta_k}|r_k^{(1)}(\btheta,u)-r_{\btheta}'(u)|\nonumber\\
\le &\sup_{u\in I, \btheta\in \Theta_k}\left|[h_{\btheta}'(u)f_{\btheta}(u)-h_{\btheta}(u)f'_{\btheta}(u)]\bigg[\frac{1}{f_k(u,\btheta)^2}-\frac{1}{f_{\btheta}(u)^2}\bigg]\right|\nonumber\\
&+\sup_{u\in I, \btheta\in \Theta_k}\left|\frac{1}{f_k(u,\btheta)^2}\{[h_k^{(1)}(u,\btheta)f_k(u,\btheta)-h_k(u,\btheta)f_k^{(1)}(u,\btheta)]-[h_{\btheta}'(u)f_{\btheta}(u)-h_{\btheta}(u)f'_{\btheta}(u)]\right|\nonumber\\
\le & [l_f\bar f + (l_r + 1)\bar f^2]\cdot \sup_{u\in I, \btheta\in \Theta_k}\left|\frac{f_k(u,\btheta)^2 - f_{\btheta}(u)^2}{f_k(u,\btheta)^2f_{\btheta}(u)^2}\right|\nonumber\\
&+\sup_{u\in I, \btheta\in \Theta_k}\frac{1}{f_k(u,\btheta)^2}
|(h_k^{(1)}(u,\btheta) - h_{\btheta}'(u))f_k(u,\btheta) + 
h_{\btheta}'(u)(f_k(u,\btheta) - f_{\btheta}(u))\nonumber\\ 
&\quad \quad\quad\quad\quad\quad\quad\quad\quad 
-(f_k^{(1)}(u,\btheta) - f_{\btheta}'(u))h_k(u,\btheta)- 
f_{\btheta}'(u)(h_k(u,\btheta) - h_{\btheta}(u))|\nonumber\\ 
\le &[l_f\bar f + (l_r + 1)\bar f^2]\cdot \sup_{u\in I, \btheta\in \Theta_k}\frac{\frac{5}{2}f_{\btheta}(u)|f_k(u,\btheta) - f_{\btheta}(u)|}{f_k(u,\btheta)^2f_{\btheta}(u)^2}\nonumber \\
&+ \sup_{u\in I, \btheta\in \Theta_k}\frac{1}{f_k(u,\btheta)^2}[\sup_{u\in I, \btheta\in \Theta_k}|f_k(u,\btheta)|\cdot |h_k^{(1)}(u,\btheta) - h_{\btheta}'(u))| + \sup_{u\in I, \btheta\in \Theta_k} |h_{\btheta}'(u)|\cdot |f_k(u,\btheta) - f_{\btheta}(u)|\nonumber\\
&\quad \quad\quad\quad\quad\quad\quad\quad\quad 
+\sup_{u\in I, \btheta\in \Theta_k} |h_k(u,\btheta)|\cdot |f_k^{(1)}(u,\btheta) - f_{\btheta}'(u)|
+\sup_{u\in I, \btheta\in \Theta_k}|f_{\btheta}'(u)|\cdot |h_k(u,\btheta) - h_{\btheta}(u)|
].\nonumber\\
\le & C_{x, K}^{(7)}n^{-\frac{m-1}{2m+1}}\sqrt{\log n}\left(\sqrt{d} + \sqrt{\log\frac{1}{\delta}}\right).\label{conc1_2}
	\end{align}
when $n^{\frac{m}{2m+1}}/\sqrt{\log n}\geq\frac{{2C_{x, K}^{(3)}}}{c}(\sqrt{d}+\sqrt{\log 1/\delta})$. Here
$$
C_{x, K}^{(7)} = \big(\frac{10}{c^3}+\frac{4}{c^2}\big)[l_f(\bar f+1)+(l_r+1)\bar f^2]C_{x, K}^{(3)} + \big(\frac{8\bar f}{c^2}+\frac{4}{c}\big)C_{x, K}^{(6)}.
$$

	Next, we bound the term $\sup_{u\in I,\btheta\in \Theta_k}|r_{\btheta}'(u)-r_{\btheta_0}'(u)|$. In fact, according to our assumptions,
	\begin{align}\label{conc2_2}
	\sup_{u\in I,\btheta\in \Theta_k}|r_{\btheta}'(u)-r_{\btheta_0}'(u)|\le C_{x, K}^{(4)}\sqrt{\frac{d\log n}{n}},
	\end{align}
	where $C_{x, K}^{(4)} = {l_r}R_{\cX} \cdot \frac{8\max\{\psi_x,1\}(B+{R_{\cX}R_{\Theta}})}{c_{\min}}$.
	Finally, after combing our results in \eqref{conc1_2}-\eqref{conc2_2}, we claim our conclusion for Lemma \ref{conckernel2_main}.

\subsection{Proof of Lemma \ref{inverse_conv_main}}

We'll need the following auxiliary result in order to prove the lemma. The proof of Lemma \ref{lemma_inverse} can be found in section \ref{proof_lemma3.9}.
\begin{lemma}\label{lemma_inverse}
Given conditions of Lemma \ref{inverse_conv_main}, for any $\tilde{\xb}_t\in \cX$ and $\btheta\in\Theta_{0}$, $\btheta^\top\tilde{\xb}_t\in [\delta_z, B-\delta_z ]$.
\end{lemma}

Now we proceed to the proof. First, we seek an uniform upper bound for $|\hat\phi_k(u)-\phi(u)|$ from lemma \ref{conckernel_main} and \ref{conckernel2_main}. 
Recall that $\phi(u) = u - \frac{1-F(u)}{F'(u)}$ and $\hat\phi_k(u)=u-\frac{1-\hat{F}_k(u)}{\hat F_{k}^{(1)}(u)}$. It's easy to see that the desired uniform bound can be achieved on an interval where $F'$ is bounded below from 0. For this reason, we choose some positive constant $c_{F'}$ and some interval $[l_{F'}, r_{F'}]$ (we'll specify how to choose them later) such that 
\begin{equation}\label{eq-lowerbound-F'}
\inf_{u\in[l_{F'}, r_{F'}]} F'(u)\geq c_{F'}.
\end{equation}

From Lemma \ref{conckernel2_main} we know that if in addition
$|I_k|^{\frac{m-1}{2m+1}}\ge\frac{2\tilde{C}_{x,K}}{c_{F'}}\sqrt{\log |I_k|}(\sqrt{d}+\sqrt{\log 1/\delta})$, then\\ $\sup_{u\in[l_{F'}, r_{F'}]}|\hat F_{k}^{(1)}(u)-F'(u)|\le \frac{c_{F'}}{2}$ with probability at least $1-4\delta$. In fact, the above condition is ensured by 
$$
T\geq \bigg(\frac{4\tilde C_{x, K}}{c_{F'}}\bigg)^8 (\log T +2\log d)^{\frac{4m-1}{m-1}} d^{\frac{2m+1}{m-1}}.
$$

Combining \eqref{eq-lowerbound-F'}, Lemma \ref{conckernel_main} and Lemma \ref{conckernel2_main}, we deduce that with probability at least $1-6\delta$, 

\begin{align}
 	\sup_{u\in[l_{F'}, r_{F'}]}|\hat\phi_k(u)-\phi(u)|&\le \sup_{u\in[l_{F'}, r_{F'}]}\Big|\frac{(1-\hat F_k(u))(F'(u)-\hat F_k^{(1)}(u))}{\hat F_k^{(1)}(u)F'(u)}\Big|\nonumber\\
	&\quad+\sup_{v\in[l_{F'}, r_{F'}]}\Big|\frac{\hat F_k(u)-F(u)}{F'(u)}\Big|\nonumber\\
	&\le \frac{2\tilde{C}_{x,K}+C_{x,K}c_{F'}}{c_{F'}^2}|I_k|^{-\frac{m-1}{2m+1}}\sqrt{\log|I_k|}\left(\sqrt{d}+\sqrt{\log \frac{1}{\delta}}\right)\label{eq-pf3.9-phi}
 \end{align}

Next, we proceed to bound $\sup_{u\in[\delta_z, B-\delta_z]}|\hat g_k(u) - g(u)|$ from $\sup_{u\in[\delta_z, B-\delta_z]}|\hat \phi_k^{-1}(-u) - \phi^{-1}(-u)|$ for some properly defined $\hat \phi_k^{-1}$. To be more specific, we will also let 
\begin{equation}\label{eq-pf3.9-range}
[\delta_z-B, -\delta_z]\subseteq \phi([l_{F'}, r_{F'}])\cap \hat \phi_k([l_{F'}, r_{F'}]).
\end{equation}

The way we ensure the above is the following: First, according to the assumptions, we know $\phi'(u)\geq c_{\phi}>0$, and that $\lim_{u\rightarrow \delta_z-0}\phi(u) = \delta_z$, $\lim_{u\rightarrow l_{F}^{(1)}+0}\phi(u) = -\infty$ with $l_{F}^{(1)} = \inf\{u: F'(u) >0\}>-\delta_z$. We can deduce that
$$
m_{F'} = \inf_{u\in[\phi^{-1}(\delta_z-B), \phi^{-1}(-\delta_z)]}F'(u)>0.
$$
Therefore, there exists some $\delta_{F'}>0$ such that 
$$
\inf_{u\in[\phi^{-1}(\delta_z-B) - \delta_{F'}, \phi^{-1}(-\delta_z)+\delta_{F'}]}F'(u)>\frac{m_{F'}}{2}.
$$
Now let $l_{F'} = \phi^{-1}(\delta_z-B) - \delta_{F'}$, $r_{F'} = \phi^{-1}(-\delta_z)+\delta_{F'}$, $c_{F'} = \frac{m_{F'}}{2}$. From the assumptions on $\phi$, we have
$$
\phi(l_{F'})\leq \delta_z-B-c_{\phi}\delta_{F'},\quad \phi(r_{F'})\geq -\delta_z + c_{\phi}\delta_{F'}.
$$
Combining \eqref{eq-pf3.9-phi}, we obtain that as long as 
$$
\frac{2\tilde{C}_{x,K}+C_{x,K}c_{F'}}{c_{F'}^2}|I_k|^{-\frac{m-1}{2m+1}}\sqrt{\log|I_k|}\left(\sqrt{d}+\sqrt{\log \frac{1}{\delta}}\right)\leq c_{\phi}\delta_{F'},
$$
we can ensure \eqref{eq-pf3.9-range}. The above condition can be obtained from the fact that
$$
T\geq \bigg(\frac{4\tilde{C}_{x,K}+2C_{x,K}c_{F'}}{c_{F'}^2c_{\phi\delta_{F'}}}\bigg)^8 (\log T +2\log d)^{\frac{4m-1}{m-1}} d^{\frac{2m+1}{m-1}}.
$$

Define 
\begin{equation}
\hat\phi_k^{-1}(u):=\inf\{v\in[l_{F'},r_{F'}] :\hat\phi_k(v)=u \}.
\end{equation}
We proceed to upper bound $\sup_{u\in [\delta_z-B,-\delta_z]}|\hat\phi_k^{-1}(u)-\phi^{-1}(u)|$. In fact, for any $u$, let $v_1=\phi^{-1}(u)$, $v_2=\hat\phi_k^{-1}(u)$. Then 

\begin{align*}
 |v_1-v_2|&\le 1/c_\phi\cdot|\phi(v_1)-\phi(v_2)|=1/c_\phi \cdot|\hat\phi_k(v_2)-\phi(v_2)|\\
 &\le 1/c_{\phi}\cdot \sup_{v\in [l_{F'},r_{F'}]}|\hat\phi_k(v)-\phi(v)|\\
 &\le\frac{2\tilde{C}_{x,K}+C_{x,K}c_{F'}}{c_{\phi}c_{F'}^2}|I_k|^{-\frac{m-1}{2m+1}}\sqrt{\log|I_k|}\left(\sqrt{d}+\sqrt{\log \frac{1}{\delta}}\right)
 \end{align*}
with probability at least $1-6\delta$.

Finally, since $g(u) = u + \phi^{-1}(-u)$ and $\hat g_k(u) = u+\hat \phi_k^{-1}(-u)$, we conclude Lemma \ref{inverse_conv_main} by choosing 
$$
\bar B_{x, K} = \max\left\{B_{x, K}, (\frac{4\tilde C_{x, K}}{c_{F'}})^8, \bigg(\frac{4\tilde{C}_{x,K}+2C_{x,K}c_{F'}}{c_{F'}^2c_\phi\delta_{F'}}\bigg)^8, \Big[\frac{C_{\btheta}^2}{\delta_v^2}(1+R_\cX^2)\Big]^{\frac{2(4m-1)}{2m+1}}\right\},
$$
$$
\bar B'_{x, K} =  \min\left\{B'_{x, K}, (\frac{c_{F'}}{4\tilde C_{x, K}})^2, (\frac{c_{F'}^2c_\phi\delta_{F'}}{4\tilde{C}_{x,K}+2C_{x,K}c_{F'}})^2\right\},
$$
and 
$$
\bar C_{x, K} = \frac{2\tilde{C}_{x,K}+C_{x,K}c_{F'}}{c_{\phi}c_{F'}^2}.
$$
 
\subsection{Proof of Theorem \ref{mainthm}}
In order to bound the total regret, we first try to bound the regret at each episode $k$. First, for all $k\leq \lfloor (\log(\sqrt{T} + \ell_0)-\log \ell_0)\log 2\rfloor + 1$, we bound the total regret during episode $k$ by $B\ell_k$. It can be easily verified that 
$$
\sum_{k\leq \lfloor (\log(\sqrt{T} + \ell_0)-\log \ell_0)\log 2\rfloor + 1}\textrm{Regret}_k\leq 2B\sqrt{T}.
$$

We now turn to the case where $k> \lfloor (\log(\sqrt{T} + \ell_0)-\log \ell_0)\log 2\rfloor + 1$. Recall that the conditional expectation of regret at time $t$ given previous information and $\tilde \xb_t$ is 
$$
\EE[R_t\given \bar{\cH}_{t-1}]
 = \EE[p_t^{*}\II_{(v_t\ge p_t^{*})}-p_t\II_{(v_t\ge p_t)}\given \bar \cH_t]
= \rho_t(p_t^*) - \rho_t(p_t),
$$
where $\bar{\cH}_t=\sigma(\xb_1, \xb_2, \cdots, \xb_{t+1}; z_1, \cdots, z_t)$, and we denote $\rho_t(p) := p(1-F(p-\btheta_0^\top\tilde \xb_t))$. Using Taylor expansion and the first order condition induced by the optimality of $p_t^*$, we have 
$$
\rho_t(p_t) = \rho_t(p_t^*)+\frac{1}{2}\rho_t''(\xi_t)(p_t-p_t^*)^2,
$$
where $\xi_t$ is some value lying between $p_t$ and $p_t^*$. Note that for any $p\in [0, B]$, $|\rho''_t(p)| = |2F'(p-\btheta_0^\top\tilde \xb_t)-pF''(p-\btheta_0^\top\tilde \xb_t)|\leq 2l_r + Bl_r'$. Thus we deduce that
$$
\EE[R_t\given \bar{\cH}_{t-1}] = \rho_t(p_t^*) - \rho_t(p_t)\leq (2l_r + Bl_r')(p_t - p_t^*)^2,
$$
which further implies that the expected regret at time $t$ is bounded by 
\begin{equation}\label{eq-mainthm-Rt}
\EE R_t\leq \frac{1}{2}(2l_r + Bl_r')\EE (p_t - p_t^*)^2
\end{equation}

On the other hand, 
\begin{align*}
(p_t - p_t^*)^2&\leq (\hat g_k(\tilde \xb_t^\top \hat \btheta_k) - g(\tilde \xb_t^\top \btheta_0))^2\\
&\leq 2(\hat g_k(\tilde \xb_t^\top \hat \btheta_k) - g(\tilde \xb_t^\top \hat \btheta_k))^2 + 2(g(\tilde \xb_t ^\top \hat \btheta_k) - g(\tilde \xb_t^\top \btheta_0))^2\\
&:= \bJ_1 + \bJ_2.
\end{align*}
We first analyze $\bJ_2$. In fact, define the event 
$$
\cE_k := \{\|\hat\btheta_k - \btheta_0\|\leq R_k\},
$$
then according to Lemma \ref{thmsignal}, $\PP(\cE_k)\leq 1-2e^{-c_1c_{\min}^2|I_{k}|/16}-2/|I_k|$. On $\cE_k$ we have
$$
\bJ_2\leq \frac{2}{c_{\phi^2}}(\tilde \xb_t ^\top \hat \btheta_k - \tilde \xb_t^\top \btheta_0)^2\leq \frac{2}{c_{\phi^2}} R_{\cX}^2 \|\hat \btheta_k - \btheta_0 \|^2\leq \frac{2}{c_{\phi^2}} R_{\cX}^2R_k^2.
$$
Therefore,
\begin{equation}\label{eq-main-J2}
\EE \bJ_2\leq \frac{2}{c_{\phi^2}} R_{\cX}^2R_k^2 + 2B^2 (2e^{-c_1c_{\min}^2|I_{k}|/16}+2/|I_k|).
\end{equation}

As for $\bJ_1$, on the event $\cE_k$, we deduce from Lemma \ref{inverse_conv_main} that for any 
$\delta \in[\max\{4\exp(-\bar B_{x, K}|I_k|^{\frac{2m-2}{2m+1}}/\log |I_k|),\frac{1}{2})$, with probability at least $1-6\delta$,
\begin{align*}
\bJ_1\leq 2\bigg[\sup_{u\in[\delta_z, B-\delta_z] }(\hat g_k(u) - g(u))\bigg]^2\leq 2\bar C_{x, K}^2|I_K|^{-\frac{2(m-1)}{2m+1}}\log |I_K|\bigg(\sqrt{d} + \sqrt{\log \frac{1}{\delta}}\bigg)^2.
\end{align*}
By choosing $\delta = 1/|I_k|$, we have 
\begin{align}
\EE \bJ_1&\leq 2\bar C_{x, K}^2|I_K|^{-\frac{2(m-1)}{2m+1}}\log |I_K|\bigg(\sqrt{d} + \sqrt{\log \frac{1}{\delta}}\bigg)^2 + 2B^2\cdot 6\delta\nonumber\\
&\leq 4\bar C_{x, K}^2|I_K|^{-\frac{2(m-1)}{2m+1}}\log |I_K|\bigg(d + \log |I_K|\bigg) + \frac{12B^2}{|I_k|}\label{eq-main-J1}
\end{align}

Combining \eqref{eq-mainthm-Rt}, \eqref{eq-main-J2} and \eqref{eq-main-J1}, we obtain an upper bound for the expected regret at any time $t$ during episode $k$:
$$
\EE R_t\leq \bar C_{x, K}^{(1)}|I_K|^{-\frac{2(m-1)}{2m+1}}\log |I_K|\bigg(d + \log |I_K|\bigg),
$$
where $\bar C^{(1)}_{x, K} = \frac{1}{2}(2l_r + Bl_r')\cdot [\frac{4}{c_{\phi}^2}R_{\cX}^2(\frac{10\max\{\psi_x, 1\}(R_{\cX}R_{\Theta}+B)}{c_{\min}})^2 + 20B^2+4C_{x, K}'^2]$. We choose $|I_k| = \lceil(l_kd)^{\frac{2m+1}{4m-1}}\rceil$. The total regret during the $k$-th episode is 
\begin{align*}
	\textrm{Regret}_k&=\sum_{t\in I_{k}}\EE R_t+\sum_{t\in I_k'}\EE R_t\\
	&\leq  B|I_k|+l_k\cdot \EE R_t\\
	&\leq B(l_kd)^{(2m+1)/(4m-1)} + B + l_k\cdot \bar C^{(1)}_{x, K}(l_kd)^{-(2m-2)/(4m-1)} \log T(d + \log T)\\
	&\leq (2B + \bar C^{(1)}_{x, K}) l_k^{\frac{2m+1}{4m-1}}d^{\frac{2m+1}{4m-1}}\log T(1 + \log T/d).
\end{align*}

Finally, the total regret defined in \eqref{eq:Regret_def} can be bounded by
\begin{align}
	&\text{Regret}_{\pi}(T)=\sum_{k=1}^{K} \textrm{Regret}_k \leq 2B\sqrt{T} + (2B + \bar C^{(1)}_{x, K})d^{\frac{2m+1}{4m-1}}\log T(1 + \log T/d)\sum_{k=1}^K  l_k^{(2m+1)/(4m-1)}\nonumber\\
&\leq \bigg[ 2B+\frac{2l_0^{(2m+1)/(4m-1)} (2B + \bar C^{(1)}_{x, K})}{2^{(2m+1)/(4m-1)}-1}\bigg](Td)^{\frac{2m+1}{4m-1}}\log T\bigg(1 + \frac{\log T}{d}\bigg).
\end{align}
Here $K=\lceil\log_2 T\rceil$. The proof is then finished by letting $C^*_{x, K} = 2B+\frac{2l_0^{(2m+1)/(4m-1)} (2B + \bar C_{x, K}^{(1)})}{2^{(2m+1)/(4m-1)}-1}$.

\section{Proof under the strong-mixing feature setting}\label{app_mix}

In this section, we mainly present the proof of Theorem \ref{mainthm2}. The proof will be decomposed to the following lemmas, and their proof is also attached.

Before stating the lemmas, we introduce the $\alpha$-mixing condition.
\begin{definition}\label{defalpha}[$\alpha$-mixing] For a sequence of random variables $x_i$ defined on a probability space $(\Omega,\cX,\PP)$, define
	\begin{align*}
		\alpha_k=\sup_{l\ge 0}\alpha(\sigma(x_t,t\le l),\sigma(x_t,t\ge l+k))
	\end{align*}
	in which
	\begin{align*}
		\alpha(\cA,\cB)=\sup_{A\in \cA, B\in \cB}\{|\PP(A\cap B)-\PP(A)\PP(B)|\}
	\end{align*}
\end{definition}
From the definition of strong $\beta$-mixing, we see that it can infer strong $\alpha$-mixing conditions. So in this case, our sequence $\xb_t$ also follows strong $\alpha$-mixing conditions, with $\alpha_k\le e^{-ck}$.

\begin{lemma}\label{para_mix}[Parametric estimation under dependence] Under Assumption \ref{ass:bound} and \ref{assmix1}, there exist positive constants $c_1$ and $c_2$ (only depend on constants given in Assumptions)  such that when $|I_k|\ge \max\{c_1(d+1),c_2\log^2|I_k|\log\log |I_k|\},$ for any episode $k$ within the horizon, with probability $1-4/|I_k|^2$, we obtain
	\begin{align*}
		\|\hat\btheta_k-\btheta_0\|_2\le \frac{2}{c_{\min}}\sqrt{\frac{(d+1)(6W^2_x\log |I_k|+6W_x\log^2 |I_k|\log\log |I_k|)}{C_w|I_k|}},
	\end{align*}
	where $W_x=2R_{\cX}(R_{\cX}R_{\Theta}+B).$
\end{lemma}
The proof of Lemma \ref{para_mix} can be found in \S\ref{proof_mixpara}. Next, we present the following results on estimation error of $F(\cdot)$ and $F'(\cdot)$:

\begin{lemma}\label{conckernel_main_mix}
{	Suppose that Assumptions \ref{ass-pdfx}, \ref{ass-F}, \ref{ass-kernel}, \ref{assmix1} and \ref{lipschitz_mix} hold. Then there exist constants $B_{mx,K},B_{mx,K}' ,C_{mx,K}$  only depending on $R_{\cX} := \sup_{\xb\in\cX}\|\xb\|_2$ and constants within assumptions,  such that as long as
	
$$
T\geq B_{mx, K} (\log T +2\log d)^{\frac{12m-3}{m}} [(d+1)\log (d+1)]^{\frac{4m-1}{m}}/d^2,
$$
we have for any $k\geq \lfloor(\log (\sqrt{T} + \ell_0) - \log \ell_0) / \log 2\rfloor + 2$,
and
$\delta\in[ 8\exp(-|I_k|^{\frac{2m}{2m+1}}/(B_{mx,K}'\log^2|I_k|)),1/2]$
	with probability at least $1-2\delta$,
	\begin{align}
		\sup_{u\in I,\btheta\in \Theta_{k}}& |\hat F_k(u, \btheta) - F(u)| \leq
		C_{mx, K}|I_k|^{-\frac{m}{2m+1}}\log|I_k|\Big(\sqrt{(d+1)\log (d+1)\log|I_k|}+ \sqrt{2\log \frac{8}{\delta}}\Big) \label{eq:conckernelm_main}.
	\end{align}
	Here $I=[-\delta_z,\delta_z]$ and we choose the bandwidth $b_k = |I_k|^{-\frac{1}{2m+1}}.$}

\end{lemma}

The proof of Lemma \ref{conckernel_main_mix} can be found in \S\ref{pf-b3}.

\begin{lemma}\label{conckernel2_main_mix}
	{ Suppose that Assumptions \ref{ass-pdfx}, \ref{ass-F}, \ref{ass-kernel}, \ref{assmix1} and \ref{lipschitz_mix} hold. Then there exist constants $\bar{B}_{mx,K},\bar{B}_{mx,K}',\bar{C}_{mx,K}$ that  depending only on $R_\cX := \sup_{\xb\in\cX}\|\xb\|_2$ and the constants within the assmptions such that as long as
$$
T\geq \bar{B}_{mx, K} (\log T +2\log d)^{\frac{12m-3}{m}} [(d+1)\log (d+1)]^{\frac{4m-1}{m}}/d^2,
$$
	for any $k\geq \lfloor(\log (\sqrt{T} + \ell_0) - \log \ell_0) / \log 2\rfloor + 2$ and
	$\delta\in[ \{8\exp(-|I_k|^{\frac{2m}{2m+1}}/(\bar{B}_{mx,K}'\log^2|I_k|)),1/2]$
	we have with probability at least $1-4\delta$,
	\begin{align}
		\sup_{u\in I,\btheta\in \Theta_{k}}& |\hat F_k^{(1)}(u, \btheta) - F'(u)| \leq
		\bar C_{mx, K}|I_k|^{-\frac{m-1}{2m+1}}\log|I_k|\Big(\sqrt{(d+1)\log (d+1)\log|I_k|}+ \sqrt{2\log \frac{8}{\delta}}\Big) \label{eq:conckernel2m_main}.
	\end{align}
	Here $I=[-\delta_z,\delta_z]$ and we choose the bandwidth $b_k = |I_k|^{-\frac{1}{2m+1}}$.}
\end{lemma}

The proof of this lemma can be found in  \S\ref{pf-b4}.

By combining these two lemmas and following our conclusions from Lemma \ref{inverse_conv_main}, we are able to achieve the regret bound at the same order with Theorem \ref{mainthm} in Theorem \ref{mainthm2}.

\section{Proof under the super smooth noise distribution setting}\label{app_inf}

Proof of Theorem \ref{mainthm3} can be followed directly from the proof of Theorem \ref{mainthm2} by substituting the Lemma \ref{inf_bias} with Lemma \ref{para_mix}. Below we'll only present the proof of Lemma \ref{inf_bias}.

\begin{proof}
	We only bound $\sup_{u\in I, \btheta\in \Theta_k}	|\EE[f_k(u,\btheta)]-f_{\btheta}(u)|$ and $\sup_{u\in I, \btheta\in \Theta_k}	|\EE[f_k^{(1)}(u,\btheta)]-f'_{\btheta}(u)|$, since the analysis for $f_k(u, \btheta)$ and $h_k(u,\btheta)$ are the same. In fact, under the settings of Lemma \ref{inf_bias}, for any $u\in I, \btheta\in \Theta_k$, 
	\begin{align*}
		\EE[f_k(u,\btheta)]-f_{\btheta}(u)&= \int_{\mathbb R} \frac{1}{b_k}K\Big(\frac{s-u}{b_k}\Big)f_{\btheta}(s)\ud s-f_{\btheta}(u)\\
		&=\cF\bigg(\cF^{-1}\Big(\int_{\mathbb R} \frac{1}{b_k}K\Big(\frac{s-u}{b_k}\Big)f_{\btheta}(s)\ud s\Big)-\cF^{-1}\circ f_{\btheta}(u)\bigg)\\
		&=\cF\bigg(\phi_{\btheta}(u)\Big[ \cF^{-1}\Big(\frac{1}{b_k}K\Big(\frac{-u}{b_k}\Big)\Big) - 1\Big] \bigg)\\
		&= \cF(\phi_{\btheta}(u)[ \kappa(-b_k u) - 1] ).
	\end{align*}
	Here $\cF$ is the Fourier transform operator defined by 
	$$
	g\rightarrow \cF\circ g(u) = \frac{1}{2\pi}\int_{\mathbb R} g(x) e^{-iux}\ud x ,
	$$
	and we've utilized the fact that $K = \cF \circ \kappa$, $\phi_{\btheta}(u) = \cF^{-1}\circ f_{\btheta}$. Since $|\kappa(x)|\leq 1$ for all $x\in \RR$ and that $\kappa(x) = 1$ for $|x|\le c_\kappa$,
	\begin{align*}
	\sup_{u\in I, \btheta\in \Theta_k}	|\EE[f_k(u,\btheta)]-f_{\btheta}(u)| &\le \sup_{u\in I, \btheta\in \Theta_k}|\cF( \phi_{\btheta}(u)[ \kappa(-b_k u) - 1] )|\\
	&\le \sup_{\btheta\in \Theta_k} \frac{1}{2\pi}\int |\phi_{\btheta}(s)|\cdot |\kappa(-b_k s) - 1|\ud s\\
	&\le \sup_{\btheta\in \Theta_0} \frac{1}{\pi}\int_{|s|>c_\kappa/b_k} |\phi_{\btheta}(s)|\ud s\\
	&\le \frac{2}{\pi}\int_{s>0} D_{\phi}e^{-d_{\phi}(s + c_\kappa/b_k)^\alpha}\ud s\\
        &\le \frac{2}{\pi}\int_{s>0} D_{\phi}e^{-d_{\phi}/2\cdot [s^\alpha + (c_\kappa/b_k)^\alpha]}\ud s.
	\end{align*}
	
	Here, the last inequality is due to the fact that for $x, y\in \RR$, $(x+y)^\alpha\geq \min\{2^{\alpha-1}, 1\}(x^\alpha+y^\alpha)\geq\frac{1}{2}(x^\alpha+y^\alpha)$. Thus, by choosing $b_k = c_{\kappa}(d_\phi/\log |I_k|)^{1/\alpha}$, we obtain that 
	$$
	\sup_{u\in I, \btheta\in \Theta_k}	|\EE[f_k(u,\btheta)]-f_{\btheta}(u)|\le C_{\inf} /\sqrt{n},
	$$
	where $C_{\inf} = 2D_\phi/\pi\cdot \int_{s>0}\exp(-d_\phi s^\alpha/2)\ud s$.
	
	The analysis for $\sup_{u\in I, \btheta\in \Theta_k}	|\EE[f_k^{(1)}(u,\btheta)]-f'_{\btheta}(u)|$ is similar as above. In fact, for any $u\in I, \btheta\in \Theta_k$, 
	\begin{align*}
		\EE[f_k^{(1)}(u,\btheta)]-f'_{\btheta}(u)&= -\int_{\mathbb R} \frac{1}{b_k^2}K'\Big(\frac{s-u}{b_k}\Big)f_{\btheta}(s)\ud s-f'_{\btheta}(u)\\
		&= \int_{\mathbb R} \frac{1}{b_k}K\Big(\frac{s-u}{b_k}\Big)f'_{\btheta}(s)\ud s-f'_{\btheta}(u)\\
		&=\cF\bigg(\cF^{-1}\Big(\int_{\mathbb R} \frac{1}{b_k}K\Big(\frac{s-u}{b_k}\Big)f'_{\btheta}(s)\ud s\Big)-\cF^{-1}\circ f'_{\btheta}(u)\bigg)\\
		&=\cF\bigg(\phi^{(1)}_{\btheta}(u)\Big[ \cF^{-1}\Big(\frac{1}{b_k}K\Big(\frac{-u}{b_k}\Big)\Big) - 1\Big] \bigg)\\
		&= \cF(\phi^{(1)}_{\btheta}(u)[ \kappa(-b_k u) - 1] ).
	\end{align*}
Following the same arguments as above, we deduce that 
	$$
	\sup_{u\in I, \btheta\in \Theta_k}	|\EE[f_k^{(1)}(u,\btheta)]-f'_{\btheta}(u)|\le C_{\inf} /\sqrt{n}.
	$$
\end{proof}

\section{Proof of technical lemmas}

\subsection{Proof of Lemma \ref{lem-bias}} \label{sec-pf-a1}

	We only prove (\ref{eq-biash}), since (\ref{eq-biasf}) can be proved in the same way.
	
	Recall that $h_k(u,\btheta) = \frac{1}{nb_k}\sum_{t=1}^n K(\frac{w_t(\btheta)-u}{b_k})y_t$, and $\EE[y_t|w_t(\btheta)=u] = r_{\btheta}(u) = \frac{h_{\btheta}(u)}{f_{\btheta}(u)}$. We have
	$$
	\EE h_k(u,\btheta) = \frac{1}{b_k} \EE K(\frac{w_t(\btheta) - u}{b_k})y_t=  \frac{1}{b_k} \EE K(\frac{w_t(\btheta) - u}{b_k})r(w_t(\btheta)).
	$$
	Thus,
	\begin{align}
	\EE h_k(u,\btheta)  - h_{\btheta}(u)&= \int \frac{1}{b_k} K(\frac{w(\btheta)-u}{b_k})r_{\btheta}(w(\btheta))f_{\btheta}(w(\btheta))\ud w(\btheta) - h_{\btheta}(u)\nonumber\\
	&=\int K(s)h_{\btheta}(u+b_ks) \ud s - h_{\btheta}(u).\label{eq-biash1}
	\end{align}
	Using Taylor's expansion, $\forall s\in \RR$, there exists some $\xi(s, u)$ lying between the points $u$ and $u+b_ks$ such that
	$$
	h_{\btheta}(u+b_ks) = h_{\btheta}(u) + \sum_{i=1}^{m-2}\frac{h_{\btheta}^{(i)}(u)}{i!}(b_ks)^{i} + \frac{h_{\btheta}^{(m-1)}(\xi(s, u))}{(m-1)!}(b_ks)^{m-1}.
	$$ 
	Plugging this into (\ref{eq-biash1}) gives
	\begin{align*}
	\EE h_k(u,\btheta)  - h_{\btheta}(u)&=\int K(s)\left[ h_{\btheta}(u) + \sum_{i=1}^{m-2}\frac{h_{\btheta}^{(i)}(u)}{i!}(b_ks)^{i} + \frac{h_{\btheta}^{(m-1)}(\xi(s, u))}{(m-1)!}(b_ks)^{m-1}\right]\ud s - h_{\btheta}(u)\\
	&= \int K(s)\frac{h_{\btheta}^{(m-1)}(\xi(s, u))}{(m-1)!}(b_ks)^{m-1} \ud s\\
	& = \int K(s)\frac{h_{\btheta}^{(m-1)}(u)}{(m-1)!}(b_ks)^{m-1} \ud s + \int K(s)\frac{[h_{\btheta}^{(m-1)}(\xi(s, u)) - h_{\btheta}^{(m-1)}(u)]}{(m-1)!}(b_ks)^{m-1} \ud s\\
	&= \int K(s)\frac{[h_{\btheta}^{(m-1)}(\xi(s, u)) - h_{\btheta}^{(m-1)}(u)]}{(m-1)!}(b_ks)^{m-1} \ud s.
	\end{align*}
	Thus we have that 
	\begin{align*}
	|\EE h_k(u,\btheta)  - h_{\btheta}(u)|& \leq \int |K(s)|\frac{|h_{\btheta}^{(m-1)}(\xi(s, u)) - h_{\btheta}^{(m-1)}(u)|}{(m-1)!}|b_ks|^{m-1} \ud s\\
	&\leq \int |K(s)| \frac{l_f|b_k s|}{(m-1)!}|b_ks|^{m-1} \ud s\\
	& \leq C_1 b_k^m,
	\end{align*}
	where $C_1 = l_f \cdot \int |s^mK(s)|\ud s / (m-1)!$. Moreover, since the inequality holds for any $u\in I$ and $\btheta\in \Theta_k$, we finish the proof.

\subsection{Proof of Lemma \ref{lem-dev}} \label{sec-pf-a2}

	We only prove (\ref{eq-devh}), since (\ref{eq-devf}) can be proved in the same way.
	
	For any $u\in I$, $\btheta\in \Theta_k$, denote $Z(u,\btheta) := h_k(u,\btheta) - \EE h_k(u,\btheta) = \frac{1}{nb_k}\sum_{t\in I_k}[K(\frac{w_t(\btheta)-u}{b_k})y_t - \EE K(\frac{w_t(\btheta)-u}{b_k})y_t]$. Then
	$$
	\sup_{u\in I,\btheta\in\Theta_k}|h_k(u,\btheta) - \EE h_k(u,\btheta)| =\sup_{u\in I,\btheta\in\Theta_k}|Z(u,\btheta)| = \max\Big\{\sup_{u\in I,\btheta\in\Theta_k} Z(u,\btheta), \sup_{u\in I,\btheta\in\Theta_k}(-Z(u,\btheta))\Big\}.
	$$
	We can then bound $\sup_{u\in I,\btheta\in\Theta_k}|h_k(u,\btheta) - \EE h_k(u,\btheta)|$ by upper bounding both $\sup_{u\in I,\btheta\in\Theta_k} Z(u,\btheta)$ and $\sup_{u\in I,\theta\in\Theta_k} (-Z(u,\btheta))$. We now give upper bound for $\sup_{u\in I,\btheta\in\Theta_k} Z(u,\btheta)$ with high probability (Bounding $\sup_{u\in I,\btheta\in\Theta_k} (-Z(u,\btheta))$ is essentially the same).
		
	We use the chaining method to obtain the desired bound. First, we construct a sequence of $\varepsilon$-nets with decreasing scale. Denote the left and right endpoints of the interval $I$ as $L_I$ and $R_I$ respectively. For any $i\in \NN^+$, construct set $S^{(i)}_1\subseteq I$ as 
	$$
	S^{(i)}_1\triangleq \left\{ L_I+\frac{j}{2^i\sqrt{n}}(R_I-L_I): j\in\{1, 2, \cdots, (2^i-1)\lceil\sqrt{n}\rceil\}\right\}.
	$$
	For any $u\in I$, $i\in\NN^+$, let $\pi_1^{(i)}(u) = \arg\min_{s\in S^{(i)}_1}|s-u|$. Moreover, let $\pi_1^{(0)}(u) = u$. Then we can easily verify that $|S^{(i)}_1|\leq 2^i(\sqrt{n}+1)$, and $\forall u\in I$, $|\pi_i(u) - \pi_{i+1}(u)|\leq \frac{2\delta_z}{2^{i-1}\sqrt{n}}$. At the same time, denote $S^{(i)}_2$ as a $R_k/{2^{i}}$-net with respective to $l_2$-distance of $\Theta_k$, where $R_k$ denotes the radius of $\Theta_k$. Similar to $\pi_1^{(i)}$, define $\pi_2^{(i)}(\bu) = \arg\min_{\bs \in S^{(i)}_2}|\bu-\bs|$. By Corollary 4.2.13 in \cite{vershynin_2018}, $|S_2^{(i)}|\le (2^{i+1}+1)^{d}$. 
	
Combining the above two nets, we have $S^{(i)}:=S_1^{(i)}\times S_2^{(i)}$ is a $2^{-i}\sqrt{4\delta_z^2/n+R_k^2}$-net of $U_k:=I\times \Theta_k$ with cardinality $|S^{(i)}|\le2^i(\sqrt{n}+1)\cdot(2^{i+1}+1)^{d}$. In fact, for any $\bu:=(u,\btheta)\in I\times \Theta_k$ with $i\ge 1$, denote $\pi_i(\bu):=(\pi_1^{(i)}(u),\pi_2^{(i)}(\btheta))$, then $\|\pi_i(\bu)-\bu\|_2\le 2^{-i}\sqrt{4\delta_z^2/n+R_k^2}$. 

Now, since $Z(u,\btheta)$ is continuous a.s., we have for any $M\in \NN^+$
$$
	Z(\bu) - Z(\pi_M(\bu)) = \sum_{i=M}^\infty [Z(\pi_{i+1}(\bu)) - Z(\pi_i(\bu))],
$$
and thus 
\begin{equation}\label{eq-supZ}
\sup_{\bu\in U_k} Z(\bu) \leq  \sup_{\bu\in U_k} Z(\pi_M(\bu))  + \sum_{i=M}^\infty \sup_{\bu\in U_k}[Z(\pi_{i+1}(\bu)) - Z(\pi_i(\bu))]
\end{equation}
almost surely. Our goal is to choose a suitable $M$ such that both terms on the right hand side of (\ref{eq-supZ}) can be controlled in a reasonable manner.

For this reason, Let {$M = \lceil\frac{3}{\log 2}\log \frac{1}{b_k}\rceil + 10$}. We first upper bound $\sup_{\bu\in U_k}  Z(\pi_M(\bu))$. Note that 
$$
Z(\bu) = \frac{1}{nb_k}\sum_{t\in I_k} A_t(\bu),
$$
where $A_t(\bu) = K(\frac{w_t(\btheta)-u}{b_k})Y_t - \EE K(\frac{w_t(\btheta)-u}{b_k})Y_t$. We have $\EE A_t(\bu) = 0$ and $|A_t(\bu)|\leq \bar K$ almost surely. Moreover, 
\begin{align*}
\textrm{Var} (A_t(\bu))&\leq \EE \left[K(\frac{w_t(\btheta) - u}{b_k})y_t\right]^2\leq \EE \left[K(\frac{w_t(\btheta) - u}{b_k})\right]^2\\
	&\leq \int K(\frac{w_t(\btheta)-u}{b_k})^2f_{\btheta}(w_t(\btheta))\ud w_t(\btheta) = b_k\int K(s)^2f_{\theta}(u+b_ks)\ud s \leq C_4 b_k,
\end{align*}
where $C_4 = \max\{\bar f\cdot \int K(s)^2\ud s, \bar f\cdot \int K(s)'^2\ud s\}$. Thus according to Bernstein's Inequality, for any $\epsilon>0$, 
\begin{align*}
\PP(|Z(\bu)|\geq \epsilon) = \PP(|\sum_{t\in I_k} A_t(\bu)|\geq nb_k\epsilon)\leq 2e^{-\frac{n^2b_k^2\epsilon^2}{2C_4nb_k+\frac{2}{3}\bar Knb_k\epsilon}}\leq 2e^{-C_5\frac{nb_k\epsilon^2}{1+\epsilon}},
\end{align*}
where $C_5 = 1/\max\{2C_4, \frac{2}{3}\bar K, 1\}$.  A union bound then gives 

\begin{align*}
\PP(\sup_{\bu\in U_k} |Z(\pi_M(\bu))|\geq \epsilon)&\leq |S^{(M)}|\cdot \PP(|Z(\bu)|\geq \epsilon)\\
&\leq 2^M(\sqrt{n}+1)\cdot(2^{M+1}+1)^{d}\cdot 2e^{-C_5\frac{nb_k\epsilon^2}{1+\epsilon}}\\
&\leq \exp\left(4dM\log 2 + \log n - \frac{C_5}{2}nb_k\min\{\epsilon, \epsilon^2\}\right).
\end{align*}
When $\delta \geq 4e^{-nb_k/3}$ and {$nb_k\geq \max\{C_{b}d(\log \frac{1}{b_k} + 1), 3\log n\}$} for some absolute constant $C_b>0$, by choosing\\ $\epsilon = \epsilon(k) =  \frac{2}{C_5}\frac{1}{\sqrt{nb_k}}\sqrt{4dM\log 2 + \log n + \log\frac{4}{\delta}}$, we can verify that the last term above is upper bounded by $\frac{\delta}{4}$, and thus we have 
	\begin{equation}\label{eq-supZ1}
	\PP\left(\sup_{\bu\in \bU_k} |Z(\pi_M(\bu))|\geq \epsilon(k)\right)\leq \frac{\delta}{4}.
	\end{equation}
	
	Now we proceed to bound the latter term on the right hand side of (\ref{eq-supZ}). For any $\bu_1:=(u,\btheta_1), \bu_2:=(s,\btheta_2)\in I\times \Theta_k$, we have 
	$$
	Z(\bu_1)-Z(\bu_2)=Z(u,\btheta_1) - Z(s,\btheta_2) = \frac{1}{nb_k}\sum_{t\in I_k}B_t(u, \btheta_1,s,\btheta_2), 
	$$
	where 
	$$
	B_t(u, \btheta_1,s,\btheta_2) = y_t\left(K(\frac{w_t(\btheta_1)-u}{b_k}) - K(\frac{w_t(\btheta_2)-s}{b_k}) \right) - \EE y_t\left(K(\frac{w_t(\btheta_1)-u}{b_k}) - K(\frac{w_t(\btheta_2)-s}{b_k}) \right).
	$$
	Then $\EE B_j(u, \btheta_1,s,\btheta_2) = 0$, and 
	\begin{align*}
	|Z(\bu_1)-Z(\bu_2)|=|B_t(u, \btheta_1,s,\btheta_2)|&\leq 2\left|y_t(K(\frac{w_t(\btheta_1)-u}{b_k}) - K(\frac{w_t(\btheta_2)-s}{b_k}))\right|\\&\leq \frac{2l_K\sqrt{(\max_{\xb\in\cX}\|\xb\|_2^2+1)}}{b_k}\cdot\|\bu_1-\bu_2\|_2.
	\end{align*}
	Using Hoeffding's Inequality, for any $\epsilon>0$,
	$$
	\PP(|\sum_{t\in I_k} B_t(\bu_1, \bu_2)|\geq \epsilon)\leq 2e^{-\frac{2\epsilon^2}{{4l_K^2(R_{\cX}^2+1)}/{b_k^2}\cdot n\|\bu_1-\bu_2\|_2^2}} = 2e^{-\frac{b_k^2\epsilon^2}{2l_K^2n(R_{\cX}^2+1)\|\bu_1-\bu_2\|_2^2}}
	$$
	Therefore, 
	$$
	\PP(|Z(\bu_1) - Z(\bu_2)|\geq \epsilon) = \PP(|\sum_{t\in I_k} B_t(\bu_1, \bu_2)|\geq nb_k\epsilon)\leq 2e^{-\frac{nb_k^4\epsilon^2}{2l_K^2(R_{\cX}^2+1)\|\bu_1-\bu_2\|_2^2}}.
	$$
	Recall that $\forall \bu$, $\|\pi_i(\bu) - \pi_{i+1}(\bu)\|_2\leq 2^{-i}\sqrt{4\delta_z^2/n+R_k^2}$. We use union bound to obtain 
	\begin{align*}
	&\PP(\sup_{\bu\in \bU_k} |Z(\pi_{i+1}(\bu)) - Z(\pi_{i}(\bu))|\geq \epsilon) \\&\quad\leq 2^i(\sqrt{n}+1)\cdot(2^{i+1}+1)^{d}\cdot 2e^{-\frac{2^{2i-2}n^2b_k^4\epsilon^2}{2l_K^2(R_{\cX}^2+1)(4\delta_z^2 + nR_k^2)}}.
	\end{align*}
	Let $\epsilon = \frac{l_K \sqrt{(R_{\cX}^2+1)(4\delta_z^2 + nR_k^2)}\epsilon_i}{2^{i-1}nb_k^2}$. The above inequality reduces to
	\begin{align}
	&\PP\Big(\sup_{\bu\in \bU_k} |Z(\pi_{i+1}(\bu)) - Z(\pi_{i}(\bu))|\geq \frac{l_K \sqrt{(R_{\cX}^2+1)(4\delta_z^2 + nR_k^2)}\epsilon_i}{2^{i-1}nb_k^2}\Big) \nonumber\\&\quad\leq 2^i(\sqrt{n}+1)\cdot(2^{i+1}+1)^{d}\cdot 2e^{-\frac{\epsilon_i^2}{2}}.\label{eq-supZ2}
	\end{align}
	
	Now we choose {$\epsilon_i = \sqrt{2\log \frac{8}{\delta} + \log n + (2i+4)(d+2)\log 2}$} and define $W^*:=\sqrt{(R_{\cX}^2+1)(4\delta_z^2 + nR_k^2)}$. Notice that 
	\begin{align*}
	\sum_{i=M}^\infty\frac{l_KW^*}{n{b_k^2}}\frac{\epsilon_i}{2^{i-1}}&\leq \frac{l_KW^*}{nb_k^2}\sum_{i=M}^\infty\frac{\sqrt{2id\log 2} + \sqrt{(4d+8)\log 2 + \log n+ 2\log \frac{8}{\delta}}}{2^{i-1}}
\\
&	\leq  \frac{l_KW^*}{nb_k^2} \left[\sqrt{2d\log2}\sum_{i=M}^\infty \frac{i}{2^{i-1}} + \frac{1}{2^{M-2}}\sqrt{(4d+8)\log 2 + \log n+ 2\log \frac{8}{\delta}}\right]
\\
&	\leq \frac{l_KW^*}{nb_k^2} \left[\sqrt{2d\log2}\frac{M+1}{2^{M-2}} + \frac{1}{2^{M-2}}\sqrt{(4d+8)\log 2 + \log n+ 2\log \frac{8}{\delta}}\right]\\
&\leq  \frac{l_K W^*}{\sqrt{nb_k}}\cdot \frac{1}{n^{1/2}b_k^{3/2}}\frac{M+2}{2^{M-2}}\left[ \sqrt{2\log\frac{8}{\delta} + \log n} + 4\sqrt{d\log 2}\right]\\
& \leq  \frac{l_K W^*}{\sqrt{nb_k}} \left[ \sqrt{\frac{2}{n}\log\frac{8}{\delta}} + 1 + \frac{6\sqrt{\log 2}}{\sqrt{c_0}}\right]
\end{align*}
	Here we use the fact that {when $B_{x, K} \geq (2c_0)^{4}$, combining the assumptions in the lemma, we have $n\geq c_0 d$}. Combining this fact and a union bound on (\ref{eq-supZ2}), we get
	\begin{align}
	&\PP \left( 
	\sup_{\bu\in \bU_k} |Z(\bu) - Z(\pi_M(\bu))| \geq 
	\frac{l_K W^*}{\sqrt{nb_k}} \left[ \sqrt{\frac{2}{n}\log\frac{8}{\delta}} + 1 + \frac{6\sqrt{\log 2}}{\sqrt{c_0}} \right]\right)\nonumber\\
	\leq &\PP\left(\sup_{\bu\in \bU_k} |Z(\bu) - Z(\pi_M(\bu))|\geq \sum_{i=M}^\infty\frac{l_KW^*}{n{b_k^2}}\frac{\epsilon_i}{2^{i-1}} \right)\nonumber\\
	\leq & \PP\left(\sum_{i=M}^\infty \sup_{\bu\in \bU_k} |Z(\pi_{i+1}(\bu)) - Z(\pi_i(\bu))|\geq \sum_{i=M}^\infty\frac{l_KW^*}{n{b_k^2}}\frac{\epsilon_i}{2^{i-1}} \right)\nonumber\\
	\leq &\sum_{i=M}^\infty \PP\left( \sup_{\bu\in \bU_k} |Z(\pi_{i+1}(\bu)) - Z(\pi_i(\bu))|\geq \frac{l_KW^*}{n{b_k^2}}\frac{\epsilon_i}{2^{i-1}} \right)\nonumber\\
	\leq & \sum_{i=M}^\infty 2^i(\sqrt{n}+1)\cdot(2^{i+1}+1)^{d}\cdot 2e^{-\frac{\epsilon_i^2}{2}} \leq \sum_{i=M}^\infty \frac{\delta}{4}\cdot \frac{1}{2^{i+1}}\leq \frac{\delta}{4\cdot 2^M}\leq \frac{\delta}{4}.\label{eq-supZ3}
	\end{align} 
	
	Finally, combining (\ref{eq-supZ}), (\ref{eq-supZ1}) and (\ref{eq-supZ3}), we obtain that 
	\begin{align*}
	\frac{\delta}{2}&\geq \PP\bigg(\sup_{\bu\in \bU_k} |Z(\pi_M(\bu))|\geq \epsilon(k)\bigg) + \PP\bigg(\sup_{\bu\in \bU_k} |Z(\bu) - Z(\pi_M(\bu))|\geq \frac{l_K W^*}{\sqrt{nb_k}} \bigg[ \sqrt{\frac{2}{n}\log\frac{8}{\delta}} + 1 + \frac{6\sqrt{\log 2}}{\sqrt{c_0}}\bigg] \bigg)\\
	&\geq \PP\bigg(\sup_{\bu\in \bU_k} Z(\bu)\geq \epsilon(k) + \frac{l_K W^*}{\sqrt{nb_k}} \bigg[ \sqrt{\frac{2}{n}\log\frac{8}{\delta}} + 1 + \frac{6\sqrt{\log 2}}{\sqrt{c_0}}\bigg]\bigg)\\
	&\geq \PP\bigg(\sup_{\bu\in \bU_k} Z(\bu)\geq \frac{4\sqrt{11}/C_5}{\sqrt{nb_k}}\sqrt{d\left(1+\log\frac{1}{b_k}\right) + \log n + \log\frac{4}{\delta}} + \\
	&\quad\quad 16\sqrt{2}\bigg(1+\frac{6\sqrt{\log 2}}{c_0}\bigg)\frac{l_K\sqrt{1+R_{\cX}^2}}{\sqrt{nb_k}}\max\Big\{\delta_z, \frac{\max\{1, \psi_x\}(B+{R_{\cX}R_{\Theta}})}{c_{\min}}\Big\}\Big(\sqrt{d\log n} + \sqrt{\frac{d\log n}{n}\log\frac{8}{\delta}}\Big)\bigg)\\
	&\geq \PP\bigg(\sup_{\bu\in \bU_k} Z(\bu)\geq C_xl_K\sqrt{\frac{{\log n}}{{nb_k}}}\left(\sqrt{d} + \sqrt{\log {1}/{\delta}}\right)\bigg).
	\end{align*}
 Here we let {$C_x = 8\sqrt{22}/C_5 + \frac{60(6\sqrt{\log 2}+\sqrt{c_0})}{c_0}\sqrt{1+R_{\cX}^2}\max\{\delta_z, \frac{\max\{1,\psi_x\}(B+{R_{\cX}R_{\Theta}})}{c_{\min}}\}$}.
	
	For the same reason, we have that 
	$$
	\PP\left(\sup_{\bu\in \bU_k} (-Z(\bu))\geq C_xl_K\sqrt{\frac{{\log n}}{{nb_k}}}\left(\sqrt{d} + \sqrt{\log {1}/{\delta}}\right) \right)\leq \frac{\delta}{2}.
	$$
	Combining the above two inequalities, we finish the proof.

\subsection{Proof of Lemma \ref{lem-bias_2}}\label{sec-pf-a3}

	We only prove \eqref{eq-biash_2}, since \eqref{eq-biasf_2} can be proved in a similar way. Recall $h_k^{(1)}(u,\btheta)=\frac{-1}{nb_k^2}\sum_{t\in I_k}K'(\frac{w_t(\btheta)-u}{b_k})y_t$, we have
	\begin{align*}
		\EE h_k^{(1)}(\btheta,u)=\frac{-1}{b_k^2}\EE K'(\frac{w_t(\btheta)-u}{b_k})y_u=\frac{-1}{b_k^2}\EE K'(\frac{w_t(\btheta)-u}{b_k})r(w_t(\btheta)).
		\end{align*}
	Then 
	\begin{align}
		\EE h_k^{(1)}(u,\btheta)-h'_{\btheta}(u)&=\int \frac{-1}{b_k^2}K'(\frac{w_t(\btheta)-u}{b_k})h_{\btheta}(w_t(\btheta))\ud w_t(\btheta)-h_{\btheta}'(u)\nonumber\\
		&=\int K(s)h_{\btheta}'(u+b_ks)\ud s-h_{\btheta}'(u)\label{eq-biash2},
		\end{align}
	where \eqref{eq-biash2} follows from integration by parts. By Taylor's expansion, we have
	\begin{align*}
		h_{\btheta}'(u+b_ks)=h'_{\btheta}(u)+\sum_{i=2}^{m-2}\frac{h_{\btheta}^{(i)}(u)}{(i-1)!}(b_ks)^{i-1}+\frac{h_{\btheta}^{(m-1)}(\xi(s,u))}{(m-1)!}(b_ks)^{m-2}.
		\end{align*}
	Similar to our proof procedure of Lemma \ref{lem-bias}, under Assumption \ref{ass-kernel}, we get
	\begin{align*}
		\EE h_k^{(1)}(u,\btheta)-h_{\btheta}'(u)=\int K(s)\frac{h^{(m-1)}_{\theta}(\xi(s,u))-h^{(m-1)}_{\btheta}(u)}{(m-2)!}(b_ks)^{m-2}\ud s.
		\end{align*}
	Thus
	\begin{align}
		|\EE h_k^{(1)}(u,\btheta)-h_{\btheta}'(u)|&\le \int |K(s)\frac{[h^{(m-1)}_{\btheta}(\xi(s,u))-h^{(m-1)}_{\btheta}(u)]}{(m-2)!}(b_ks)^{m-2}|\ud s\nonumber\\
		&\le |K(s)|\frac{l_f|b_ks|}{(m-2)!}|b_ks|^{m-2}\ud s\nonumber\\
		&\le C_{x, K}^{(5)}b_k^{m-1},\label{biash2}
		\end{align}
		in which $C_{x, K}^{(5)}=\frac{l_f}{(m-2)!}\int |K(s)s^{m-1}|\ud s$. Because \eqref{biash2} holds for any $t\in I$ and $\btheta\in \Theta_k$, we have
		\begin{align*}
		\sup_{u\in I,\btheta\in \Theta_k}|\EE h_{k}^{(1)}(u,\btheta)-h_{\btheta}'(u)|\le C_{x, K}^{(5)}b_k^{m-1},
		\end{align*}
	which claims inequality \ref{eq-biash_2} of Lemma \ref{lem-bias_2}. On the other hand, \eqref{eq-biasf_2} follows directly from our proof procedure above, so we omit the details.

\subsection{Proof of Lemma \ref{lem-dev_2}} \label{sec-pf-a4}
For any $u\in I,\btheta \in \Theta_k$, write 
\begin{align*}
Z^{(1)}(u,\btheta)=h_k^{(1)}(u,\btheta)-\EE h_k^{(1)}(u,\btheta)=\frac{-1}{b_k}\cdot\frac{1}{nb_k}\sum_{t\in I_k}\Big[K'(\frac{w_t(\btheta)-u}{b_k})y_t-\EE K'(\frac{w_t(\btheta)-u}{b_k})y_t\Big]
\end{align*}

Under Assumption \ref{ass-F} and Assumption \ref{ass-kernel}, by following a similar proof procedure with Lemma \ref{lem-dev}, for $\delta\in[4e^{-nb_k/3}, \frac{1}{2})$, with probability at least $1-\delta$,
\begin{align*}
\sup_{u\in I,\btheta\in \Theta_k}\Big|\frac{1}{nb_k}\sum_{t\in I_k}[K'(\frac{w_t(\btheta)-u}{b_k})y_t-\EE K'(\frac{w_t(\btheta)-u}{b_k})y_t]\Big|\le C_{x, K}^{(2)}\sqrt{\frac{{\log n}}{{nb_k}}}\left(\sqrt{d} + \sqrt{\log {1}/{\delta}}\right),
	\end{align*}
	
where $C_{x, K}^{(2)} = l_K\bigg(8\sqrt{22}\max\{2\bar f \int K^2\db s, 2\bar f \int K'^2\db s, \frac{2}{3}\bar K, 1\} +$
		
		$ \frac{60(6\sqrt{\log 2}+\sqrt{c_0})}{c_0}\sqrt{1+R_{\cX}^2}\max\{\delta_z, \frac{\max\{1,\psi_x\}(B+{R_{\cX}R_{\Theta}})}{c_{\min}}\}\bigg)$. Thus, with probability at least $1-\delta$, 
\begin{align*}
	\sup_{u\in I,\btheta\in \Theta_k}|h_k^{(1)}(u,\btheta)-\EE h_k^{(1)}(u,\btheta)|\le  C_{x, K}^{(2)}\sqrt{\frac{{\log n}}{{nb_k^3}}}\left(\sqrt{d} + \sqrt{\log {1}/{\delta}}\right),
	\end{align*}
which claims the inequality \eqref{eq-devh_2} in Lemma \ref{lem-dev_2}. Moreover, \eqref{eq-devf_2} also follows directly from our procedure given above. Thus, we claim our our conclusion of Lemma \ref{lem-dev_2}.

{
\subsection{Proof of Lemma \ref{lemma_inverse}}\label{proof_lemma3.9}

First, we argue that for any $\tilde \xb_t$, 
\begin{equation}\label{eq-A5pf0}
\btheta_0^\top \tilde \xb_t\in [\delta_z+\delta_v, B-\delta_z-\delta_v].
\end{equation}
In fact, we have $v_t = \btheta_0^\top \tilde \xb_t + z_t$, where $z_t\in[-\delta_z, \delta_z]$ and that $\btheta_0^\top \tilde \xb_t$ is independent from $z_t$. Therefore, in order to satisfy the condition $v_t\in [\delta_v, B-\delta_v]$, it ought to be true that $\btheta_0^\top \tilde \xb_t\in [\delta_z+\delta_v, B-\delta_z-\delta_v]$.

On the other hand, 
\begin{align}
\sup_{\tilde \xb_t\in \cX, \btheta\in\Theta_0} |\btheta^\top\tilde \xb_t - \btheta_0^\top\tilde \xb_t|&\leq \sup_{\btheta\in\Theta_0} \|\btheta- \btheta_0\|\cdot \sup_{\tilde \xb_t\in \cX} \|\xb_t\|\nonumber\\
&\leq C_{\btheta}T^{-\frac{2m+1}{4(4m-1)}}d^{\frac{m-1}{4m-1}}\sqrt{\log T+2\log d}\cdot R_\cX\nonumber \\
&\leq \delta_v. \label{eq-A5pf1}
\end{align}
The last inequality is due to the condition on $T$. The lemma is proved by combining \eqref{eq-A5pf0} and \eqref{eq-A5pf1}.
}

\subsection{Proof of Lemma \ref{para_mix}} \label{proof_mixpara}
The proof of Lemma \ref{para_mix} is similar with our proof of Lemma \ref{thmsignal}, the major difference between them is that here we assume our covaraites $\tilde{\xb}_t,t\ge 0$ follow $\beta$-mixing condition instead of of i.i.d. assumption. After following similar proof procedures of \eqref{lossepisode}-\eqref{ineq}, we obtain the same inequality with \eqref{ineq} and
we also divide the following proofs into two steps.

\textbf{Step I:}
{In this step, we prove under $\beta$-mixing conditions given in Assumption \ref{assmix1}, with high-probability, there exists a constant $c>0$ such that $\lambda_{\min}(\frac{1}{|I_k|}\sum_{t\in I_k}\tilde{\xb}_t\tilde{\xb}_t^\top)\ge c$.} In order to prove this, we first use the following matrix Bernstein inequality under $\beta$-mixing conditions to prove the concentration between $\Sigma_k:=\frac{1}{|I_k|}\sum_{t\in I_k}\tilde{\xb}_t\tilde{\xb}_t^\top$ and $\Sigma:=\EE[\tilde{\xb}_t\tilde{\xb}_t^\top]$. Similar to \S\ref{proofparathm}, here for notational convenience, we also denote $n=|I_k|$ for any $k\ge 1$ respectively.
\begin{lemma}[Matrix Bernstein Inequality under Mixing]\label{matbern} We assume $\tilde{\xb}_t,t\ge 0$ satisfy Assumption \ref{assmix1}, and we also assume there exists a positive constant $M_x$ such that $\|\tilde{\xb}_t\|_{2}\le M_x$. Then for any x and integer $n\ge 2$ we have
	\begin{align}\label{matbernstein}
		\PP\Big(\|\sum_{t\in I_k}\tilde{\xb}_t\tilde{\xb}_t^\top-n\Sigma\| \ge nx\Big)\le 2(d+1)\exp\bigg( -\frac{C_{u}n^2x^2}{v^2n+M_x^4+nxM_x^2\log n}\bigg)
	\end{align}
	where $C$ is a universal constant and
	\begin{align*}
		v^2=\sup\limits_{K\in\{1,\dots,n\}}\frac{1}{\textrm{Card}(K)}\lambda_{\max}\Big\{ \EE\big[\sum_{i\in K}(\tilde{\xb}_i\tilde{\xb}_i^\top-\Sigma)\big]^2   \Big\}
	\end{align*}
	and $v^2$ is at the order of $M_x^4$.
\end{lemma}
\begin{proof}
	\eqref{matbernstein} is a direct consequence of Theorem 1 in \cite{Bannamat}, so here we just need to prove the order of $v^2$.
	\begin{align*}
		\lambda_{\max}\Big\{\EE \big[\sum_{i\in K}(\tilde{\xb}_i\tilde{\xb}_i^\top-\Sigma)\big]^2 \Big\}&=\lambda_{\max}\Big\{\sum_{{i,j}\in K} \textrm{Cov}\big(\tilde{\xb}_i\tilde{\xb}_i^\top,\tilde{\xb}_j\tilde{\xb}_j^\top\big)\Big\}\\&=\lambda_{\max}\Big\{\sum_{i\in K}\textrm{Var}(\tilde{\xb}_i\tilde{\xb}_i^\top)+2\sum_{j>i,\,i,j\in K}\textrm{Cov}(\tilde{\xb}_i\tilde{\xb}_i^\top,\tilde{\xb}_j\tilde{\xb}_j^\top)   \Big\}
	\end{align*}
	Then we get 
	\begin{align*}
		v^2\le \max_{i\in K}\lambda_{\max} \Big \{\textrm{Var}( \tilde{\xb}_i \tilde{\xb}_i^\top) +2\sum_{j>i,\,i, j\in K }   \textrm{Cov}(\tilde{\xb}_i\tilde{\xb}_i^\top,\tilde{\xb}_j\tilde{\xb}_j^\top)       \Big\}
	\end{align*}
	We know $\|\tilde{\xb}_i\|_2\le M_x$, so we have 
	\begin{align*}
		\lambda_{\max}\{\textrm{Var}(\tilde{\xb}_i\tilde{\xb}_i^\top ) \}\le\| \EE[\tilde{\xb}_i\tilde{\xb}_i^\top\tilde{\xb}_i\tilde{\xb}_i^\top]\|\le M_x^4 
	\end{align*}
	In addition, we obtain
	\begin{align}\label{spectralmix}
		\|\textrm{Cov}(\tilde{\xb}_i\tilde{\xb}_i^\top,\tilde{\xb}_j\tilde{\xb}_j^\top) \|=\|\EE[\tilde{\xb}_i\tilde{\xb}_i^\top\tilde{\xb}_j\tilde{\xb}_j^\top]-\EE[\tilde{\xb}_i\tilde{\xb}_i^\top]\EE[\tilde{\xb}_j\tilde{\xb}_j^\top]\|
	\end{align}
	By Lemma 1.1 (Berbee's Lemma) given in \cite{Bosq1996}, we are able to construct a $\tilde{\xb}_j^{*}$ such that the distribution of $\tilde{\xb}_j^{*}$ is the same with $\tilde{\xb}_j$ but is independent with $\tilde{\xb}_i$. At the same time, we also have $\PP(\tilde{\xb}_j^{*}\neq \tilde{\xb}_j)=\beta_{j-i}$ according to Berbee's Lemma.
	We then proceed to bound \eqref{spectralmix}.
	\begin{align*}
		\eqref{spectralmix}&=\|\EE[\tilde{\xb}_i\tilde{\xb}_i^\top\tilde{\xb}_j\tilde{\xb}_j^\top]-\EE[\tilde{\xb}_i\tilde{\xb}_i^\top]\EE[\tilde{\xb}_j^{*}\tilde{\xb}_j^{*\top}]\|\\&=\| \EE[\tilde{\xb}_i\tilde{\xb}_i^\top(\tilde{\xb}_j\tilde{\xb}_j^\top-\tilde{\xb}_j^{*}\tilde{\xb}_j^{*\top}) ]\|
		\\&\le\| \EE[\tilde{\xb}_i\tilde{\xb}_i^\top(\tilde{\xb}_j\tilde{\xb}_j^\top-\tilde{\xb}_j^{*}\tilde{\xb}_j^{*\top})\given \tilde{\xb}_j\neq\tilde{\xb}_j^{*} ]\|\beta_{j-i}\le M_x^4\beta_{j-i}
	\end{align*}
	Then we obtain that there exists a constant $C_v\ge 1+\sum_{j>i}\beta_{j-i}$ s.t.
	\begin{align*}
		v^2\le C_vM_x^4,
	\end{align*}
	holds, since the term $1+\sum_{j>i}\beta_{j-i}$ is finite by our Assumption \ref{assmix1} on $\beta_{j},\,\,j\ge 0$. Then we conclude our proof of Lemma \ref{matbern}
\end{proof}
By using conclusions from this Lemma \ref{matbern},  according to Assumption \ref{ass:bound} we have $\lambda_{\min}(\Sigma)=c_{\min}$ and $\|\tilde{\xb}_t\|_2\le M_x:=\sqrt{R_{\cX}^2+1}$, so when $n \ge \max\{(12C_v(R_{\cX}^2+1)^2\log n+6(R_{\cX}^2+1)\log^2 n)/(C_u\min\{c_{\min}^2/4,1\}),d+1\}$, 
\begin{align}\label{mineign}
	\lambda_{\min}(\Sigma_k)\ge c_{\min}/2. 
\end{align}
holds with probability $1-2/n^2$.

\textbf{Step II}: The next step is to prove the upper bound of $\|\nabla_{\btheta}L_k(\btheta_0)\|_{\infty}.$ By definition we know 
\begin{align*}
	\nabla_{\btheta}L_k(\btheta_0)=\frac{1}{n}\sum_{t\in I_k} 2(\btheta_0^\top\tilde{\xb}_t-By_t)\tilde{\xb}_t.
\end{align*}
Since the expression of $\nabla_{\btheta}L_k(\btheta_0)$ involves both $\tilde{\xb}_t$ and $y_t,\,\,t\in[n]$, next we show the sequence $(\tilde{\xb}_t,y_t),t\ge 0$ satisfy $\alpha$-mixing condition with $\alpha_k\le \exp(-ck)$ under Assumption \ref{assmix1}.
\begin{lemma}[strong $\alpha$-mixing of both $\tilde{\xb}$ and $y$]
	Here we denote $\cA=\sigma((\tilde{\xb}_t,y_t)_{t\le l})$ and $\cB=\sigma((\tilde{\xb}_t,y_t)_{t\le l+k})$. In addition, we also denote $\cA_x=\sigma({\tilde{\xb}_t,\,}_{t\le l})$ and $\cB_x=\sigma({\tilde{\xb}_t,\,}_{t\ge l+k})$. Then under Assumption \ref{assmix1}, we have for any $l,k\ge 0,$
	\begin{align*}
		\sup_{l\ge 0}\sup_{A\in \cA,B\in \cB}|\PP(A,B)-\PP(A)\cdot\PP(B)|\le \alpha_k
	\end{align*}
	where the definition of $\alpha_k$ is given in Definition \ref{defalpha}.
\end{lemma} 

\begin{proof}	
	\begin{align*}
		\sup_{l\ge 0}\sup_{A\in \cA,B\in \cB}|\PP(A,B)-\PP(A)\cdot\PP(B)|&=	\sup_{l\ge 0}\sup_{A\in \cA,B\in \cB}\big|\EE[\II_{A,B}]-\EE[\II_{A}]\EE[\II_{B}]\big|\\
		&=\sup_{l\ge 0}\sup_{A\in \cA,B\in \cB}\big|\EE[\EE[\II_{A,B}\given \cA_x,\cB_x]]-\EE[\EE[\II_{A}\given \cA_x]]\EE[  \EE[\II_{B}\given \cB_x]]\big|
	\end{align*}
After conditioning on $\tilde{\xb}_i,\tilde{\xb}_j$, we observe that $y_i,y_j$ are independent with each other, then we get $\EE[\II_{A,B}\given \cA_x,\cB_x]=\EE[\II_{A}\given \cA_x]\cdot \EE[\II_{B}\given \cB_x].$ 
	Thus, we have for any $k\ge 0,$
	\begin{align*}
		\sup_{l\ge 0}\sup_{A\in \cA,B\in \cB}\big|\EE[\II_{A,B}]-\EE[\II_{A}]\EE[\II_{B}]\big|&=\sup_{l\ge 0}\sup_{A\in \cA,B\in \cB}\big|\EE[\EE[\II_{A}\given \cA_x]\cdot \EE[\II_{B}\given \cB_x]]-\EE[\EE[\II_{A}\given \cA_x]]\EE[  \EE[\II_{B}\given \cB_x]]\big|\\&\le \alpha_k \|\II_{A}\|_{\infty}\cdot\|\II_{B}\|_{\infty}=\alpha_k
	\end{align*}
	The last inequality follows directly from Corollary 1.1 in \cite{Bosq1996}, since $\EE[\II_{A}\given \cA_x]$ lies in $\cA_x$ and $\EE[\II_{B}\given \cB_x]$ lies in $\cB_x$.
\end{proof}

By using the same proof given in \S\ref{proofparathm}, we have $\EE[\nabla_{\btheta}L_k(\btheta_0)]=0$. In addition, we obtain an upper bound of every entry of $\nabla_{\btheta}L_k(\btheta_0)$ in a way that there exists a upper bound {$W_x=2R_{\cX}(R_{\cX}R_{\Theta}+B)$ of $|2(\btheta_0^\top\tilde{\xb}_t-By_t)\tilde{\xb}_{t,i}|$, for every $i \in[d]$}. Then using the following vector Bernstein inequality under $\alpha$-mixing conditions, we obtain an upper bound for $\|\nabla_{\btheta}L_k(\btheta_0)\|_{\infty}$.
\begin{lemma}\label{lemmixbern}(Vector Bernstein under $\alpha$-Mixing Conditions, Theorem 1 in \cite{merlevede})
	Let $X_j,j\ge 0$ be a sequence of centered real-valued random variables. Suppose there exists a positive $W_x$ such that $\sup_{i}\|X_i\|_{\infty}\le W_x$, then when $n\ge 4$ and $x\ge 0$, we obtain
	\begin{align*}
		\PP\Big(\Big|\frac{1}{n}\sum_{i=1}^{n}X_i\Big|\ge x\Big)\le \exp\Big(  -\frac{C_wn^2x^2}{nW_x^2+W_x nx\log n\log\log n} \Big)
	\end{align*}
	where $C_w$ is a universal constant.
\end{lemma}
By leveraging conclusions from Lemma \ref{lemmixbern}, we have
\begin{align*}
	\PP(\|\nabla_{\btheta}L_k(\btheta_0)\|_{\infty}\ge x)\le 2(d+1)\exp\Big(-\frac{C_wn^2x^2}{nW_x^2+W_xnx\log n\log\log n}\Big).
\end{align*}
Thus, when $n\ge \max\{(6W^2_x\log n+6W_x\log^2 n\log\log n)/C_w, d+1\}$ we obtain, with probability $1-2/n^2$, we have
\begin{align}\label{gradupmix}
	\|\nabla_{\btheta}L_k(\btheta_0)\|_{\infty}\le \sqrt{(6W_x^2\log n+6W_x\log^2 n\log\log n)/(C_w n)}.
\end{align}
Then combining our results given in \eqref{ineq}, \eqref{mineign} and \eqref{gradupmix}, with probability $1-4/|I_k|^2$ we obtain
\begin{align*}
	\|\hat\btheta_k-\btheta_0\|_2\le \frac{2}{c_{\min}}\sqrt{\frac{(d+1)(6W_x^2\log |I_k|+6W_x\log^2 |I_k|\log\log |I_k|)}{C_w|I_k|}}
\end{align*} 
for any $k\ge 1.$
\subsection{Proof of Lemma \ref{conckernel_main_mix}}\label{pf-b3}
\begin{proof}
Similar with our proof given in \S\ref{pf:conckernel_main}, we suppose $\{w_t(\btheta):=p_t-\tilde{\xb} _t^\top \btheta , y_t\}_{t\in[n]}$ are observations from the stationary distribution $P_{w(\btheta), y}$. We assume that the marginal distribution $P_{w(\btheta)}$ has density $f_{\btheta}(u)$ and let $r_{\btheta}(u)=\EE [y_t\given w_t(\btheta)=u]$ be the regression function to be estimated by estimator
$$
\hat r_k(u,\theta) = \frac{h_k(u,\theta)}{f_k(u,\theta)},
$$
where
$$
h_k(u,\theta) = \frac{1}{nb_k}\sum_{t\in I_k}^n K(\frac{w_t(\btheta)-u}{b_k})Y_t,\quad f_k(u,\btheta) = \frac{1}{nb_k}\sum_{t\in I_k}^n K(\frac{w_t(\btheta)-u}{b_k}).
$$
Here, $b_k>0$ is the bandwidth (to be chosen) in episode $k$, $|I_k|$ is denoted as $n$ for simplicity and $K(\cdot)$ is some kernel function. For the true signal $\btheta_0$, we denote the true regression function as $r_{\btheta_0}(u)=\EE[y_t\given w_t(\btheta_0)=u ]$. The following proof procedures are similar with that given in \S\ref{pf:conckernel_main}, where their major differences are related to  control the biases of $|\EE[h_k(u,\btheta)]-h_{\btheta}(u)|$ and $|\EE[f_k(u,\btheta)]-f_{\btheta}(u)|$ given in Lemma \ref{mix_bias} and the variances of $h_k(u,\btheta)$ and $f_k(u,\btheta)$ given in Lemma \ref{mix_var} under strong-mixing settings respectively.
\begin{lemma}\label{mix_bias}
	Under Assumptions \ref{ass-pdfx}-\ref{ass-kernel} and \ref{assmix1}, with any choice of $b_k\le 1$, we obtain
	\begin{align*}
		\sup_{u\in I,\btheta\in \Theta_k}|\EE h_k(u,\btheta)-h_{\btheta}(u)|&\le C_{mx,K}^{(1)}b_k^{m}\\
		\sup_{u\in I,\btheta\in \Theta_k}|\EE f_k(u,\btheta)-f_{\btheta}(u)|&\le C_{mx,K}^{(1)}b_k^{m}
	\end{align*}
	where $C_{mx,K}=l_f\frac{\int |s^mK(s)\ud s}{(m-1)!}$.
\end{lemma}
\begin{proof}
	The proof of Lemma \ref{mix_bias} is the same with the proof of Lemma \ref{lem-bias}. So we omit the details.
\end{proof}
\begin{lemma}\label{mix_var}
Under Assumption \ref{ass-pdfx}-\ref{ass-kernel} and \ref{assmix1}, there exists a constant $C_{17}'$ only depending on constants given in assumptions, such that for $I=[-\delta_z,\delta_z]$, if $b_k\in[1/n,1]$, $nb_k\ge 4C_{17}'^2\log^3 n[(d+1)\log (d+1)]$ and $\delta\in[ 8\exp(-nb_k/(8C_{17}'^2\log^2n)),1/2]$, the following inequalities hold simultaneously with probability $1-\delta$:
	\begin{align}
		&\sup_{u\in I,\btheta\in \Theta_k}|h_k(u,\btheta)-\EE[h_k(u,\btheta)]|\le \frac{C_{17}'\log n}{\sqrt{nb_k}}\bigg(\sqrt{(d+1)\log (d+1)\log n}+\sqrt{2\log \frac{8}{\delta}}\bigg)\label{mix_varh}\\
		&\sup_{u\in I,\btheta\in \Theta_k}|f_k(u,\btheta)-\EE[f_k(u,\btheta)]|\le \frac{C_{17}'\log n}{\sqrt{nb_k}}\bigg(\sqrt{(d+1)\log (d+1)\log n}+\sqrt{2\log \frac{8}{\delta}}\bigg) \label{mix_varf}
	\end{align}
\end{lemma}
\begin{proof}
	We only prove (\ref{mix_varh}), since (\ref{mix_varf}) can be proved in the same way. 
	For any $u\in I$ and $\btheta\in \Theta_k$, we denote $Z(u,\btheta) := h_k(u,\btheta) - \EE h_k(u,\btheta) = \frac{1}{nb_k}\sum_{t\in I_k}[K(\frac{w_t(\btheta)-u}{b_k})y_t - \EE K(\frac{w_t(\btheta)-u}{b_k})y_t]$. Then we have that
	$$
	\sup_{u\in I,\btheta\in\Theta_k}|h_k(u,\btheta) - \EE h_k(u,\btheta)| =\sup_{u\in I,\btheta\in\Theta_k}|Z(u,\theta)| = \max\Big\{\sup_{u\in I,\btheta\in\Theta_k} Z(u,\btheta), \sup_{u\in I,\btheta\in\Theta_k}(-Z(u,\btheta))\Big\}.
	$$
	Similar with our proof procedure of Lemma \ref{lem-dev}, we  then bound $\sup_{u\in I,\btheta\in\Theta_k}|h_k(u,\btheta) - \EE h_k(u,\btheta)|$ by upper bounding both $\sup_{u\in I,\btheta\in\Theta_k} Z(u,\btheta)$ and $\sup_{u\in I,\btheta\in\Theta_k} (-Z(u,\btheta))$. 
	We next also use chaining method to achieve desired bound. We also construct a sequence of $\epsilon$-nets with decreasing scale.
	
	As a reminder, here we also denote the left and right endpoints of the interval $I$ as $L_I$ and $R_I$ respectively. For any $i\in \NN^+$, construct set $S^{(i)}_1\subseteq I$ as 
	$$
	S^{(i)}_1\triangleq \left\{ L_I+\frac{j}{2^i\sqrt{n}}(R_I-L_I): j\in\{1, 2, \cdots, (2^i-1)\lceil\sqrt{n}\rceil\}\right\}.
	$$
	For any $u\in I$, $i\in\NN^+$, let $\pi_i(u) = \arg\min_{s\in S^{(i)}_1}|s-u|$. Moreover, let $\pi_0(u) = u$. 
	Then we can easily verify that $|S^{(i)}_1|\leq 2^i(\sqrt{n}+1)$, and that $\forall u\in I$, $|\pi_i(u) - \pi_{i+1}(u)|\leq \frac{2\delta_z}{2^{i-1}\sqrt{n}}$. \\
	As for the $\epsilon$-net of $\Theta_k$, we let $S^{i}_2$ be a $R_{m}/({2^{i}\sqrt{n}})$-net with respective to $l_2$-distance of $\Theta_k$, where $R_{m}=2/c_{\min}\sqrt{6W_x/C_w}$ (constants are specified in the Lemma \ref{para_mix}). By Proposition 4.2.12 in \cite{vershynin_2018}, we have $|S_2^{(i)}|\le (2^{i+1}C(d,n)+1)^{d}$, where $C(d,n)=\sqrt{(d+1)(W_x\log n+\log^2 n\log\log n)}.$\\ Then we have for any $\ub:=(u,\btheta)\in I\times \Theta_k$ with $i\ge 1$, there exist $\pi_i(u)\in S_1^{(i)}$ and $\pi_i(\btheta)\in S_2^{(i)}$  such that $\|\pi_i(\ub):=(\pi_i(u),\pi_i(\btheta))-\ub\|_2\le\sqrt{4\delta_z^2+R_m^2}/(2^{i}\sqrt{n})$. So $S^{(i)}:=S_1^{(i)}\times S_2^{(i)}$ is a $\sqrt{4\delta_z^2+R_m^2}/(2^{i}\sqrt{n})$-net of $U_k:=I\times \Theta_k$ with size $|S^{(i)}|\le2^i(\sqrt{n}+1)\cdot(2^{i+1}C(d,n)+1)^{d}$ and  $C(d,n)=\sqrt{(d+1)(W_x\log n+\log^2 n\log\log n)}$.\\
	Because $Z(u,\btheta)$ is continuous almost surely, we have that for any $M\in \NN^+$
	$$
	Z(\ub) - Z(\pi_M(\ub)) = \sum_{i=M}^\infty [Z(\pi_{i+1}(\ub)) - Z(\pi_i(\ub))],
	$$
	and thus 
	\begin{equation}\label{eq-supZmix}
	\sup_{\ub\in U_k} Z(\ub) \leq  \sup_{\ub\in U_k} Z(\pi_M(\ub))  + \sum_{i=M}^\infty \sup_{\ub\in U_k}[Z(\pi_{i+1}(\ub)) - Z(\pi_i(\ub))]
	\end{equation}
	almost surely. 
	If we can choose a $M$ properly then the two terms at the right hand side of \eqref{eq-supZmix} can be both well controlled. For this reason, we let {$M=\lceil \frac{4}{\log 2}\log \frac{1}{b_n} \rceil$}. We then first bound $\sup_{\ub\in U_k}Z(\pi_M(\ub))$ by using a union bound. By our definition on $Z(\ub)$, we can write 
	\begin{align*}
		Z(\ub)=\frac{1}{nb_k}\sum_{t\in I_k}A_j(\ub).
	\end{align*}
	in which $A_t(\ub)=K(\frac{w_t(\btheta)-u}{b_k})y_t - \EE K(\frac{w_t(\btheta)-u}{b_k})y_t$. Similar with our case in proving Lemma \ref{lem-dev}, we have that $\EE[A_t(\ub)]=0$ and $|A_t(\ub)|\le \bar{K}$ almost surely. We next prove the bound of variance of $A_t(\ub)$ and the covariance between $A_j(\ub)$ and $A_i(\ub)$ with $j>i$. Following similar procedures with Lemma \ref{lem-dev}, we first conclude that
	\begin{align*}
		\textrm{Var}(A_t(\ub))\le C_4'b_k,
	\end{align*}
	where $C_4'=C_4=\max\{\bar{f}\cdot\int K(s)^2\ud s, \bar{f}\cdot\int K'(s)^2 \ud s\}$ is defined in the same way with our proof of Lemma \ref{lem-dev}. We next  control the covariance of $A_j(\ub)$ and $A_i(\ub)$ with $j>i$.
	\begin{align*}
		\textrm{Cov}(A_j(\ub),A_i(\ub))&=\EE\Big[K(\frac{w_j(\btheta)-u}{b_k})y_jK(\frac{w_i(\btheta)-u}{b_k})y_i\Big]-\EE\Big[K(\frac{w_j(\btheta)-u}{b_k})y_j\Big]\EE\Big[K(\frac{w_i(\btheta)-u}{b_k})y_i\Big]\\
		&=\EE\Big[K(\frac{w_j(\btheta)-u}{b_k})K(\frac{w_i(\btheta)-u}{b_k})\EE[y_jy_i\given w_j(\btheta),w_i(\btheta)]\Big]\\
		&\quad -\EE\Big[K(\frac{w_j(\btheta)-u}{b_k})y_j\Big]\EE\Big[K(\frac{w_i(\btheta)-u}{b_k})y_i\Big]
	\end{align*}
	For simplicity, for any $\btheta\in\Theta_0$, we define $r(u_i,u_j):=\EE[y_iy_j\given w_j(\btheta)=u_j,w_i(\btheta)=u_i]$ and $r(u_j)=\EE[y_j\given w_j(\btheta)=u_j ].$ Then after some simple calculation, we further obtain
	\begin{align*}
		\textrm{Cov}(A_j(\ub),A_i(\ub))&=\int\int K(\frac{w_j(\btheta)-u}{b_k})K(\frac{w_i(\btheta)-u}{b_k})r(w_i(\btheta),w_j(\btheta))f(w_i(\btheta), w_j(\btheta))\ud w_i(\btheta)\ud w_j(\btheta)\\&\quad -\int\int K(\frac{w_j(\btheta)-u}{b_k})K(\frac{w_i(\btheta)-u}{b_k})r(w_i(\btheta))r(w_j(\btheta))f(w_i(\btheta))f(w_j(\btheta))\ud w_i(\btheta)\ud w_j(\btheta)\\
		&=b_k^2\int\int K(s_1)K(s_2)[r(b_ks_1+u,b_ks_2+u)f(b_ks_1+u,b_ks_2+u)\\&\quad-r(b_ks_1+u)r(b_ks_2+u)f(b_ks_1+u)f(b_ks_2+u)]\ud s_1 \ud s_2
	\end{align*}
	We next prove that $h(u_i,u_i):=r(u_i,u_j)f(u_i,u_j)$ stays close to $h(u_i)h(u_j):=r(u_i)f(u_i)r(u_j)f(u_j)$ for all $(u_i,\,u_j)$ in the following Lemma \ref{sup_mixh}.
	\begin{lemma}\label{sup_mixh}
		Under Assumptions given in Lemma \ref{mix_var}. We let $g^*(u_i,u_j):=h(u_i,u_j)-h(u_i)h(u_j)$, if we further assume $g^*(u_i,u_j)$ is Lipschitz continuous w.r.t. $(u_i,u_j)$ with Lipschitz constant $l$, then 
		we have 
		\begin{align*}
			\sup_{u_i,u_j}|g^*(u_i,u_j)|\le (1/4+\sqrt{2}l)\beta_{j-i}^{1/3}
		\end{align*}
	\end{lemma}
	\begin{proof}
		For any $x$ we define 
		\begin{align*}
			B(x,\epsilon);=\{ x’: \|x'-x\| \le \epsilon\}, \,\, \epsilon>0, x\in \RR
		\end{align*}
		First, we prove $|\EE [y_iy_j\II_{\{w_i(\btheta)\in B(x,\epsilon),w_j(\btheta)\in B(y,\epsilon)\}}]-\EE[y_i\II_{\{w_i(\btheta)\in B(x,\epsilon)\}}]\EE[y_j\II_{\{w_j(\btheta)\in B(y,\epsilon)\}}]|\le \beta_{j-i}.$ We have
		\begin{align*}
			&\big|\EE [y_iy_j\II_{\{w_i(\btheta)\in B(x,\epsilon),v_j(\btheta)\in B(y,\epsilon)\}}]-\EE[y_i\II_{\{w_i(\btheta)\in B(x,\epsilon)\}}]\EE[y_j\II_{\{v_j(\btheta)\in B(y,\epsilon)\}}]\big|
			\\&\quad =\big|\EE[\II_{\{w_i(\btheta)\in B(x,\epsilon),v_j(\btheta)\in B(y,\epsilon)\}}\EE[y_iy_j\given \tilde{\xb}_i,\tilde{\xb}_j,p_i,p_j]]\\&\quad\quad-\EE[\II_{\{w_i(\btheta)\in B(x,\epsilon)\}}\EE[y_i\given \tilde{\xb}_i,p_i] ]\EE[\II_{\{w_j(\btheta)\in B(y,\epsilon)\}}\EE[y_j\given \tilde{\xb}_j,p_j]]\big|
			\\&\quad=\big|\EE[\EE[y_i\II_{\{w_i(\btheta)\in B(x,\epsilon)\}}\given \tilde{\xb}_i,p_i]\EE[y_j\II_{\{w_j(\btheta)\in B(y,\epsilon)\}}\given \tilde{\xb}_j,p_j]]\\&\quad\quad-\EE[\EE[y_i\II_{\{w_i(\btheta)\in B(x,\epsilon)\}}\given \tilde{\xb}_i,p_i] ]\EE[\EE[y_j\II_{\{w_j(\btheta)\in B(y,\epsilon)\}}\given \tilde{\xb}_j,p_j]]\big|
		\end{align*}
		As $p_i,i\in |I_k|,k\ge 0$ are independent, so the $\sigma$-algebra generated by the joint distribution of $\tilde{\xb}_i,p_i$ still follows strong-$\beta$ and -$\alpha$ conditions given in our Assumption \ref{assmix1}.
		Moreover, we have $\EE[y_i\II_{\{w_i(\btheta)\in B(x,\epsilon)\}}\given \tilde{\xb}_i,p_i] $ lies in $\sigma(\tilde{\xb}_i,p_i)$ and $\EE[y_j\II_{\{w_j(\btheta)\in B(y,\epsilon)\}}\given \tilde{\xb}_j,p_j]$ lies in $\sigma(\tilde{\xb}_j,p_j)$ with $j>i$. So we are able to obtain the upper bound: 
		\begin{align}\label{uniform_g}
			\big|\EE [y_iy_j\II_{\{w_i(\btheta)\in B(x,\epsilon),w_j(\btheta)\in B(y,\epsilon)\}}]-\EE[y_i\II_{\{w_i(\btheta)\in B(x,\epsilon)\}}]\EE[y_j\II_{\{w_j(\btheta)\in B(y,\epsilon)\}}]\big|\le \beta_{j-i}
		\end{align}
		by using Corollary 1.1 in \cite{Bosq1996}.
		
		Next, we get an upper bound of  $\sup_{(u_i,u_j)}|g^*(u_i,u_j)|$. From \eqref{uniform_g} and our definition on $g^*$, we obtain
		\begin{align*}
			\beta_{j-i}\ge \Big|\int_{B(x,\epsilon)\times B(y,\epsilon)}g^*(u_i,u_j) \ud u_i \ud u_j\Big|:=\cI
		\end{align*}
		Then by the mean value property we have $\cI=4\epsilon^2|g^*(x',y')|$ for some  $(x',y')\in B(x,\epsilon)\times B(y,\epsilon)$.
		Moreover, as we assume $g$ is Lipschitz, then we get
		$$
		|g^*(x,y)|\le |g^*(x',y')|+\sqrt{2}l\epsilon
		$$
		Hence, we finally achieve
		$$
		|g^*(x,y)|\le \beta_{j-i}/(4\epsilon^2)+\sqrt{2}l\epsilon.
		$$
		for any fixed $(x,y)$.
		As this inequality holds for all $\epsilon>0$,  we choose $\epsilon=\beta_{j-i}^{1/3}$ and we conclude the proof of our Lemma \ref{sup_mixh}.
	\end{proof}
	By our conclusion from Lemma \ref{sup_mixh}, we are able to find a constant $C_5'$ such that $|\sum_{j>i}\textrm{Cov}(A_j(\bu),A_i(\bu))|\le C_5'b_n$ holds according to our assumptions on $\beta_{j-i},\,j>i$, where we set $C_5'=(1/4+\sqrt{2}l)\sum_{j>0}\beta_j^{1/3}$.
	Next we introduce the following Bernstein inequality under strong-mixing conditions, in order to achieve an upper bound of $Z(\ub)$.
	\begin{lemma}\label{vecbern_mix}[Theorem 2 in \cite{merlevede}]
		Under conditions of Lemma \ref{mix_var}, for all $n\ge 2$, we have
		\begin{align*}
			\PP(|Z(\ub)|\ge nb_kx)=\PP(|\sum_{j\in I_k}A_j(\ub)|\ge nb_kx)\le 2\exp\Big(-\frac{C_bb_k^2n^2x^2}{v^2n+\bar{K}^2+nb_kx\log^2n}\Big)
		\end{align*}
		Here
		\begin{align*}
			v^2=\sup_{i>0}(\textrm{Var}(A_i(\ub))+2\sum_{j>i}|\textrm{Cov}(A_i(\ub),A_j(\ub))|),
		\end{align*}
		$C_b$ is a pure constant and $\bar{K}$ is defined as the upper bound of $|A_j(\ub)|$ with any $j\in[n]$.
	\end{lemma}
	By our conclusions from Lemma \ref{sup_mixh} and Lemma \ref{vecbern_mix}, we conclude there exists a constant $C_6'=(C_4'+2C_5')$ such that $v^2\le C_6'b_n$, so we obtain
	\begin{align*}
		\PP(|Z(\ub)|\ge x)&\le 2\exp\Big(-\frac{C_bb_k^2n^2x^2}{C_6'nb_k+\bar{K}^2+nb_kx\log^2 n}\Big)\\
		&\le 2\exp\Big(-\frac{C_bnb_kx^2}{(C_6'+\bar{K}^2+\log^2 n)(1+x)}\Big)
	\end{align*}
	The last inequality follows from our assumption that $b_k\ge 1/n=1/|I_k|$ for any $k\ge 1$ in given Lemma \ref{mix_var}. Further, we set $C_7'=C_b/(2C_6'+2\bar{K}^2+2)$. Then we take the union bound over $U_k$, which gives 
	\begin{align*}
		\PP(\sup_{\ub\in U_k}|Z(\pi_M(\ub))|\ge x)&\le |S^{(M)}|\cdot\PP(|Z(\ub)|\ge x)\\
		&\le 2\cdot 2^{M}(\sqrt{n}+1)\cdot (2^{M+1}\sqrt{d}+1)^d\cdot e^{-\frac{C_7'nb_k}{\log^2n}\min\{x,x^2\}}\\
		&\le 2e^{(d+1)M\log 2+\log(\sqrt{n}+1)+d\log(2C(n,d)+2)-\frac{C_7'nb_k}{\log^2n}\min\{x,x^2\}}.
	\end{align*}
	Since we define $M=\lceil \frac{4}{\log 2}\log \frac{1}{b_k}\rceil$, then we choose	
	\begin{align}\label{x_nd}
		x(n,d):=\frac{\log n}{\sqrt{nb_k}}\sqrt{\Big[(d+1)4\log\frac{1}{b_k}+2(d+1)\log 2+\log(\sqrt{n}+1)+d\log(2C(n,d)+2)+\log\frac{8}{\delta}\Big]/C_7'},
	\end{align}
	where $C(n,d)=\sqrt{(d+1)(W_x\log n+\log^2 n \log\log n)}$. We then have
	\begin{align*}
		\PP(\sup_{\ub\in U_k}|Z(\pi_M(\ub))|\ge x(n,d))\le \frac{\delta}{4}.
	\end{align*}
	 when $\delta>8\exp(-nb_k/(C_7'\log^2n))$ and $nb_k\ge 2\log^2n[(d+1)4\log\frac{1}{b_k}+2(d+1)\log 2+\log(\sqrt{n}+1)+d\log(2C(n,d)+2)]/C_7'$ (because under such conditions, we have $x(n,d)\le 1$). Now, we proceed to bound the later term at the right hand side of \eqref{eq-supZmix}.
	Similar with our cases stated in the proof of Lemma \ref{lem-dev}, for any $\ub_1:=(u,\btheta_1), \ub_2:=(s,\btheta_2)\in I\times \Theta_k$, we have that 
	$$
	Z(\ub_1)-Z(\ub_2)=Z(u,\btheta_1) - Z(s,\btheta_2) = \frac{1}{nb_k}\sum_{t\in I_k}B_t(u, \btheta_1,s,\btheta_2), 
	$$
	where 
	$$
	B_t(u, \btheta_1,s,\btheta_2) = y_t\left(K(\frac{w_t(\btheta_1)-u}{b_k}) - K(\frac{w_t(\btheta_2)-s}{b_k}) \right) - \EE y_t\left(K(\frac{w_t(\btheta_1)-t}{b_k}) - K(\frac{w_t(\btheta_2)-s}{b_k}) \right).
	$$
	We have $\EE B_t(u, \btheta_1,s,\btheta_2) = 0$, and that 
	\begin{align*}
		|Z(\ub_1)-Z(\ub_2)|=|B_t(u, \btheta_1,s,\btheta_2)|&\leq 2\left|y_j(K(\frac{w_t(\btheta_1)-u}{b_k}) - K(\frac{w_t(\btheta_2)-s}{b_k}))\right|\\&\leq \frac{2l_K\sqrt{1+\max_{\xb\in\cX}\|\xb\|_2^2+1}}{b_n}\cdot\|\ub_1-\ub_2\|_2:=\frac{C^*}{b_n}\|\ub_1-\ub_2\|_2.
	\end{align*}
	The last inequality follows from the Lipschitz property of $K(\cdot)$ and for simplicity we use $C^{*}$ to denote the constant $2l_K\sqrt{\max_{\xb\in\cX}\|\xb\|_2^2+2}=2l_K\sqrt{R_{\cX}^2+2}$. Then according to the Bernstein inequality given in Lemma  \ref{lemmixbern}, we have
	\begin{align*}
		\PP(|\sum_{t=1}^{n}B_t(\ub_1,\ub_2)|\ge nb_kx)\le 2\exp\bigg(-\frac{C_wn^2b_k^2x^2}{n\frac{C^{*2}\|\ub_1-\ub_2\|_2^2}{b_k^2}+nb_kx\frac{C^*\|\ub_1-\ub_2\|_2}{b_k}\log^2 n} \bigg ).
	\end{align*}
	Recall that $\forall \ub\in U_n$, we have $\|\pi_i(\ub)-\pi_{i+1}(\ub)\|_2\le\frac{\sqrt{4\delta_z^2+R_m^2}}{2^{i-1}\sqrt{n}}$. We then use the union bound to get
	
	\begin{align*}
		&\PP(\sup_{\ub\in U_k}|Z(\pi_{i+1}(\ub))-Z(\pi_i(\ub))|\ge x)\\&\quad\le 2^{2i+2}(\sqrt{n}+1)^2(2^{i+2}C(n,d)+1)^{2d}\cdot 2e^{-\frac{C_8’ 2^{i-1}n^{3/2}b_k^4x^2}{(\frac{4\delta_z^2+R_m^2}{2^{i-1}\sqrt{n}}+b_k^2\sqrt{4\delta_z^2+R_m^2}\log^2 n)(1+x)}}
	\end{align*}
	in which $C_8'=C_w/\max\{C^{*2},C^*\}$. We let $x=\frac{\sqrt{(4\delta_z^2+R_m^2)/(2^{i-1}\sqrt{n})+b_k^2\sqrt{4\delta_z^2+R_m^2}\log^2n}}{2^{(i-1)/2}n^{3/4}b_k^2}\cdot \epsilon_i.$
	Then we have
	\begin{align}
		&\PP\bigg(\sup_{\ub\in U_k}|Z(\pi_{i+1}(\ub))-Z(\pi_i(\ub))|\ge {\sqrt{(4\delta_z^2+R_m^2)/(2^{i-1}\sqrt{n})+b_k^2\sqrt{4\delta_z^2+R_m^2}\log^2n}}/({2^{(i-1)/2}n^{3/4}b_k^2})\cdot \epsilon_i\bigg)\noindent\\&\quad\le 2^{2i+2}(\sqrt{n}+1)^2(2^{i+2}C(n,d)+1)^{2d}\cdot 2e^{-\frac{C_8'\epsilon_i^2}{1+\frac{\sqrt{(4\delta_z^2+R_m^2)/(2^{i-1}\sqrt{n})+b_k^2\sqrt{4\delta_z^2+R_m^2}\log^2n}}{2^{(i-1)/2}n^{3/4}b_k^2}\cdot \epsilon_i}} \label{bern_var2}
	\end{align}
	We observe that if we could choose $\epsilon_i$ such that
	$$
	\frac{\sqrt{(4\delta_z^2+R_m^2)/(2^{i-1}\sqrt{n})+b_k^2\sqrt{4\delta_z^2+R_m^2}\log^2n}}{2^{(i-1)/2}n^{3/4}b_k^2}\cdot \epsilon_i<1,$$
	holds, then the right hand side of \eqref{bern_var2} satisfies
	\begin{align}\label{bernmix2}
		\eqref{bern_var2}\le 2^{2i+2}(\sqrt{n}+1)^2(2^{i+2}C(n,d)+1)^{2d}\cdot 2e^{-\frac{C_8'\epsilon_i^2}{2}}.
	\end{align}
	Now we choose $\epsilon_i=\sqrt{[(4d+6)(i+1)\log 2+4\log(\sqrt{n}+1)+4d\log(2C(n,d)+2)+2\log (8/\delta)]/C_8'}$. Then we have
	\begin{align*}
		&\frac{\sqrt{(4\delta_z^2+R_m^2)/(2^{i-1}\sqrt{n})+b_k^2\sqrt{4\delta_z^2+R_m^2}\log^2n}}{2^{(i-1)/2}n^{3/4}b_k^2}\cdot \epsilon_i\\
		&\quad\le\frac{1}{2^{(i-1)/2}n^{3/4}b_k}\Big[\frac{\sqrt{4\delta_z^2+R_m^2}}{2^{(i-1)/2}b_kn^{1/4}}+(4\delta_z^2+R_m^2)^{1/4}\log n \Big]\cdot \epsilon_i.
	\end{align*}
	Here we only consider $i\ge M=\lceil \frac{4}{\log 2}\log \frac{1}{b_k} \rceil$, and we have $2^{M/4}\cdot b_k=1$. In addition, we also get $\max_i (i+1)/2^{(i-2)/2}\le 3$. Hence, we have 
	\begin{align*}
		&\frac{1}{2^{(i-1)/2}n^{3/4}b_k}\Big[\frac{\sqrt{4\delta_z^2+R_m^2}}{2^{(i-1)/2}b_kn^{1/4}}+(4\delta_z^2+R_m^2)^{1/4}\log n \Big]\cdot \epsilon_i<1,
	\end{align*}
	if $\delta\ge 8\exp(-C_8'n^{3/2}/(16(4\delta_z^2+R_m^2)\log^2n))$ and $n\ge \{8(4\delta_z^2+R_m^2)\log^2n\cdot[(12d+18)\log 2+4\log(\sqrt{n}+1)+4d\log(2C(n,d)+2)]/C_8' \}^{2/3}.$
	Then after plugging our setting of $\epsilon_i$ into \eqref{bernmix2}, we obtain
	\begin{align*}
		\PP\bigg(\sup_{\ub\in U_k}|Z(\pi_{i+1}(\ub))-Z(\pi_i(\ub))|&\ge {\sqrt{(4\delta_z^2+R_m^2)/(2^{i-1}\sqrt{n})+b_k^2\sqrt{4\delta_z^2+R_m^2}\log^2n}}/({2^{(i-1)/2}n^{3/4}b_k^2})\cdot \epsilon_i\bigg)\\&\quad
		\le\frac{1}{2^{i+1}}\cdot \frac{\delta}{4}.
	\end{align*}
	And we notice
	\begin{align*}
		&	\sum_{i=M}^{\infty}\frac{\sqrt{(4\delta_z^2+R_m^2)/(2^{i-1}\sqrt{n})+b_k^2\sqrt{4\delta_z^2+R_m^2}\log^2n}}{2^{(i-1)/2}n^{3/4}b_k^2}\cdot \epsilon_i\\
		&\quad\le \sum_{i=M}^{\infty}\frac{\sqrt{4\delta_z^2+R_m^2}}{2^{i-1}nb_k^2}\cdot \epsilon_i+\frac{\sqrt{(4\delta_z^2+R_m^2)\log^2n}}{2^{(i-1)/2}n^{3/4}b_k}\cdot \epsilon_i:=\mathbf{I}+\mathbf{II}.
	\end{align*}
	For term $\mathbf{I}$, we have
	\begin{align*}
		\mathbf{I}&=\sum_{i=M}^{\infty}\frac{\sqrt{4\delta_z^2+R_m^2}}{2^{i-1}nb_k^2}\cdot\sqrt{[(4d+6)(i+1)\log 2+4\log(\sqrt{n}+1)+4d\log(2C(n,d)+2)+2\log(8/\delta)]/C_8'}
		\\&\le\frac{\sqrt{(4\delta_z^2+R_m^2)/C_8'}}{nb_k^2}\Big[\sqrt{(4d+6)\log 2}\sum_{i=M}^{\infty}\frac{i+1}{2^{i-1}}+ \frac{\sqrt{4\log(\sqrt{n}+1)+4d\log(2C(n,d)+2)+2\log(8/\delta)}}{2^{M-2}} \Big]\\&\le \frac{\sqrt{(4\delta_z^2+R_m^2)/C_8'}}{nb_k^2}\frac{2M}{2^{M-2}}\Big[\sqrt{(4d+6)\log 2}+ \sqrt{4\log(\sqrt{n}+1)}+\sqrt{4d\log(2C(n,d)+2)}+\sqrt{2\log(8/\delta)}\Big]\\
		&\le \frac{\sqrt{(4\delta_z^2+R_m^2)/C_8'}}{n}\frac{8M}{2^{M/2}}\frac{1}{2^{M/2}b_k^2}\Big[\sqrt{(4d+6)\log 2}+ \sqrt{4\log(\sqrt{n}+1)}+\sqrt{4d\log(2C(n,d)+2)}+\sqrt{2\log(8/\delta)}\Big]
		\\&\le \frac{C_9'}{n}\Big[\sqrt{(4d+6)\log 2}+ \sqrt{4\log(\sqrt{n}+1)}+\sqrt{4d\log(2C(n,d)+2)}+\sqrt{2\log(8/\delta)}\Big],
	\end{align*}
	in which $C_9'$ is a pure constant such that $C_9'=\sqrt{(4\delta_z^2+R_m^2)/C_8'}\cdot \max_{i}(8i/2^{i/2})=16\sqrt{(4\delta_z^2+R_m^2)/C_8'}$ and $C(n,d)\le\sqrt{(d+1)(W_x\log n+\log^3n)}$.
	Then we obtain
	\begin{align}
		\sqrt{4d\log(2C(n,d)+2)}\nonumber&\le\sqrt{4d\log\Big(4\sqrt{(d+1)(W_x\log n+\log^3 n)}\Big)}\\&\le \sqrt{4d\log \Big(4\sqrt{2}\sqrt{(d+1)\max\{1,W_x\}\log^3 n}\Big)}\nonumber
		\\&\le \sqrt{4d\log(4\sqrt{2})}+\sqrt{2d\log(\max\{W_x,1\}(d+1))}+\sqrt{6d\log n}. \label{ineq_cnd}
	\end{align}
	Next, we are able to find a pure constant $C_{10}'=6\sqrt{6}$ such that $ \sqrt{(4d+6)\log 2}+ \sqrt{4\log(\sqrt{n}+1)}+\sqrt{4d\log(2C(n,d)+2)}\le 6\sqrt{6}\sqrt{(d+1)\log (\max\{W_x,1\}(d+1))\log n}$ as long as $n\ge 3$ according to \eqref{ineq_cnd}. Thus, we finally achieve
	$$ \mathbf{I}\le \frac{C_{11}'}{n}\Big(\sqrt{(d+1)\log (\max\{W_x,1\}(d+1))\log n}+\sqrt{2\log(8/\delta)}\Big),$$
	where $C_{11}'=C_{10}'\cdot C_9'$. For term $\mathbf{II}$, we obtain
	\begin{align*}
		\mathbf{II}&=\sum_{i=M}^{\infty}\frac{\sqrt{(4\delta_z^2+R_m^2)\log^2n/C_8'}}{2^{(i-1)/2}n^{3/4}b_k}[\sqrt{(4d+6)(i+1)\log 2}\\&\quad+\sqrt{4\log(\sqrt{n}+1)}+\sqrt{4d\log(2C(n,d)+2)}+\sqrt{2\log(8/\delta)}]
		\\&\le \frac{\sqrt{(4\delta_z^2+R_m^2)/C_8'}\log n}{n^{3/4}b_k}\Big[\sqrt{(4d+6)\log 2}\sum_{i=M}^{\infty}\frac{i+1}{2^{(i-1)/2}}\\&\quad+\frac{\sqrt{4\log(\sqrt{n}+1)+4d\log(2C(n,d)+2)+2\log(8/\delta)}}{2^{(M-2)/2}}\Big]
		\\&\le \frac{\sqrt{(4\delta_z^2+R_m^2)/C_8'}\log n}{n^{3/4}}\frac{8\sqrt{2}M}{2^{M/4}}\frac{1}{2^{M/4}b_k}\Big[\sqrt{(4d+6)\log 2} + \sqrt{4\log(\sqrt{n}+1)}\\&\quad+\sqrt{4d\log(2C(n,d)+2)}+\sqrt{2\log(8/\delta)}\Big].
	\end{align*}
	We are also able to find a pure constant $C_{12}'$ such that $C_{12}'=\sqrt{(4\delta_z^2+R_m^2)/C_8'}\max_{i}(8\sqrt{2}i/2^{i/4})=24\sqrt{(4\delta_z^2+R_m^2)/C_8'}$ and $C_{13}'=C_{10}'\cdot C_{12}' $. Then we obtain 
	$$ \mathbf{II}\le \frac{C_{13}'\log n}{n^{3/4}}\Big(\sqrt{(d+1)\log (\max\{W_x,1\}(d+1))\log n}+\sqrt{2\log {8}/{\delta}}\Big).$$

	After combining our inequalities of $\mathbf{I}$ and $\mathbf{II}$, we obtain a union bound: 
	\begin{align*}
		\PP\bigg(\sup_{\ub\in U_k}|Z(\ub)-Z(\pi_M(\ub))|\ge x_2(n,d):&=\frac{C_{14}'\log n}{n^{3/4}}\Big(\sqrt{(d+1)\log (\max\{W_x,1\}(d+1))\log n}+\sqrt{2\log (8/\delta)}\Big)\bigg)\\&\le\sum_{i=M}^{\infty}\frac{1}{2^{i+1}}\frac{\delta}{4}\le \frac{\delta}{4},
	\end{align*}
	in which we choose $C_{14}'=2\max\{C_{11}',C_{13}'\}$. Then we  get 
	\begin{align*}
		\PP\Big(\sup_{\ub\in U_k}Z(\ub)\ge x(n,d)+x_2(n,d)\Big)\le \frac{\delta}{4}+\frac{\delta}{4}=\frac{\delta}{2}.
	\end{align*}
	where the expression of $x(n,d)$ is given in \eqref{x_nd}. As a reminder, we have 
	\begin{align*}
		x(n,d):=\frac{\log n}{\sqrt{nb_k}}\sqrt{\Big[(d+1)4\log\frac{1}{b_k}+2(d+1)\log 2+\log(\sqrt{n}+1)+d\log(2C(n,d)+2)+\log\frac{8}{\delta}\Big]/C_7'}.
	\end{align*}We obtain there exist a universal constant $C_{15}'=8/\sqrt{C_7'}$ such that 
	\begin{align*}
		x(n,d)\le \frac{C_{15}'\log n}{\sqrt{nb_k}}\bigg(\sqrt{(d+1)\log (\max\{W_x,1\}(d+1))\log n}+\sqrt{2\log \frac{8}{\delta}}\bigg).
	\end{align*}
	Then we finally achieve
	\begin{align*}
		\PP\bigg(\sup_{\ub\in U_k}Z(\ub)\ge \frac{C_{16}'\log n}{\sqrt{nb_k}}\bigg(\sqrt{(d+1)\log (\max\{W_x,1\}(d+1))\log n}+\sqrt{2\log \frac{8}{\delta}}\bigg)\bigg)=\frac{\delta}{2},
	\end{align*}
	where we let $C_{16}'=2\max\{C_{14}',C_{15}' \}$ and $C_{17}'=C_{16}\log(\max\{ W_x,e\})$. Thus,  $nb_k\ge 4C_{17}'^2\log^3 n[(d+1)\log (d+1)]$ and $\delta\ge 8\exp(-nb_k/(8C_{17}'^2\log^2n))$ becomes a sufficient condition to make $x(n,d)+x_2(n,d)$ be smaller than $1$. Following similar procedure, we are able to prove the same inequality for $f_n$, so  we conclude our proof of Lemma \ref{mix_var}.
\end{proof}
The remaining part of Lemma \ref{conckernel_main_mix} only involves getting  a uniform upper bound for $|r_{\btheta}(u)-r_{\btheta_0}(u)|$  and thus $|\hat{r}_k(u,\btheta)-r_{\btheta_0}(u)|$ for any $\btheta\in \Theta_k$ and $u\in I$. Similar with the corresponding proof of Lemma \ref{conckernel_main}, we have 
\begin{align*}
	\sup_{u\in I,\btheta\in \Theta_k}|r_{\btheta}(u)-r_{\btheta_0}(u)|\le l_r R_{\cX}\cdot\frac{2}{c_{\min}}\sqrt{\frac{(d+1)(6W_x^2\log n+6W_x\log^2 n\log\log n)}{C_w n}}.
	\end{align*}
Finally, by setting $b_k=n^{-1/(2m+1)}$ and combining our results obtained in Lemma \ref{mix_bias} and Lemma \ref{mix_var}, we conclude our results for Lemma \ref{conckernel_main_mix}. In addition, our way of deriving constants $B_{mx,K},B'_{mx,K}$ and $C_{mx,K}$ is similar with that in Lemma \ref{conckernel_main_mix}, so we omit the details here.
\end{proof}

\subsection{Proof of Lemma \ref{conckernel2_main_mix} and Theorem \ref{mainthm2}}\label{pf-b4}
The proof of Lemma \ref{conckernel2_main_mix}  and Theorem \ref{mainthm2} are straight forward by combining the proof of Lemma \ref{conckernel2_main} and Lemma \ref{conckernel_main_mix}, so we omit the details here.

\section{Additional Plots}\label{add_plots}

{In this section, we directly plot reg$(T)$ for all the settings discussed in the main paper. From Figure \ref{indp_app} - Figure \ref{mix_dpapp}, we see that the blue solid lines depicted in every figure are close to the other two lines that depict regrets with either known $\btheta_0$ or $g(\cdot )$ in Algorithm \ref{pricing_1}. This fact reflects the robustness of our estimators on $\btheta_0$ and $g(\cdot)$ in every episode.}

\begin{figure}[H]
	\centering
	\begin{tabular}{ccc}
		\hskip-30pt\includegraphics[width=0.35\textwidth]{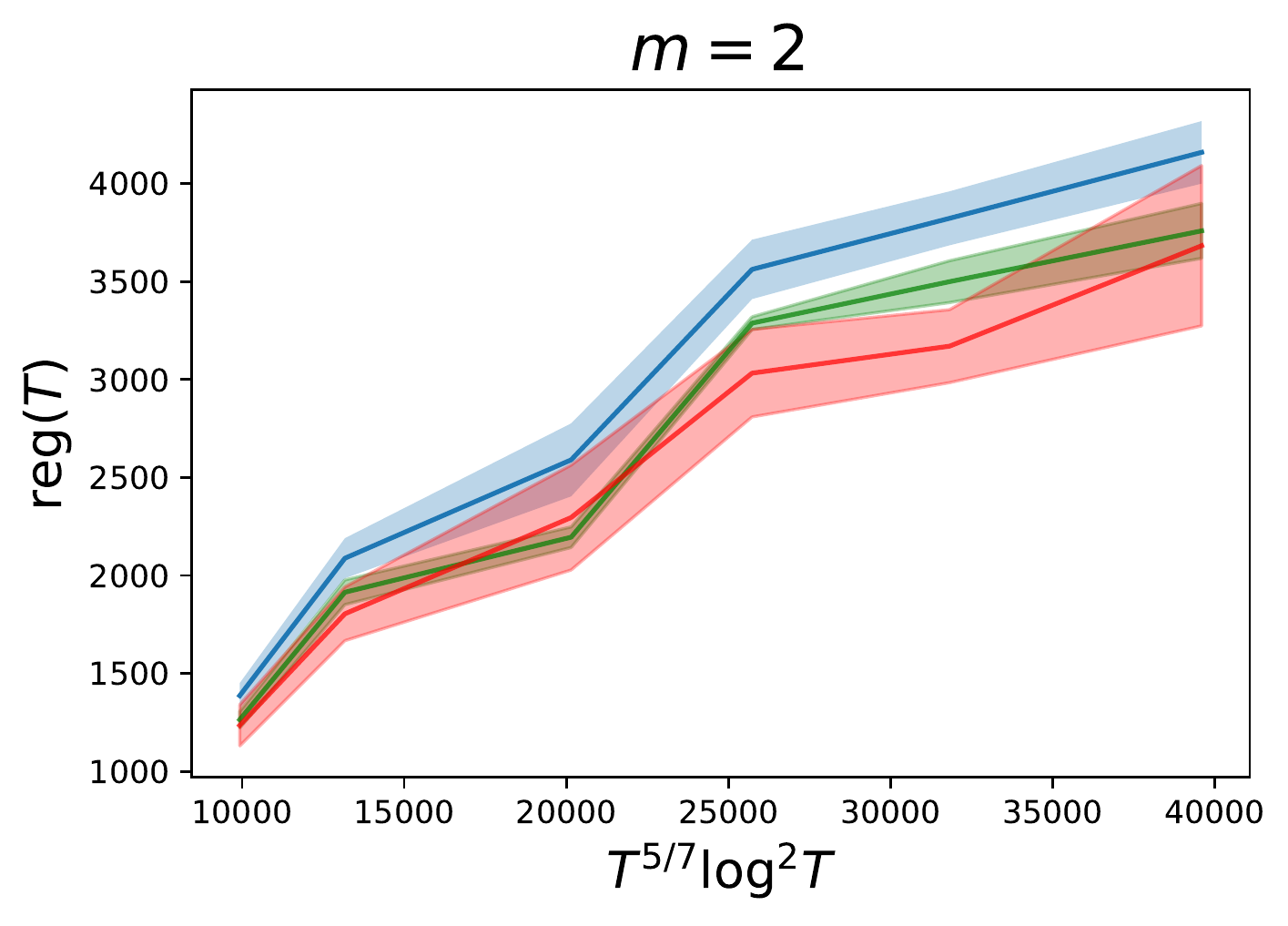}
		&
		\hskip-6pt\includegraphics[width=0.35\textwidth]{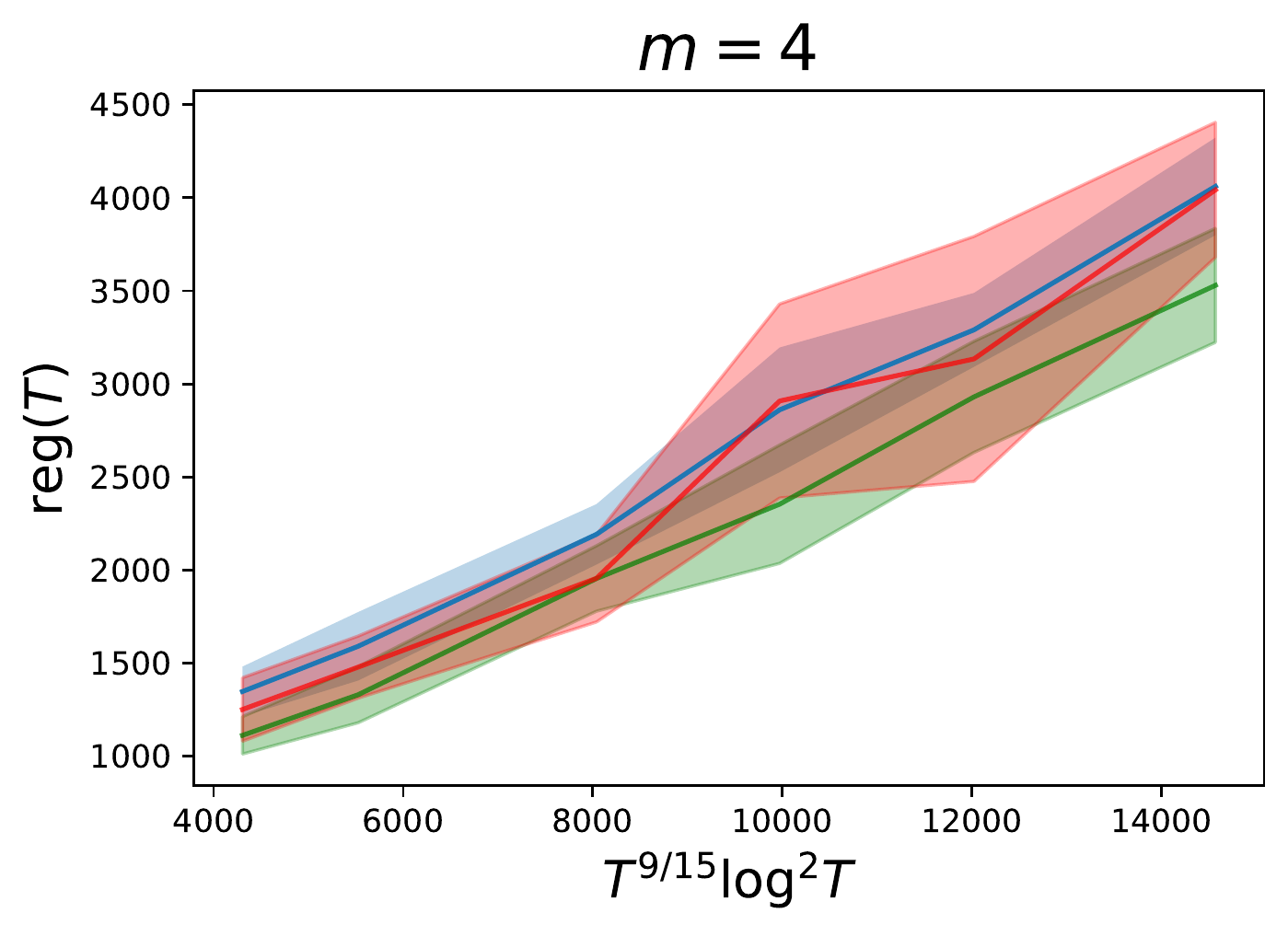}
		&
		\hskip-5pt\includegraphics[width=0.35\textwidth]{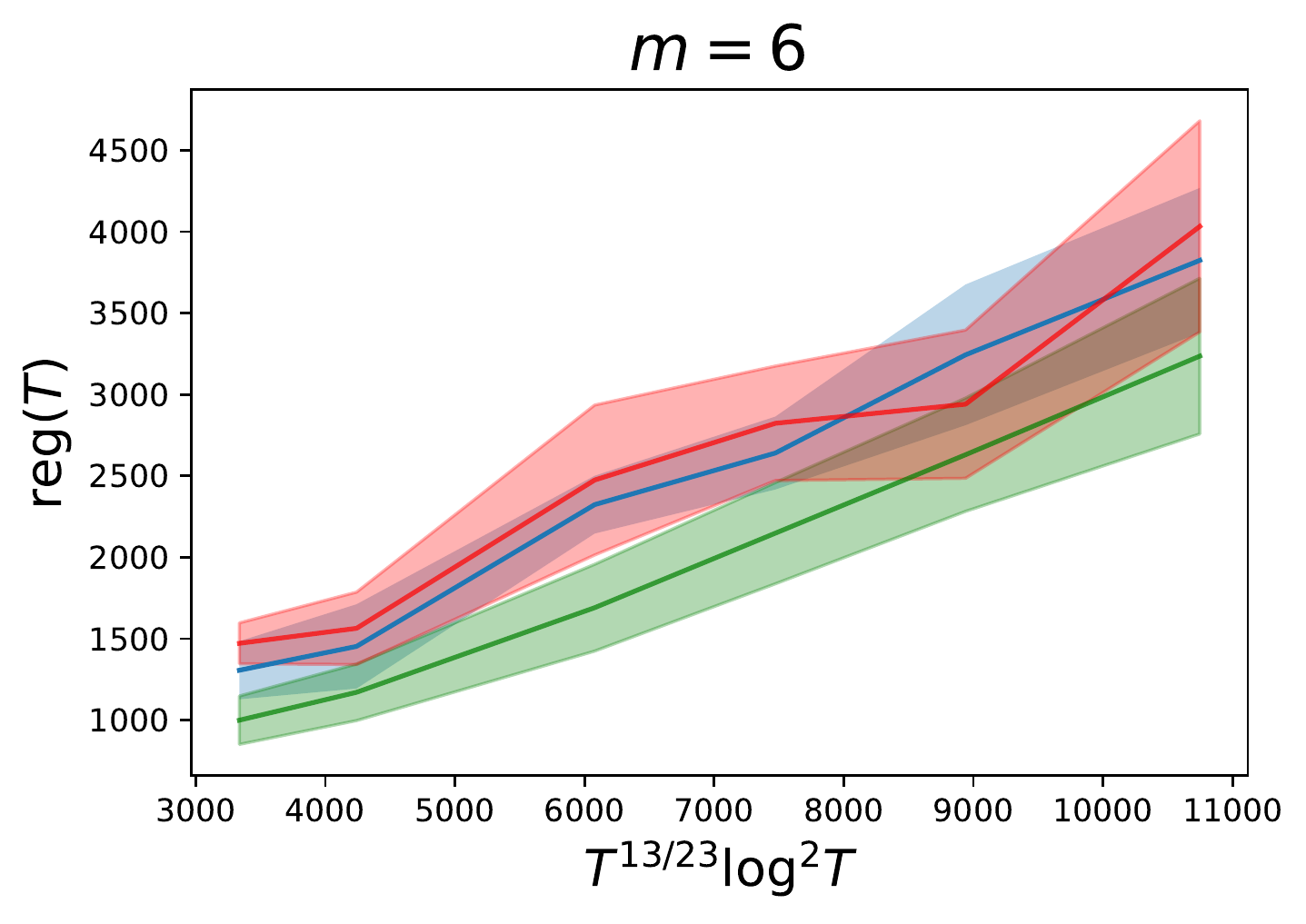} \\
		(a)  & (b) &(c)
	\end{tabular}\\
	\caption{From left to right, we plot empirical regret reg$(T)$ against $T^{(2 m+1)/(4 m-1)}\log^2 T$ with $ m\in [2,4,6]$ in the setting with i.i.d. covariates with independent entries. Solid blue, green, red lines, represent the mean regret collected by implementing the Algorithm \ref{pricing_1} for $30$ times with unknown $g(\cdot)$, $\btheta_0$, unknown $g(\cdot)$ but known $\btheta_0$ and known $g(\cdot)$ but unknown $\btheta_0$ in the exploitation phase respectively. Light color areas around those solid lines depict the standard error of our estimation of reg$(T)$.}
	\label{indp_app}
\end{figure}

\begin{figure}[H]
	\centering
	\begin{tabular}{ccc}
		\hskip-30pt\includegraphics[width=0.35\textwidth]{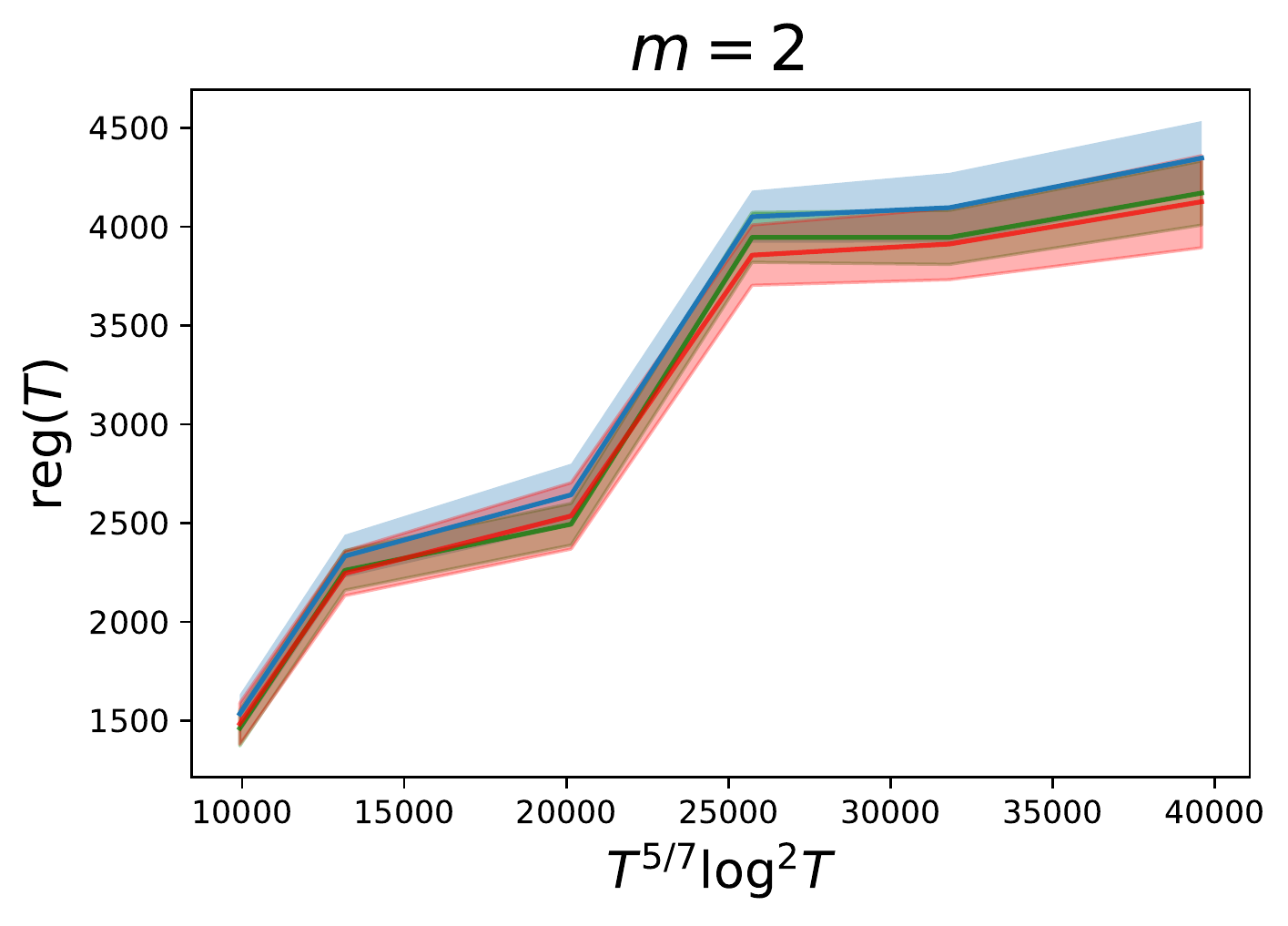}
		&
		\hskip-6pt\includegraphics[width=0.35\textwidth]{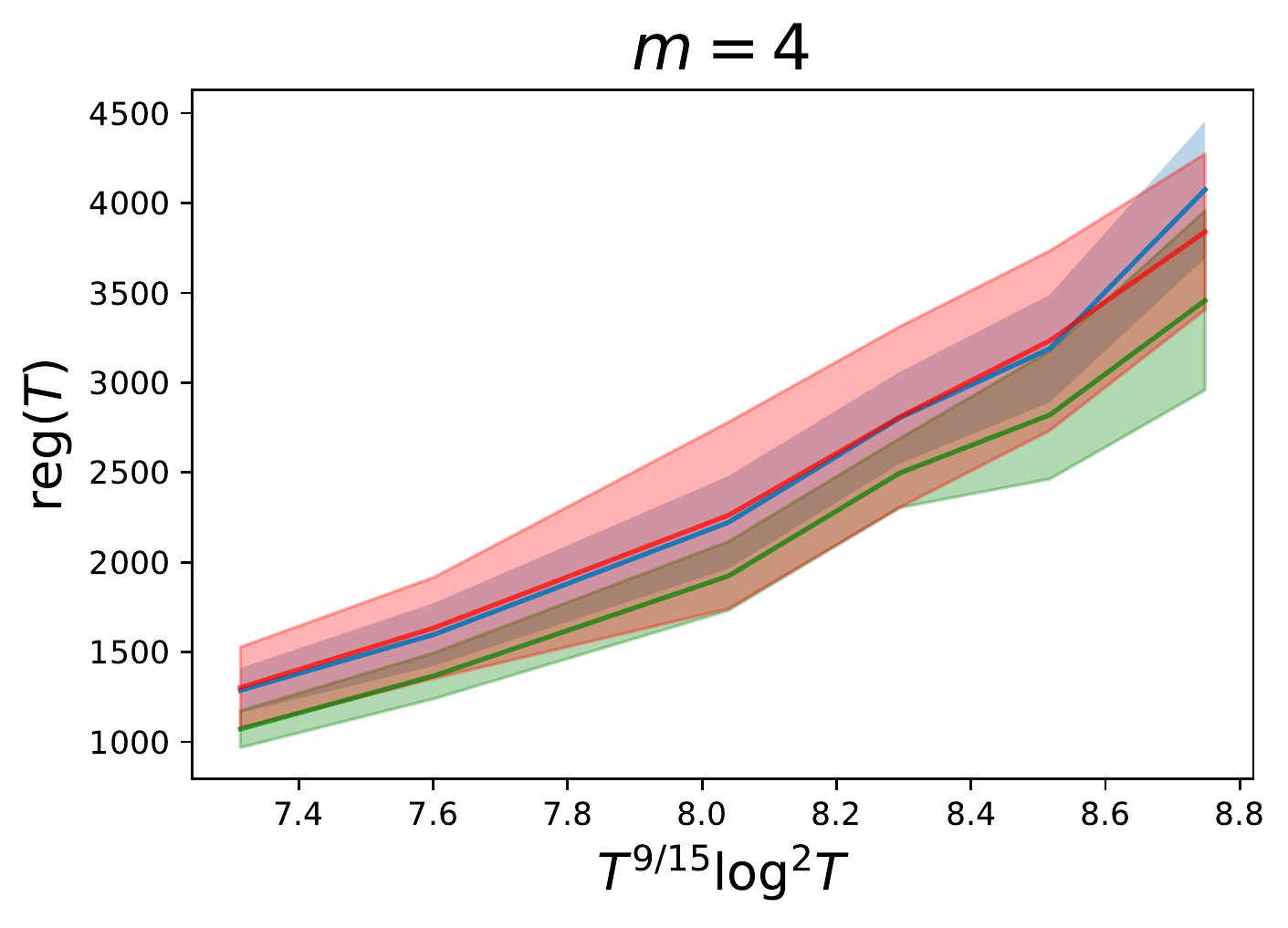}
		&
		\hskip-5pt\includegraphics[width=0.35\textwidth]{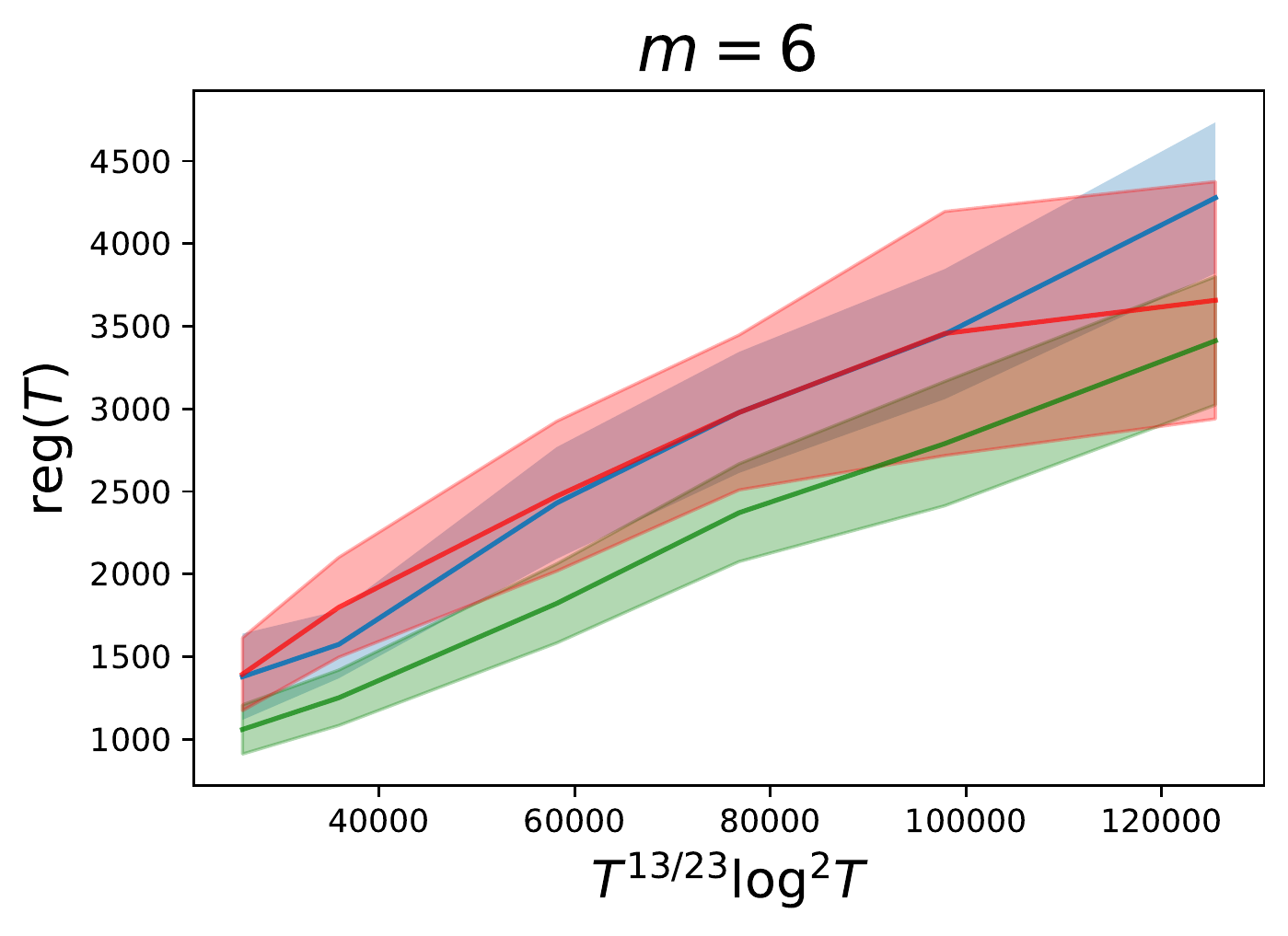} \\
		(a)  & (b) &(c)
	\end{tabular}\\
	\caption{From left to right, we plot empirical regret reg$(T)$ against $T^{(2 m+1)/(4 m-1)}\log^2 T$ with $ m\in [2,4,6]$  in the setting with i.i.d. covariates but dependent entries.  The rest caption is the same as in Figure~
	\ref{indp_app}. }
	\label{dp_app}
\end{figure}

\begin{figure}[H]
	\centering
	\begin{tabular}{ccc}
		\hskip-30pt\includegraphics[width=0.35\textwidth]{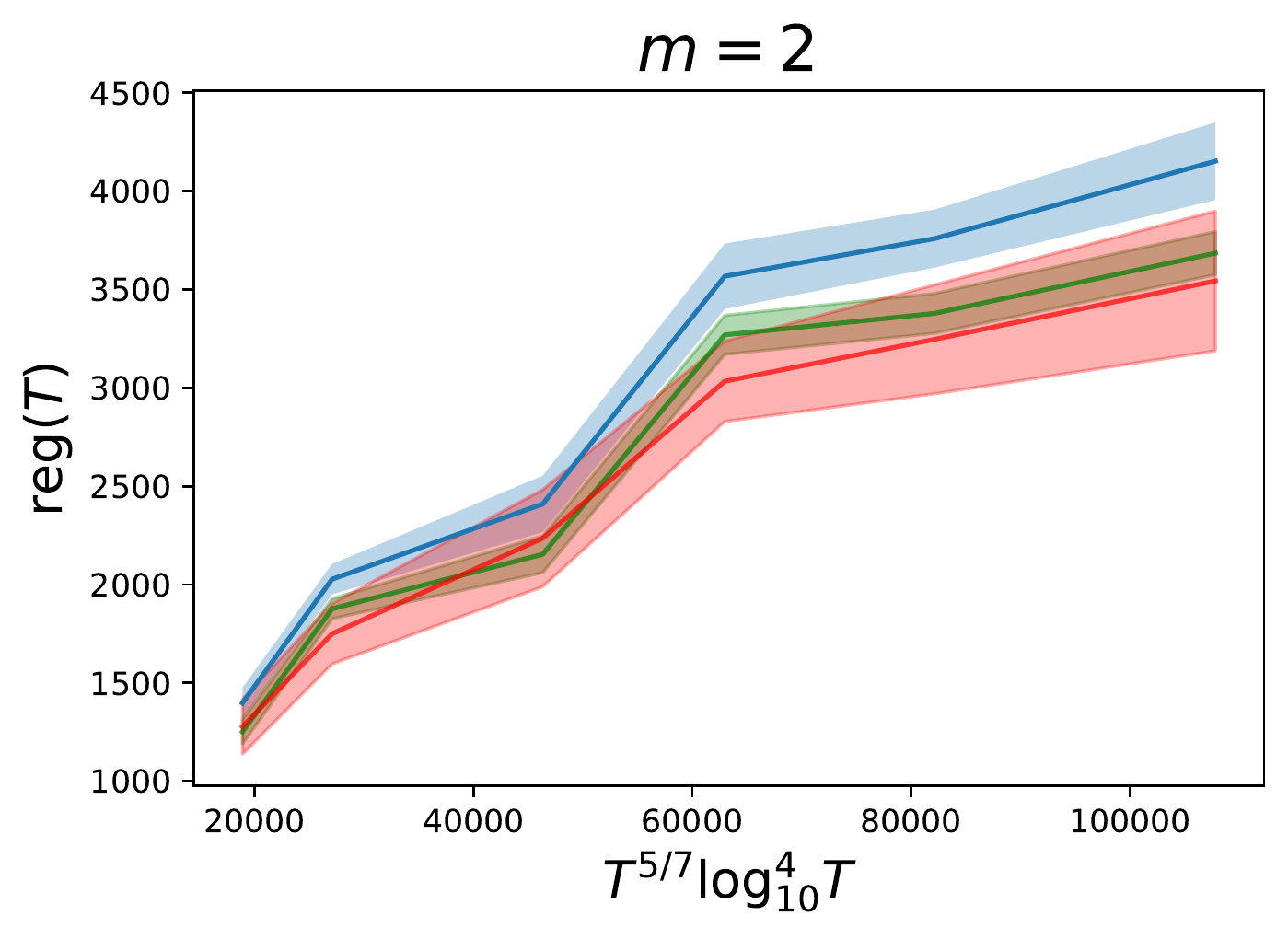}
		&
		\hskip-6pt\includegraphics[width=0.35\textwidth]{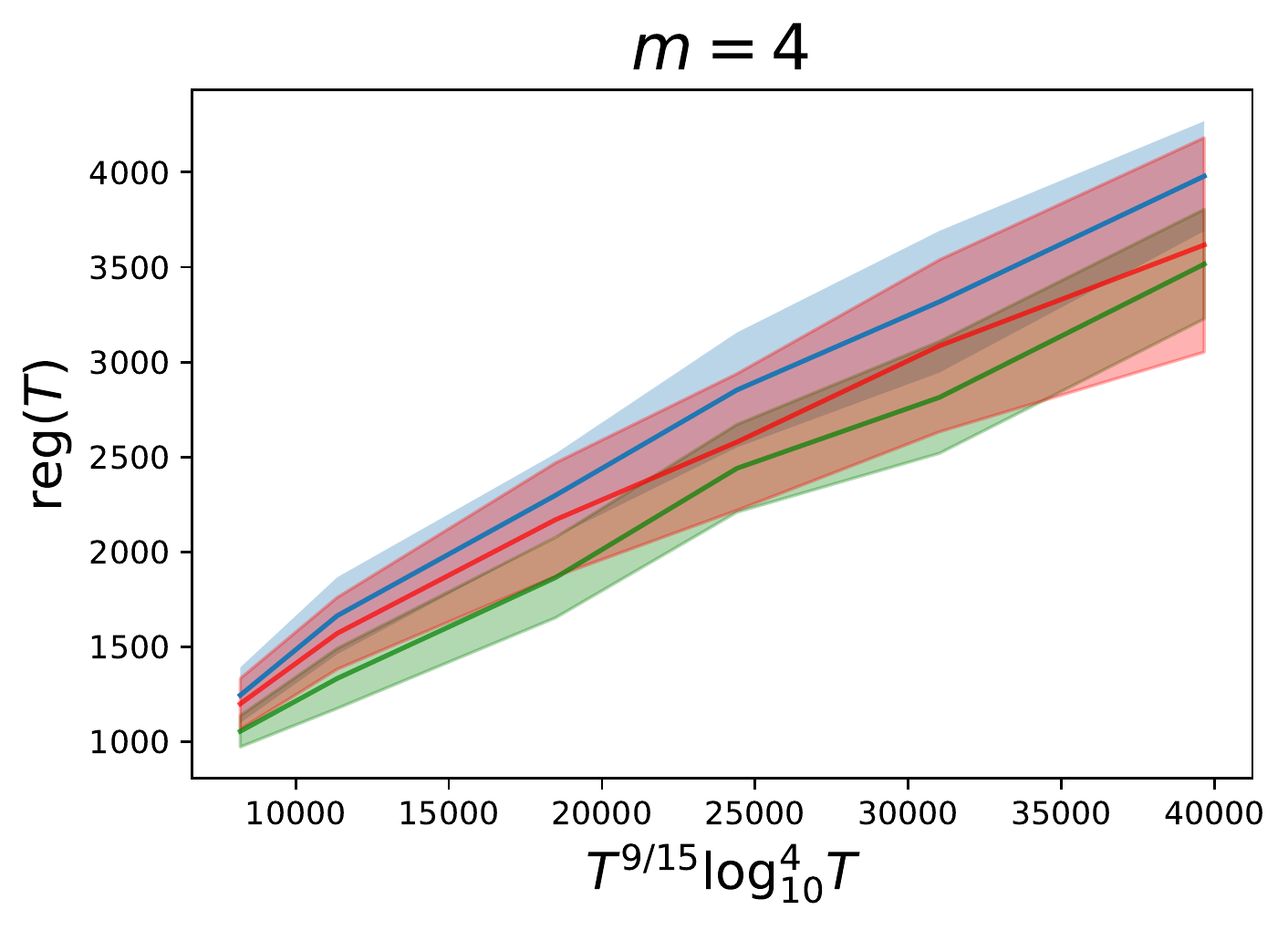}
		&
		\hskip-5pt\includegraphics[width=0.35\textwidth]{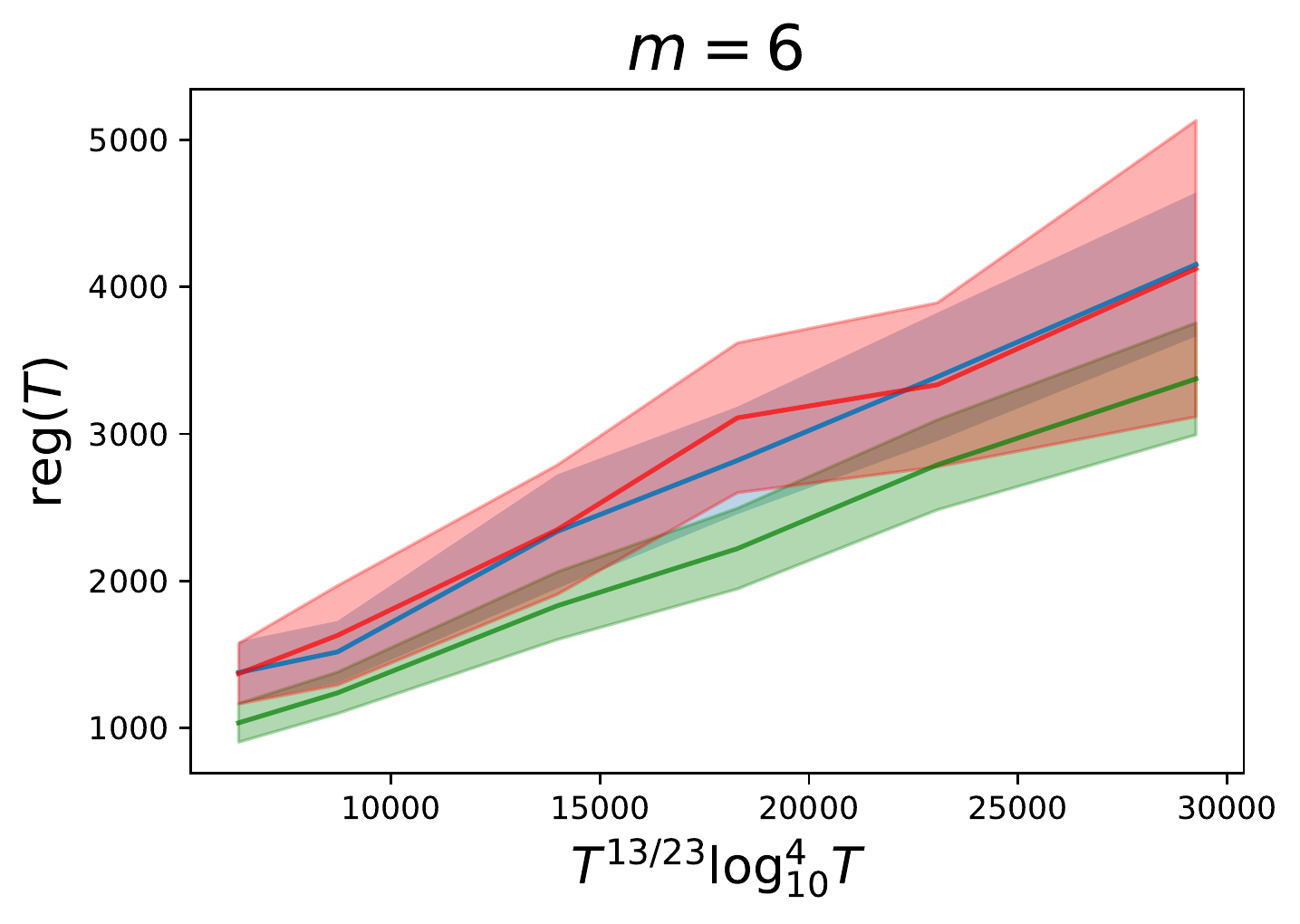} \\
		(a)  & (b) &(c)
	\end{tabular}\\
	\caption{From left to right, we plot empirical regret reg$(T)$ against $T^{(2 m+1)/(4 m-1)}\log^4_{10} T$ with $ m\in [2,4,6]$ in the setting with strong-mixing covariates.  The rest caption is the same as in Figure~
		\ref{indp_app}. }
	\label{mix_dpapp}
\end{figure}

\section{Regret bounds when $F(\cdot)$ is Lipschitz}\label{sec:m=0}
 All our main results require bounded second derivatives of $F$. This allows  the pricing strategy  $p_t=\hat\phi_k^{-1}(-\xb_t^\top\hat\btheta)+\xb_t^\top\hat\btheta$ to achieve low regret if the revenue function has bounded second derivative.   When  $F(\cdot)$ is only Lipschitz continuous, the above method is no longer applicable.  Fortunately, we can directly define the offered price based on the substitution of $\hat\btheta$ and $\hat F$  into \eqref{eq:pt*} We summarize these in the following Algorithm \ref{pricing_2}. 




\begin{algorithm}[H]
	\caption{Feature based dynamic pricing with unknown noise distribution when $F(\cdot)$ is $\ell$-Lipschitz}	
	\label{pricing_2}
	\begin{algorithmic}[1]		
		\STATE \textbf{Input:} { Upper bound of market value ($\{v_t\}_{t\ge 1}$)}: $B>0$, minimum episode length: $\ell_0$, degree of smoothness: $m=0$.
		\STATE \textbf{{Initialization:}}
		$p_1=0,\,\hat\btheta_1=0.$
		\FOR{each episode $k=1,2,\dots,$}
		\STATE Set length of the $k$-th episode $\ell_k = 2^{k-1}\ell_0$; Length of the exploration phase $a_k = \lceil(\ell_kd)^{\frac{3}{4}}\rceil$.
		\STATE \textbf{\underline{Exploration Phase} ($t\in I_k:= \{\ell_k,\cdots,\ell_{k} + a_k-1 \}$):}
		\STATE \quad Offer price $p_t\sim \text{Unif}(0,B).$ 
		\STATE \textbf{\underline{Updating Estimates} (at the end of the exploration phase with data $\{(\tilde{\xb}_t, y_t)\}_{t\in I_{k}}$):}
		\STATE \quad Update estimate of $\btheta_0$ by $\hat\btheta_k = \hat\btheta_k(\{(\tilde{\xb}_t, y_t)\}_{t\in I_{k}})$;
		{\begin{align}\label{thetaupdate_G}
\hat{\btheta}_k=\argmin_{\btheta} L_k(\btheta):= \frac{1}{|I_{k}|}\sum_{t\in I_{k}} (By_t-\btheta^\top\bx_t)^2
	 \end{align}}
		\STATE \quad Update estimates of $F$, by $ F_k(u,\hat\btheta_k) =  F_k(u; \hat\btheta_k, \{(\tilde{\xb}_t, y_t,p_t)\}_{t\in I_{k}})$ given by \eqref{exp:f_k}.
		\STATE \textbf{\underline{Exploitation Phase} ($t\in I_k':= \{\ell_{k} + a_k, \cdots, \ell_{k+1} -1 \}$): }
		\STATE \quad Offer $p_t$ as
	 \begin{align}\label{pt-2}
p_t=\text{argmax}_{p\ge 0} \{p(1-\hat F_k(p-\bx_t^\top\hat\btheta_k))\}
	 \end{align}
		\ENDFOR
	\end{algorithmic}
\end{algorithm}

\begin{theorem}\label{mainthm_m0}
Let Assumptions \ref{assp_inverse_0}, \ref{ass-pdfx}, \ref{ass-F} and \ref{ass-kernel} hold. Then there exist constants $C$ (depending only on  the absolute constants within the assumptions) such that for all $T$ satisfying
$T\geq Cd,$
the regret of Algorithm \ref{pricing_2} over time $T$ is no more than $C^*_{x, K}(Td)^{\frac{3}{4}}\log T(1+\log T/d)$.
\end{theorem}
\begin{proof}
We write
\begin{align}
\EE[R_t|\bar{\cH}_{t-1}]&=p_t^{*}(1-F(p_t^{*}-\bx_t^\top\btheta_0))-p_t(1-F(p_t-\bx_t^\top\btheta_0))\\
&=\text{rev}_t(p_t^{*},\btheta_0,F)-\text{rev}_t(p_t,\btheta_0,F). \label{diffrevenue}
\end{align}
The last inequality follows from our definition of \eqref{revenuet}. When $t\in I_k'$ (the $k$-th exploitation phase) we can then further expand \eqref{diffrevenue} into
\begin{align}
\eqref{diffrevenue}&=\text{rev}_t(p_t^{*},\btheta_0,F)-\text{rev}_t(p_t^{*},\hat\btheta_k,F)\label{termI}\\
&\quad + \text{rev}_t(p_t^{*},\hat\btheta_k,F)-\text{rev}_t(p_t^{*},\hat\btheta_k,\hat F_k)\label{termII}\\&\quad +\text{rev}_t(p_t^{*},\hat\btheta_k,\hat F_k)-\text{rev}_t(p_t,\hat\btheta_k,\hat F_k)\label{termIII}\\&\quad+\text{rev}_t(p_t,\hat\btheta_k,\hat F_k)-\text{rev}_t(p_t,\hat\btheta_k, F)\label{termIV}\\&\quad+\text{rev}_t(p_t,\hat\btheta_k, F)-\text{rev}_t(p_t,\btheta_0, F). \label{termV}
\end{align}
For \eqref{termIII}, by our definition on $p_t$ in \eqref{pt} , we have
\begin{align}
\eqref{termIII}=\text{rev}_t(p_t^{*},\hat\btheta_k,\hat F_k)-\text{rev}_t(p_t,\hat\btheta_k,\hat F_k)\le 0
\end{align}
For terms \eqref{termI} and \eqref{termV}, we can control both of them by difference between $\hat\btheta_k$ and $\btheta_0$ in a sense that
\begin{align}\label{signaldiff}
\eqref{termI},\eqref{termV}\lesssim |\langle\bx_t,\hat\btheta_k-\btheta_0 \rangle| \lesssim \frac{1}{\sqrt{a_k}}
\end{align}
holds with high probability by our Lemma {\red 4.1},
since we assume $F$ is Lipschitz continuous.  Recall $a_k=|I_k|,$ which is the length of the $k$-th exploration phase

For the rest two parts \eqref{termII} and \eqref{termIV}, as we are able to control $\hat F(x)$ to $F(x)$ uniformly with rate $a_k^{-1/3}$ using data in the exploration phase.
we can then bound $\EE[R_t]$  by 
\begin{align}
\EE[R_t]=\EE[\EE[R_t|\bar{\cH}_{t-1}]]\lesssim \frac{1}{a_k^{1/3}} 
\end{align}
Then for the regret in $k$-th episode we can bound it as
\begin{align}
\text{Regret}_k&=\sum_{t\in I_k} (\text{rev}_t^{*}-\text{rev}_t)+\sum_{t\in E_k\backslash I_k}(\text{rev}^{*}_t-\text{rev}_t)\\
&\le Ba_k+a_k/a_k^{1/3}=a_k^{3/4}+a_k/a_k^{1/4}=\cO(a_k^{3/4})
\end{align}
let $K=\lfloor\log_2 T\rfloor+1$,
we have our total regret can be bounded by
\begin{align}
\text{Regret}_{\pi}(T)=\sum_{k=1}^{K} 2^{3(k-1)/4}=\cO(T^{3/4}).
\end{align}
\end{proof}

\section{A Data Driven Way to Determine $m$}\label{secH}
{
As mentioned in Remark \ref{rem-m}, we are able to adopt the cross-validation method \citep{HALL2015} to determine the order of smoothness $m$ using data from the previous exploration phase. In the below, we briefly introduce how the order of smoothness can be determined in local polynomial regression in the context of nonparametric regression.

Given training data $\{x_i,y_i\}_{i=1}^{n}$ and we assume they are generated following model
\begin{align*}
    Y=g^*(X)+\epsilon,
\end{align*}
with $\EE[\epsilon\given X]=0.$ Define
\begin{align*}
    \textrm{CV}(h,m)=\frac{1}{n}\sum_{i=1}^{n}(Y_i-\hat g_{-i}(X_i))^2,
\end{align*}
where $\hat g_{-i}(\cdot)$ is fitted using all samples except the $i$-th pair $(x_i,y_i)$. Here we use bandwidth $h$ and $m$-th order local polynomial to fit the regression function.
According to Theorem 3.2 given in \cite{HALL2015}, optimizing CV$(h,m)$ is equivalent to optimizing over $(h,m)$ with respective to the averaged summed squared errors defined in \eqref{pred_error} up to some small order terms.
\begin{align}\label{pred_error}
    \frac{1}{n}\sum_{i=1}^{n}(\hat g(x_i)-g^*(x_i))^2.
\end{align}
Thus, this method is a valid way to determine the order of smoothness. We summarize the combined procedure in Algorithm \ref{pricing_with_m}.}

\begin{algorithm}[H]
	\caption{Feature based dynamic pricing with unknown $m$}
	\label{pricing_with_m}
	\begin{algorithmic}[1]		
		\STATE \textbf{Input:} { Upper bound of market value ($\{v_t\}_{t\ge 1}$)}: $B>0$, minimum episode length: $\ell_0$.
		\STATE \textbf{{Initialization:}}
		$p_1=0,\,\hat\btheta_1=0.$
		\FOR{each episode $k=1,2,\dots,$}
		{ \STATE If $k\ge 2$, use $\{\xb_t^\top \hat\btheta_{k-1},p_t,y_t\}_{t\in I_{k-1}}$ and Algorithm 4 to determine $m$. If $k=1$, set $\hat m=2.$}
		\STATE Set length of the $k$-th episode $\ell_k = 2^{k-1}\ell_0$; Length of the exploration phase $a_k = \lceil(\ell_kd)^{\frac{2\hat m+1}{4\hat m-1}}\rceil$.
		\STATE \textbf{\underline{Exploration Phase} ($t\in I_k:= \{\ell_k,\cdots,\ell_{k} + a_k-1 \}$):}
		\STATE \quad Offer price $p_t\sim \text{Unif}(0,B).$ 
		\STATE \textbf{\underline{Updating Estimates} (at the end of the exploration phase with data $\{(\tilde{\xb}_t, y_t)\}_{t\in I_{k}}$):}
		\STATE \quad Update estimate of $\btheta_0$ by $\hat\btheta_k = \hat\btheta_k(\{(\tilde{\xb}_t, y_t)\}_{t\in I_{k}})$;
			{\begin{align}\label{thetaupdate_m}
			\hat{\btheta}_k=\argmin_{\btheta} L_k(\btheta):= \frac{1}{|I_{k}|}\sum_{t\in I_{k}} (By_t-\btheta^\top\tilde{\xb}_t)^2.
			\end{align}}
		\STATE \quad If $\hat m\ge 1$, update estimates of $F$, $F'$ by $ F_k(u,\hat\btheta_k) =  F_k(u; \hat\btheta_k, \{(\tilde{\xb}_t, y_t,p_t)\}_{t\in I_{k}},\hat h_k)$, $ F_k^{(1)}(u,\hat\btheta_k)= F_k^{(1)}(u,\hat\btheta_k, \{(\tilde \xb_t,y_t,p_t)\}_{t\in I_{k}},\hat h_k)$. The detailed formulas are given by \eqref{exp:f_k} and \eqref{exp:f_11}. 
		\STATE \quad Update estimate of $\phi$ by $\hat\phi_k(u) =  u -\frac{1-\hat F_k (u)}{\hat F^{(1)}(u)} $	     and estimate of $g$ by $\hat g_k(u)=u+\hat\phi^{-1}_k(-u)$.\\
				\quad If $\hat m=0,$ update estimates of $F$, by $ F_k(u,\hat\btheta_k) =  F_k(u; \hat\btheta_k, \{(\tilde{\xb}_t, y_t,p_t)\}_{t\in I_{k}})$, The detailed formulas are given by \eqref{exp:f_k}.
		\STATE \textbf{\underline{Exploitation Phase} ($t\in I_k':= \{\ell_{k} + a_k, \cdots, \ell_{k+1} -1 \}$): }
		\STATE \quad If $\hat m\ge 1,$ offer $p_t$ as
		\begin{align}\label{pt_m}
	p_t=\min\{ \max\{\hat g_k(\tilde\xb_t^\top\hat\btheta_k),0\},B\}.		\end{align}
	\quad If $\hat m=0$, offer $p_t$ as
	\begin{align*}
p_t=\text{argmax}_{p\ge 0} \{p(1-\hat F_k(p-\bx_t^\top\hat\btheta_k))\}.
	\end{align*}
		\ENDFOR
	\end{algorithmic}
\end{algorithm}

\begin{algorithm}[H]
	\caption{Selection of $m.$}	
	\label{setting_m}
	\begin{algorithmic}[1]		
		\STATE \textbf{Input:} {Data $\{\xb_t^\top \hat\btheta_{t-1},p_t,y_t\}_{t\in I_{k-1}}$}
		\STATE {\textbf{For} $(m,h)\in \cM\times \cH$, compute:}
		\begin{align*}
		    (\hat m,\hat h)=\argmin_{(m,h)}L(m,h)=\frac{1}{|I_{k-1}|}\sum_{i=1}^{|I_{k-1}|}(Y_i-\hat g_{-i}^{(m,h)}(X_i))^2
		\end{align*}
		\STATE \textbf{Output: $\hat m$ }
	\end{algorithmic}
\end{algorithm}

\newpage
\bibliographystyle{ims}
\bibliography{dynamic}
\end{document}